\documentclass[12pt]{article}
\usepackage{amsmath}
\usepackage{graphicx}
\usepackage{enumerate}
\usepackage{natbib}
\usepackage{hyperref}
 \hypersetup{
     colorlinks=true,
     linkcolor=blue,
     filecolor=blue,
     citecolor = black,
     urlcolor=blue,
     }
\usepackage{url} 
\usepackage{graphicx,amsfonts,amsmath,amssymb,bm,natbib,booktabs,
colortbl,xcolor,float,comment,subfigure,makecell,dcolumn,multirow}

\usepackage{enumerate}
\usepackage{algorithm, algorithmicx, algpseudocode}
\floatname{algorithm}{Algorithm} 

\usepackage{caption}
\usepackage{enumitem}
\usepackage[section]{placeins}

\newtheorem{theorem}{Theorem}
\newtheorem{remark}{Remark}
\newtheorem{definition}{Definition}
\newtheorem{proof}{Proof}

\newtheorem{lemma}{Lemma}[section]
\newtheorem{assumption}{Assumption}

\definecolor{blue}{RGB}{0, 144, 178}
\definecolor{red}{RGB}{255,18,0}
\definecolor{yellow}{RGB}{240,228,66}
\definecolor{PolyRed}{RGB}{160,35,55}
\definecolor{SufeRed}{HTML}{811C21}
\definecolor{UIUCOrange}{HTML}{FF552E}
\definecolor{UIUCBlue}{HTML}{13294B}
\definecolor{MSU}{RGB}{24, 69, 59}
\definecolor{Maroon}{HTML}{8B008B}

\newcommand{\argmin}{\operatorname{argmin}}

\newcommand{\bD}{\mathbb{D}}
\newcommand{\bI}{\mathbf{I}}

\newcommand{\bX}{\mathbf{X}}

\newcommand{\bR}{\mathbb{R}}
\newcommand{\bE}{\mathbb{E}}

\newcommand{\bx}{\boldsymbol{x}}

\newcommand{\vs}{\boldsymbol{s}}

\newcommand{\vt}{\boldsymbol{t}}

\newcommand{\mB}{\mathcal{B}}

\newcommand{\mN}{\mathcal{N}}
\newcommand{\mD}{\mathcal{D}}

\newcommand{\mF}{\mathcal{F}}
\newcommand{\mH}{\mathcal{H}}
\newcommand{\mO}{\mathcal{O}}
\newcommand{\mW}{\mathcal{W}}
\newcommand{\mS}{\mathcal{S}}
\newcommand{\mM}{\mathcal{M}}
\newcommand{\mV}{\mathbb{V}}
\newcommand{\mL}{\mathcal{L}}

\newcommand{\bbeta}{\boldsymbol{\beta}}
\newcommand{\btheta}{\boldsymbol{\theta}}
\newcommand{\betatilde}{\tilde{\beta}}

\newcommand{\indep}{\;\, \rule[0em]{.03em}{.67em} \hspace{-.25em}
	\rule[0em]{.65em}{.03em} \hspace{-.25em}
	\rule[0em]{.03em}{.67em}\;\,}

\newcommand{\blind}{1}

\usepackage{times}
\begin{document}

\def\spacingset#1{\renewcommand{\baselinestretch}%
{#1}\small\normalsize} \spacingset{1}


\if1\blind

{
\title{\large \bf  Transfer Learning through Enhanced Sufficient Representation: \\
Enriching Source Domain Knowledge with Target Data}


\author{\normalsize Yeheng Ge$^a$\footnote{Equal contributions},  \ Xueyu Zhou$^{a *}$,\  and
Jian Huang$^{a,b}$\thanks{Corresponding author}\\
\newline
	{\small $^{a}$Department of Data Science and Artificial Intelligence, The Hong Kong Polytechnic University}\\
	{\small $^{b}$Department of Applied Mathematics, The Hong Kong Polytechnic University}
\date{\small \today}
}
\maketitle
}
 \fi

\if0\blind
{
  \bigskip
  \bigskip
  \bigskip
  \begin{center}
    {\large\bf Transfer Learning through Enhanced Sufficient Representation: \\
Enriching Source Domain Knowledge with Target Data
    }
\end{center}
  \medskip
} \fi

\bigskip

\begin{abstract}
Transfer learning is an important approach for addressing the challenges posed by limited data availability in various applications. It accomplishes this by transferring knowledge from well-established source domains to a less familiar target domain. However, traditional transfer learning methods often face difficulties due to rigid model assumptions and the need for a high degree of similarity between source and target domain models. In this paper, we introduce a novel method for transfer learning called Transfer learning through Enhanced Sufficient Representation (TESR). Our approach begins by estimating a sufficient and invariant representation from the source domains. This representation is then enhanced with an independent component derived from the target data, ensuring that it is sufficient for the target domain and adaptable to its specific characteristics.  A notable advantage of TESR is that it does not rely on assuming similar model structures across different tasks. For example, the source domain models can be regression models, while the target domain task can be classification. This flexibility makes TESR  applicable to a wide range of supervised learning problems. We explore the theoretical properties of TESR and validate its performance through simulation studies and real-world data applications, demonstrating its effectiveness in finite sample settings.
\end{abstract}

\noindent%
\vfill

\newpage
\spacingset{1.0} 

\section{Introduction}
Transfer learning has emerged as a powerful paradigm in statistics and machine learning, enabling models to leverage knowledge from related tasks to enhance performance on a target task with limited data \citep{pan2009survey, torrey2010transfer}. Traditional machine learning approaches often require large amounts of data, which can be costly and time-consuming to obtain. Transfer learning addresses the data scarcity problem by transferring knowledge from source domains, where data is abundant, to a target domain, where data is scarce \citep{weiss2016survey}.
Despite its success, existing transfer learning methods often face limitations in terms of scalability and adaptability to diverse domains \citep{zhuang2020comprehensive}.
In this paper, we propose a novel transfer learning method that overcomes these limitations by introducing a more flexible framework. This framework is capable of efficiently adapting to various target domains while maintaining high performance. Our approach builds upon recent advancements in sufficient representation learning and invariant risk minimization, offering a robust solution to the challenges faced by current transfer learning methodologies.

The interest in transfer learning across various settings has grown significantly in recent years. Many existing studies have focused on developing transfer learning methods and theories within parametric modeling frameworks. For instance,  \citet{bastani2021predicting} investigates transfer learning in high-dimensional linear regression with a large source dataset, providing an upper bound for the estimation error. Similarly, \citet{li2022transfer} propose a transfer learning method with minimax optimality in high-dimensional sparse linear regression, featuring a data-driven algorithm to exclude dissimilar source datasets. Both studies assume a small difference between the regression functions of the source and target domains as a prerequisite for successful knowledge transfer. Additionally, \citet{gu2022robust} introduce angle-based similarity for transfer learning in high-dimensional linear models. In the context of non-parametric models,  \citet{tony2021transClass} explore transfer learning for non-parametric classification, defining similarity through the relative signal exponent of the regression function. \citet{cai2022transfer} examine transfer learning in non-parametric regression models.

These studies require that the models used for both the source and target domains share the same functional form, with similar model parameters across these domains. Knowledge transfer is achieved by transferring parameter estimates from the source domains to the target domain. However, the assumption of having the same model form across domains can be restrictive and may not hold in many real-world scenarios. For example, the source domain might involve a classification problem for
cancer sub-types, while the target domain focuses on prognosis prediction for the same cancer. Transferring information from the classification outcomes to a continuous prognosis prediction task presents significant challenges. This limitation highlights the need for more flexible and robust transfer learning approaches that can accommodate varying model forms and complex relationships between domains.

It has been suggested that features with strong predictive power offer a significant advantage in transfer learning \citep{bengio2013representation, neyshabur2020being}. Building on this idea, some existing works on transfer learning assume that all tasks share a common representation \citep{maurer2016benefit, arjovsky2019invariant, hu2022improving}. Under this shared representation assumption, features learned from the source data can be directly applied to the target task, facilitating efficient knowledge transfer. However, in practice, representations from related tasks are often similar but not identical \citep{tony2021transClass}. Beyond the condition of shared representations, it is important to investigate the role of sufficient representations in facilitating knowledge transfer for general supervised learning problems.
Recently, several authors have developed nonlinear representation learning methods based on sufficient dimension reduction \citep{lee2013general, huang2020deep, YuGMDD2024}.
\citet{jiao2024deepTrans} consider transferring source domain sufficient and domain-invariant representations to the target task. They develop the prediction model for the target task,
assuming a linear relationship between the transferred source domain representations and the response of interest.
Despite these advancements, the study of data representations in transfer learning problems remains under-explored. Specifically, it remains unclear how to construct representations that can leverage knowledge from source domains while simultaneously capturing specific information from the target domain for general supervised learning problems.

In this paper, we introduce a novel transfer learning method that facilitates knowledge transfer through data representation. Specifically, the proposed method first estimates a  sufficient and invariant representation from the source domains.
This representation is then enhanced with additional independent representations obtained from the target data, which allow us to adapt to the specific characteristics of the target domain.
This enhancement strategy ensures its sufficiency for the target domain and still leverages valuable information from the source domains.
For simplicity,
we refer to the proposed method as Transfer learning through Enhanced Sufficient Representation (TESR).

A notable advantage of TESR is that it does not rely on assuming similar model structures across different tasks. For example, the source domain models can be regression models, while the target domain task can be classification. This flexibility makes TESR applicable to a wide range of supervised learning problems where the source and target domains may differ substantially.
Additionally, unlike conventional transfer learning approaches, which typically transfer model parameters, a key feature of TESR is that it focuses on transferring data representations from sources to the target domain. Another important aspect of TESR is that it does not assume that a representation sufficient for the source domain will also be sufficient for the target domain. Instead, we posit that a sufficient representation in the source domains provides useful information for the target domain.  This assumption allows for greater flexibility and adaptability in the transfer learning.

This paper is structured as follows: Section 2 presents the framework of the proposed TESR method. It presents the underlying principles and architecture of TESR.
In Section 3, we describe the objective functions for TESR and outlines the estimation procedure.
Section 4 establishes the convergence rate of the proposed method and examines its theoretical advantages. Section 5 evaluates the numerical performance of TESR through a series of simulation studies. Additionally, Section 6 illustrates the method's practical applications by applying it to two real-world datasets. Finally, Section 7 discusses the findings and suggests potential directions for future research.

\section{Method} 

In this section, we present the framework of
Transfer Learning through Sufficient and Invariant Representation. We begin by learning a sufficient and invariant representation (SIRep) from the source data. Next, we enhance SIRep with an augmented component derived from the target data. This augmented component is specifically designed to capture information present in the target dataset that is not contained in the SIRep.

Suppose that there are $S+1$ domains denoted as $\mD_s = \{X_s,Y_s\}$ for $s =0,1,\dots,S$,
where $X_s\in \mathbb{R}^d$ is the vector of predictors and $Y_s\in \bR^{q_s}$ is the vector of responses for task $s.$ We denote
$s=0$ as the index of the target domain and $s=1,\dots,S$ as
the source domains. Let $n_s = |\mD_s|$ be the sample size of the $s$th domain.
We are mainly interested in the scenario when the sample size of the target dataset $n_0 = |\mD_0|$ is limited, and the total sample size of the source datasets $N = \sum_{s=1}^S n_s$ is large.

The goal of transfer learning is to improve performance on the target task by leveraging information from source datasets. In this work, we allow the possibility that the source datasets may not share the same structure as the target dataset. For instance, the source domain task might involve a regression problem, whereas the target domain task could involve a classification problem.

\subsection{Sufficient representation}
For a given domain $s$,  a sufficient  representation is a measurable function $R_s: \mathbb{R}^d \to \mathbb{R}^{r_s}$ with the property  \citep{huang2020deep}:
\begin{align}
\label{def_sufficientRepre}
    Y_s\indep X_s \mid R_s(X_s), \ s= 0, 1, \ldots, S,
\end{align}
that is,  $Y_s$ and $X_s$ are conditionally independent given $R_s(X_s)$.
This implies the representation $R_s(X_s)$ contains all the information in $X_s$ relevant to $Y_s$.
This formulation is a nonparametric generalization of the basic condition in sufficient dimension reduction \citep{li1991sliced, cook2005sufficient}, where it is assumed $R_s(X_s)=B_s^T X_s$ with
$B_s\in \mathbb{R}^{d \times r_s}$ satisfying $B_s^T B_s=I_{r_s}.$ We refer to   \citet{huang2020deep} and   \citet{YuGMDD2024} for discussions on nonlinear sufficient representation.

 \subsection{Sufficient and invariant representation of source data}
 The main challenge in transfer learning lies in determining the specific type of knowledge that should be transferred from the source domains to the target domain. We want to transfer the most essential and relevant knowledge from these source domains to enhance the performance in the target domain.
 One effective approach to identifying this essential knowledge is through the concept of invariance  \citep{arjovsky2019invariant, jiao2024deepTrans}.
Invariance refers to the consistent and stable knowledge encoded in the predictors. Therefore, invariant representations from the source domains are expected to carry the most essential and robust knowledge, thus enjoying high generalization power to unknown target domain.

Therefore, we consider a  common representation $R_c(\cdot)$ that is sufficient and invariant
for all the source domains. We denote the joint dataset where the sources  $s=1,\dots,S$ are pooled together as $\mD_{pool} = \{X_{pool},Y_{pool},Z\},$ where $Z\in \{1,\dots,S\}$  is the categorical indicator  as the index of the domains. The pooled dataset $\mD_{pool}$ is of sample size $N = \sum_{s=1}^S n_s$ and $Z_i = s$ means the sample point $(X_{pool,i},Y_{pool,i})$ is from the dataset $\mD_s$.
Then a sufficient and  invariant representation (SIRep) for the source domains is defined as a function
$R_c$ that satisfies
\begin{align}
\label{def_Rc}
     Y_s \indep X_s | R_c(X_s) \ \text{ and } \
     R_c(X_{pool}) \indep Z,\
     \text{ for } s=1,\dots,S.
\end{align}
The first term
$Y_s \indep X_s | R_c(X_s)$ implies that $R_c(\cdot)$ is sufficient for all the source domains.
The second term $R_c(X_{pool}) \indep Z$ implies that the distributions of $R_c(X_s)$  are unchanged across the sources  $s=1,\dots,S$.
These invariant features are essential to ensure the transferred knowledge is both robust and applicable to new target domains  \citep{arjovsky2019invariant,jiao2024deepTrans}.
Such a representation always exists, since a trivial solution is $R_c = [R_1,\dots,R_S]$ that concatenates all the sufficient representations from the sources. Of course, such a simple combination of the individual sufficient representations is generally not an efficient solution for invariant representation as the $R_s$ for $s=1,\dots,S$ may share some information, leading to
overlapping components in terms of information content in the concatenated representation.

The estimation of sufficient and invariant representations does not preclude the existence of heterogeneity among the source domains. As mentioned earlier, in the extreme case where all source domains have different sufficient representations, we can simply concatenate these distinct representations. Generally, greater heterogeneity may result in a more complex invariant representation, while less heterogeneity leads to simpler ones. In practice, the extent of heterogeneity among source domains is often unknown. The proposed invariant representation method does not require prior knowledge of this heterogeneity or identification of which sources differ.

\subsection{Enhancing source domain representation with target data}
A critical issue is that a representation that is sufficient for the source domains may not be sufficient for the target domain. This discrepancy arises due to potential differences between the data distributions of the source and target domains. Consequently, it is essential to enhance the SIRep from the sources using the target data.

\begin{figure}[H]
\centering
\includegraphics[width=0.8\textwidth,trim = 0 100 20 0,clip]{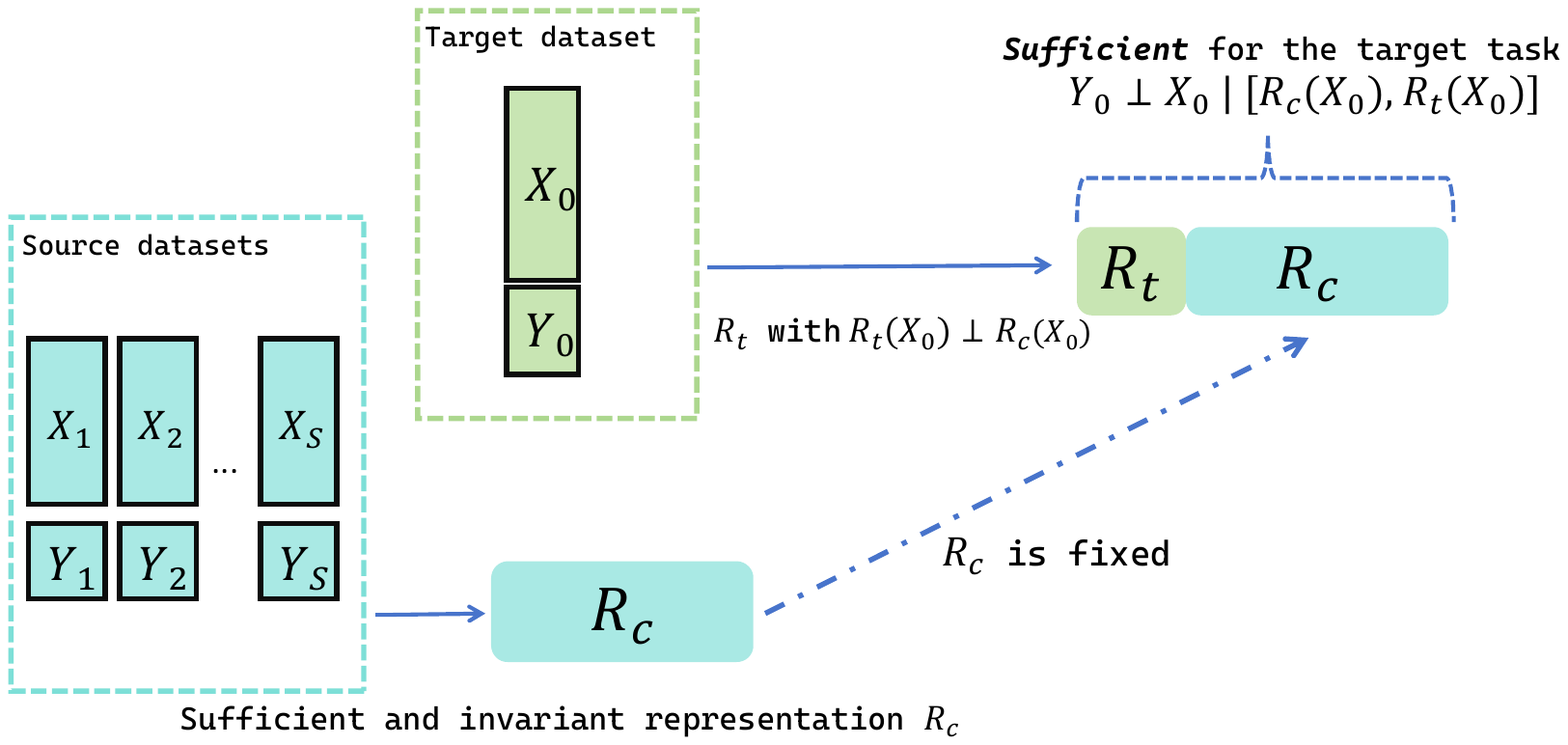}
\caption{The illustration of the proposed framework. The sufficient and invariant representation
$R_c$ summarizes information from the sources, while $R_t$ captures specific information relevant to the target task, ensuring that $[R_c, R_t]$ is sufficient for the target data. Importantly, $R_t(X_0)$ and $R_c(X_0)$ are required to be independent.
}
\label{Fig_Framework_Summary}
\end{figure}

We enrich the source data representation $R_c$ with an additional component $R_t$ such that the combined representation $\big[R_c, R_t\big]$ satisfies the following two requirements:

\begin{itemize}
\setlength{\itemsep}{0pt}
\setlength{\parsep}{0pt}
\setlength{\parskip}{0pt}
\item The combined representation is sufficient for the target data, that is,
\begin{align}
\label{Sec2:def_RcRstar}
Y_0\indep X_0 \mid \big[R_c( X_0),R_t( X_0)\big].
\end{align}
\item Any additional information represented by
$R_t$ for the target data, if necessary, should be independent of $R_c,$ that is,
 \begin{align}
 \label{Sec2:def_RcRstarb}
 R_c(X_0) \indep R_t(X_0).
 \end{align}
\end{itemize}
Condition (\ref{Sec2:def_RcRstar}) ensures that the target response $Y_0$ is independent of the target predictors $X_0$ given the combined representation $[R_c(X_0), R_t(X_0)].$ By achieving this, we can effectively leverage the knowledge from the source domains while adapting to the specific characteristics of the target domain. Meanwhile, by requiring the independence condition (\ref{Sec2:def_RcRstarb}), we  avoid learning  information that already captured by
$R_c$ with limited target dataset. Figure \ref{Fig_Framework_Summary} illustrates our framework.

The existence of $R_t$ is always guaranteed, since for any given $R_c$, we can set $R_t = R_0$, the sufficient representation for target domain, ignoring the contribution of $R_c.$
Generally,  $R_t$ can be significantly simpler and easier to estimate than $R_0$, when $R_c$ already contains a significant amount of the information for a sufficient representation in the $\mD_0.$ This representation learning approach aligns with the rationale of transfer learning, which seeks to leverage transferred knowledge to reduce the modeling complexity. By decomposing the representation into $R_c$ and $R_t,$ we can isolate the complex, invariant features captured by $R_c$ from the simpler, target-specific features captured by $R_t.$

By enhancing the source data representation with a component learned from the target data, we aim to create a combined representation that is both sufficient and adaptable.
This dual-focus strategy allows us to effectively utilize the valuable knowledge from the source domains while tailoring the model to address the unique characteristics of the target domain.
This results in a more robust, accurate, and generalizable model that can better achieve the overarching goals of transfer learning.
\begin{figure}[!htbp]
\centering
\includegraphics[width=0.40\textwidth,trim = 20 30 30 20,clip]{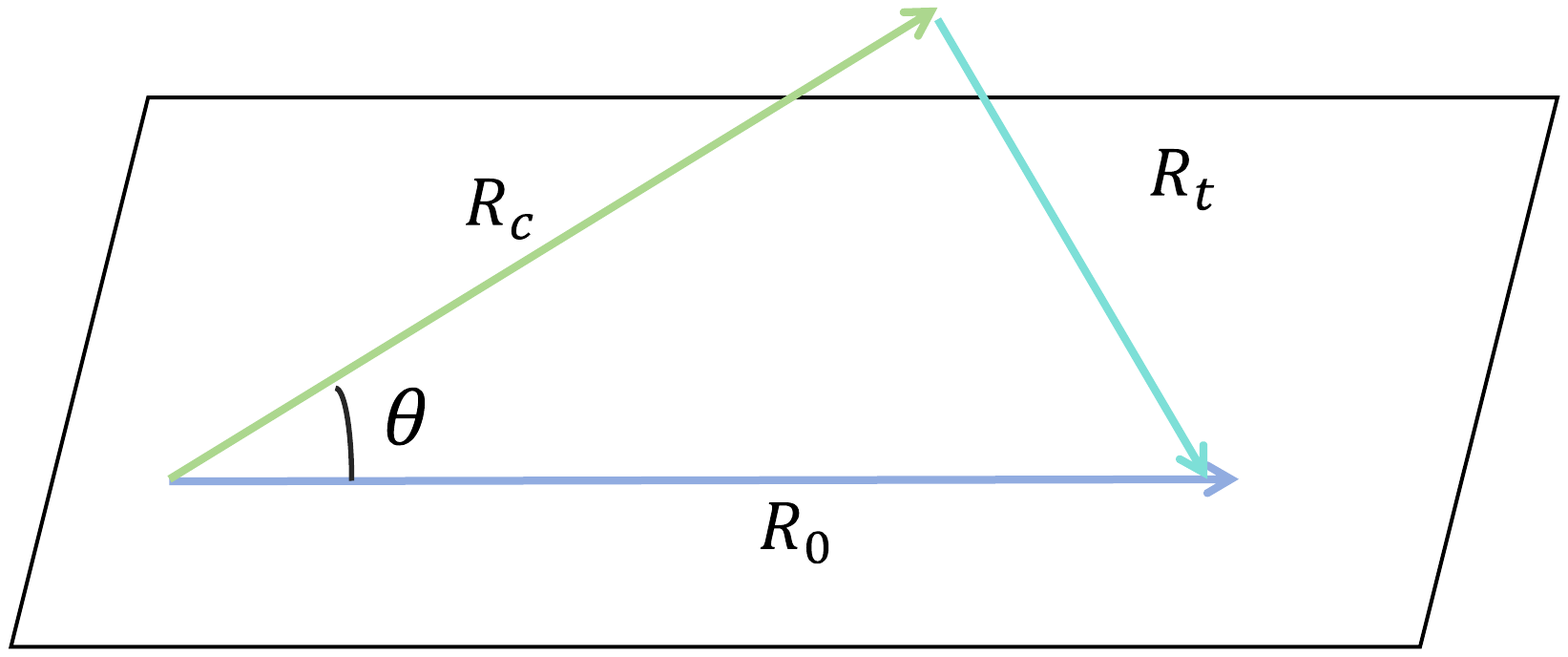}
\includegraphics[width=0.40\textwidth,trim = 20 30 30 20,clip]{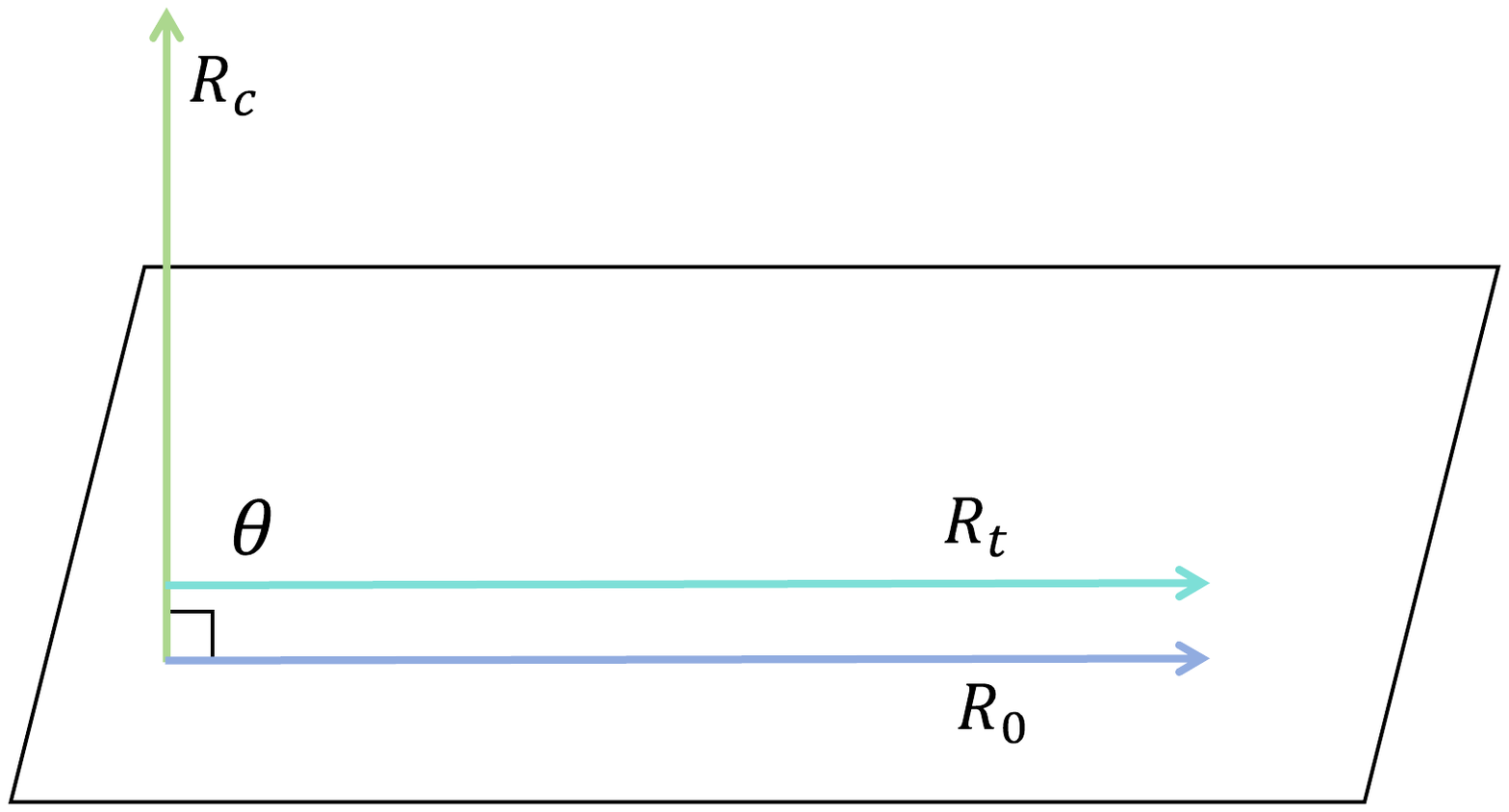}
\caption{The relationships between $R_c$, $R_t$, and $R_0$.
Estimating $R_t$ is simpler than estimating $R_0$ when $R_c \not\perp\!\!\!\perp R_0$ (left panel) with dataset $\mD_0$  because $R_c$ contains useful information about $R_0$. Conversely, when source data does not contain useful information about the target data, meaning $R_c \indep R_0$ (right panel), estimating $R_t$ is similar to directly estimating $R_0$.}
\label{Fig_Vertial_illustration}
\end{figure}

Figure \ref{Fig_Vertial_illustration} illustrates how TESR adapts to the heterogeneity between source and target data.
In cases where $R_c$ provides useful information about $R_0$ (left panel), $R_t$ can be learned such that the combination $[R_c, R_t]$ captures all the information necessary to recover $R_0$. Conversely, if $R_c$ is irrelevant or uninformative (right panel), that is, the source domains do not contain any useful information about the target domain, then we have $R_t = R_0$.
This adaptability makes TESR robust to both the presence and absence of meaningful information in $R_c$ for the target task. Additionally, the property $R_t \indep R_c$ ensures that the target dataset is used exclusively to learn new information that cannot be obtained from the source datasets,
leading to the efficient use of the limited sample size available in the target dataset.

\section{Estimation method}
We implement the TESR framework in two steps: First, we estimate $R_c$ using all source datasets $\mD_s$ with $s=1,\dots,S$. Then, with the estimated $R_c$, we estimate $R_t$ on the target task $\mD_0$, ensuring orthogonality between $R_c(X_0)$ and $R_t(X_0)$. The combined representation $[R_c, R_t]$ is then used for the target tast on $\mD_0$.

\subsection{Preliminaries}
In this subsection, we describe the measures that will be used for characterizing independence and equality in distributions.

\subsubsection{Distance covariance}
\label{sub_ds}
Let $\mV$ be a  dependence measure between random variables $U$ and $V$ with the following properties:
(a) $\mV[U,V] \ge 0$ with $\mV[U,V]=0$ if and only if $U\indep V$;
(b)  $\mV[U,V] \ge \mV[R(U), V]$ for every measurable function $R$;
(c) $\mV[U,V]=\mV[R^*(U), V] \ \text{ if and only if } R^* \text{ is sufficient}.$
These properties imply that $R^*$ is a sufficient representation for prediction $V$ if and only if
$ R^* \in \argmin_{R}\{-\mV[R(U), V]\}. $

We use distance covariance to measure the dependence of two random variables  \citep{szekely2007measuring}.
Let $\mathfrak{i}$ be the imaginary unit $(-1)^{1 / 2}$. For any $\vs\in \mathbb{R}^{d}$ and $\vt\in \mathbb{R}^m$, let  $\psi_{U}(\vs)=\mathbb{E} [\exp^{\mathfrak{i}\vs^TU}], \psi_{V}(\vt)=\mathbb{E} [\exp^{\mathfrak{i}\vt^TV}],$ and $\psi_{U, V}(\vs,\vt) = \mathbb{E} [\exp^{\mathfrak{i}(\vs^T U+ \vt^T V)}]$ be the characteristic functions of random vectors $U\in \mathbb{R}^d , V \in \mathbb{R}^m,$ and the pair  $(U, V)$, respectively.
A specific squared distance covariance $\mathbb{V}[U, V]$ is defined as
$$
\mathbb{V}[U, V]=\int_{\mathbb{R}^{d+m}} \frac{\left|\psi_{U, V}(\vs, \vt)-\psi_{U}(\vs) \psi_{V}(\vt)\right|^2}{c_{d}c_{m}  \|\vs\|^{d+1}\|\vt\|^{m+1}} \mathrm{d} \vs \mathrm{d} \vt,
$$
where $c_{d}=\frac{\pi^{(d+1) / 2}}{\Gamma((d+1) / 2)}.$
This form of distance covariance has a simple expression,
\begin{align*}
\begin{aligned}
    \mV(U,V) = &
    \bE\|U-U^{\prime}\| \|V-V^{\prime}\|+\bE\|U-U^{\prime}\|\bE[\|V-V^{\prime}\|]
    \\&-\bE\|U-U^{\prime}\|\|V-V^{\prime\prime}\|
    -\bE\|U-U^{\prime\prime}\|\|V-V^{\prime}\|,
\end{aligned}
\end{align*}
where
 $\|\cdot\|$ is the Euclidean distance,
$\{U', V'\}$ and
$\{U^{''}, V^{''}\}$ are independent copies of $(U,V)$.

Given $n$ i.i.d. copies $\{U_{i}, V_{i}\}_{i=1}^n$ of $(U,V)$,
an unbiased estimator of $\mV$
is the empirical distance covariance $\widehat{\mV}_{n}$,
which can be elegantly expressed as a $U$-statistic  \citep{huo2016fast}
\begin{equation}\label{usta}
\widehat{\mV}_{n}[U, V]
 = \frac{1}{\tbinom{n}{4}} \sum_{1 \leq i_{1}<i_{2}<i_{3}<i_{4} \leq n} h\left(\left(U_{i_{1}}, V_{i_{1}}\right), \cdots,\left(U_{i_{4}}, V_{i_{4}}\right)\right), \nonumber
\end{equation}
where $h$ is the kernel defined by
\begin{eqnarray*} 
h\left(\left(u_{1}, v_{1}\right),\ldots,
\left(u_{4}, v_{4}\right)\right)
&=&\frac{1}{4} \sum_{1 \leq i, j \leq 4 \atop i \neq j}\|u_{i}-u_{j}\| \|v_{i}-v_{j}\|
 +\frac{1}{24} \sum_{1 \leq i, j \leq 4 \atop i \neq j}\left\|u_{i}-u_{j}\right\| \sum_{1 \leq i, j \leq 4 \atop i \neq j}\|v_{i}-v_{j}\| \\
& &
 -\frac{1}{4} \sum_{i=1}^{4}(\sum_{1 \leq j \leq 4 \atop j \neq i}\left\|u_{i}-u_{j}\right\| \sum_{1 \leq j \leq 4 \atop i \neq j}\|v_{i}-v_{j}\|).
\end{eqnarray*}

\subsubsection{Energy distance}
\label{sub_ed}
Since sufficiency is invariant under one-to-one transformations, it is always possible to transform
a sufficient representation $R^*(X)$ to have a Gaussian distribution, under the assumption
that $R^*(X)$ has a finite second moment and is absolutely continuous with respect to the Gaussian distribution \citep{huang2020deep}. Therefore, we can restrict the space of sufficient representations to those with a standard Gaussian distribution. For the Gaussian regularization, we use a divergence measure $\mathbb{D}$  to quantify the difference between the distributions of two random variables $U$ and $V$. This measure should satisfy the condition
$\mathbb{D}(U\Vert V)\ge 0 $
and  $\mathbb{D}(U\Vert V)=0$  \text{ if and only if }
$U$ and $V$ have the same distribution.
In this work, we use the energy distance  \citep{rs2016},
\begin{align}
\label{ed1}
    \bD(U\Vert V) = 2\bE\|U-V\|-\bE\|U-U'\|-\bE\|V-V'\|,
\end{align}
where $\{U',V'\}$ are independent copies of $\{U,V\}$, respectively.
We note that \citet{huang2020deep} used Generative Adversarial Networks (GANs) to push the distribution of  data representations to a Gaussian distribution. However, GAN-based regularization requires complex optimization designs and may lead to suboptimal convergence in practice. In contrast, the energy distance used is simpler and easier to compute, providing a practical alternative to GAN-based approaches.

When two empirical samples  $\{U_i, i=1, \ldots, n\}$ and $\{V_j, j=1, \ldots, n\}$ are available,
following the suggestion of   \citet{gretton2012kernel}, we use the following empirical version
of the energy distance,
    \begin{align}
    \label{edn}
 \bD_n(U\Vert V) = \frac{1}{\tbinom{n}{2}} \sum_{1\leq i,j\leq n}h_e(u_i,u_j;v_i,v_j),
\end{align}
where
$h_e(u_1,u_2;v_1,v_2)=   \|u_1- v_2\| + \|u_2- v_1\|
-
\|u_1- u_2\| - \|v_1- v_2\|
.$

\subsection{Estimating sufficient and invariance representation from source data}
We first present the proposed method for estimating a SIRep $R_c$ from the source datasets $\mD_s, s=1,\dots,S$. We are interested in estimating a representation characterized in
(\ref{def_Rc}). Based on (\ref{def_Rc}), a SIRep $R_c$ can be characterized as a
solution to the constrained optimization problem:
\begin{align}
\label{Sec3_obj1}
    &R^*_c = \underset{R}{\argmin}\Big\{ -\sum\limits_{s=1}^S \mV(R(X_s),Y_s)\Big\},\nonumber\\
    &\text{subject to }\  R(X_s) \sim N(0,\bI)\ \text{ and } \ R(X_{pool}) \indep Z, \text{ for } s=1,\dots,S,
\end{align}
where, as discussed in Subsection \ref{sub_ed}, we imposed a constraint that the representation has
a Gaussian distribution.
The Lagrangian form of (\ref{Sec3_obj1}) is
    \begin{align}
        \label{pop0}
       \mL_{S}(R) = \sum\limits_{s=1}^S
     \Big\{-\mV(R(X_s),Y_s)
    + \lambda_E \bD(R(X_s)\Vert \gamma_{r_c})
    \Big\}+ \lambda_Z \mV(R(X_{pool}),Z) ,
    \end{align}
where
$\gamma_{r_c}\sim N(0,\boldsymbol{I}_{r_c})$,
$\lambda_E$ and $\lambda_Z$ are regularization parameters.

Now suppose we have source datasets
$\mD_s = \{X_{s,i}, Y_{s,i}, i=1, \ldots n_s\}, s=1, \ldots, S.$
Let $\mV_{n}(\cdot,\cdot)$ be the empirical distance covariance defined in (\ref{usta}), and
let $\bD_{n}(\cdot \Vert \cdot)$ be the empirical version of the energy distance defined in (\ref{edn}).
Then we have the empirical objective function,
\begin{align}
\label{obj_empirical}
  \mL_{S,n}(R) = \sum\limits_{s=1}^S\Big\{-\mV_{n}(R(X_s),Y_s)
    + \lambda_E \bD_{n}(R(X_s)\Vert \gamma_{r_c})\Big\}+ \lambda_Z \mV_{n}(R(X_{pool}), Z),
\end{align}
where $X_{pool}$ is the pooled covariates from $\mD_s$, $s=1,\dots,S$,
$Z$ is the one-hot domain indicator, $\lambda_E$ and $\lambda_{Z}$ are tuning parameters. Then
$\widehat R_c =  \argmin\limits_{R\in \mF_{R_c}} \mL_{S,n}(R).$

\subsection{Estimating enhanced representation with target data}
After a SIRep $R_c$ is derived from the source datasets is obtained,  we move on to construct a
representation for the target data. If there is no difference between the source and target data, we can simply use $R_c$ as the representation for the target dataset. However, to account for potential heterogeneity between the source and target data, we enhance $R_c$ using the target data.

Specifically, for a given $R_c$ and the target data $\mD_0 = \{X_0,Y_0\}$, we seek an augmentation $R_t$ that satisfies (\ref{Sec2:def_RcRstarb}). This can be equivalently formulated as an optimization problem with the objective
\begin{align}
\label{Sec3_obj2}
&R^*_t = \underset{R}{\argmin}\{ -\mV([R(X_0),R_c(X_0)\big],Y_0)\}, \nonumber\\
& \text{subject to }\   R(X_0) \indep R_c(X_0) \text{ and } R(X_0) \sim N(0,\bI),
\end{align}
where similar to Subsection \ref{sub_ed}, we require the augmentation $R^*_t$ to be Gaussian.

We reformulate this problem by expressing the objective function in the  Lagrangian form,
\begin{align}
 \label{obj_pop22}
\mL_{T}(R, R_c) = -\mV(\big[R(X_0), R_c(X_0)\big],Y_0) + \lambda_C \mV(R(X_0), R_c(X_0))
+ \lambda_{E,0} \bD(R(X_0)\Vert \gamma_{r_t}),
\end{align}
where
$\gamma_{r_t}\sim N(0,\boldsymbol{I}_{r_t})$,
the $\lambda_C \ge 0 $ and $\lambda_{E,0}\ge 0 $ are regularization parameters.
Then for a given estimator $\widehat R_c$ based on the source datasets, we augment it by
\begin{align}
\label{obj_empirical2}
   & \widehat R_t = \underset{R\in \mF_{R_t}}{\argmin} \
 \mL_{T,n}(\big[R,\widehat R_c\big]),
\end{align}
with
$
\mL_{T,n}(R,\widehat R_c) = -\mV_n(\big[R(X_0),\widehat R_c(X_0)\big],Y_0)
    + \lambda_C \mV_n(R(X_0),\widehat R_c(X_0))
    + \lambda_{E,0} \bD_n(R(X_0)\Vert \gamma_{r_t}).
$

We summarize the implementation of TESR in Algorithm \ref{algo:TransRep_with_reg}.
After obtaining the representation $[\widehat R_c,\widehat  R_t]$,  we can build predictive models based on the representation.

\begin{algorithm}[!ht]
\footnotesize
\spacingset{1.3}
\caption{Transfer learning through Enhanced Sufficient Representation} 
  \label{algo:TransRep_with_reg}
  \textbf{Input: }
  The  source datasets
  $\mD_{s} = \{Y_{s},X_{s}\}$
  for $s=1,\dots, S$ and
  the target dataset
  $\mD_{0} = \{Y_{0},X_{0}\}$,
  tuning parameter $\lambda_E,\lambda_Z,\lambda_C,\lambda_{E,0}$.

\textbf{Step I: Estimating SIRep $R_c$ with source datasets }
  \begin{algorithmic}[1]
    \State
    With the source datasets $\mD_{s}$ with $s=1,\dots, S$,
    we learn $R_c$ with the following objectives,
\begin{align}
\label{obj_Rc}
  \widehat  R_c = \underset{R\in \mF_{R_c}}{\argmin} \  \mL_{S,n}(R),
\end{align}
where
$\mF_{R_c}$ is the neural network class for learning $R_c$,
$\mL_{S,n}(R) =
 \sum\limits_{s=1}^S\Big\{-\mV_{n}(R(X_s),Y_s)
    + \lambda_E \bD_{n}(R(X_s)\Vert \gamma_{r_c})\Big\}+ \lambda_Z \mV_{n}(R(X_{pool}), Z)$, $Z$ is the one-hot  domain indicator and $X_{pool}$ is the pooled covariates for datasets with $s=1,\dots,S$.
  \end{algorithmic}

  \textbf{Step II: Estimating $R_{t}$ with target dataset} 
  \begin{algorithmic}[1]
  \State
For a given estimator $\widehat R_c$ based on the source datasets, we learn $R_t$ with the target dataset $\mD_{0} = \{Y_{i},X_{i}\}$ to augment $\widehat R_c.$
Specifically,
\begin{align*}
   & \widehat R_t = \underset{R\in \mF_{R_t}}{\argmin} \
 \mL_{T,n}(\big[R,\widehat R_c\big]),
\end{align*}
where
$\mF_{R_t}$ is the neural network class
and
$
\mL_{T,n}(R,\widehat R_c)=-\mV_n(\big[R(X_0),\widehat R_c(X_0)\big],Y_0) + \lambda_C \mV_n(R(X_0),\widehat R_c(X_0)) + \lambda_{E,0} \bD_n(R(X_0)\Vert \gamma_{r_t}).
$
\end{algorithmic}

\textbf{
  Output: The estimated representation $[\widehat  R_c,\widehat  R_t]$.
}
\end{algorithm}

\subsection{Linear representations}
To better illustrate the concept of the proposed TESR framework, we consider
the case where all the representation functions are linear in this section.

\begin{itemize}
\setlength{\itemsep}{0pt}
\setlength{\parsep}{0pt}
\setlength{\parskip}{0pt}
\item
A linear sufficient representation $R_s(X_s)=B_s^\top X_s$ for the $s$th source dataset satisfies
    $Y_s \indep X_s \mid B_s^\top X_s,$
where $B_s\in \mathbb{R}^{d \times r_s}$ and $B_s^T B_s=I_{r_s}.$

\item
A linear sufficient and invariant representation $R_c(X)=B_c^\top X$ satisfies
\begin{align}
\label{Case_linear_Bc}
    Y_s \indep X_s \mid B_c^\top X_s, \text{ for } s=1,\dots,S,
    \text{ s.t. }
    B_c^\top X_{pool}\indep Z,
\end{align}
where $B_c\in \mathbb{R}^{d \times r_c}$ and $B_c^T B_c=I_{r_c}$,
$X_{pool}$ is the pooled covariates from different sources and $Z$ is the source indicator.
\end{itemize}
The invariance constraint $B_c^\top X_{pool} \indep Z$ ensures that $B_c$ contains stable knowledge encoded in the predictors. The orthogonality condition $B_c^T B_c = I_{r_c}$ is similar to the normality constraint \eqref{Sec3_obj1} for nonlinear representation functions.
These constraints help in achieving a stable and interpretable representation by ensuring that the learned features are orthogonal, which is crucial for maintaining the disentanglement and regularity properties.

Similar to \eqref{pop0},  $B_c$ can be defined as a solution of the optimization problem,
    \begin{align*}
       B_c^*\in \underset{B_c\in \bR^{d\times r_c}}{\argmin} \sum\limits_{s=1}^S
     \Big\{-\mV(B_c^\top X_s,Y_s)
    + \lambda_E  ||B_c^\top B_c - \mathbf{I}_{r_c}||^2
    \Big\}+ \lambda_Z \mV(B_c^\top X_{pool},Z) ,
    \end{align*}
where
$\lambda_E$ and $\lambda_Z$ are regularization parameters. A trivial solution for summarizing the information from the sources is the stacked matrix $[B_1, \dots, B_S]$.
However, the sources usually share some similarities, leading to the matrix $[B_1, \dots, B_S]$ not being of full rank. Therefore, it is preferred to seek a more concise representation that leverages these similarities among the sources efficiently. To this end, it is sufficient to learn the $B_c$ which is the basis of the space spanned by $[B_1,\dots,B_S]$. This approach aligns with the findings presented in  \citet{xu2022distributed} for the analysis of dimension reduction with heterogeneous sub-datasets.

With the $B_c$ from the source domains that satisfies (\ref{Case_linear_Bc}),
our goal is to learn a $B_t\in \mathbb{R}^{d \times r_t}$ augmenting $B_c$ such that
$[B_c,B_t]$ is sufficient and $B_c \indep B_t.$ This can be stated as
\begin{align*}
    Y_0 \indep X_0 \mid [B_c^\top X_0,B_t^\top X_0]
  \  \text{ s.t. } \ B_c^\top  B_t = \boldsymbol{0}.
\end{align*}
The $B_t$ can be characterized as a solution to the optimization problem
    \begin{align*}
       B_t^* \in \underset{B_t\in \bR^{d\times r_t}}{\argmin}
     \Big\{-\mV([B_c^\top X_0,B_t^\top X_0],Y_0)
    + \lambda_{E,0}||B_t^\top B_t - \mathbf{I}_{r_t}||^2+ \lambda_{C} ||B_c^\top  B_t||^2
    \Big\},
    \end{align*}
where $\lambda_{E,0}$ and $\lambda_{C}$ are tuning parameters. The solutions $B_c^*$ and $B_t^*$ are typically non-unique, but all solutions lead to the same central subspace  \citep{ma2013efficient,xu2022distributed}. We refer to the Section A.2 of the Supplementary Materials for more details.

The method of linear sufficient dimension reduction using distance covariance has been discussed by  \citet{sheng2016sufficient,sheng2013direction}.
The linear cases discussed above serve as an illustration of the concept of TESR.
However, the intrinsic structures in real-world data are usually more complex and cannot be precisely represented by linear combinations of predictors. Nonlinear representations allow for a more flexible and comprehensive characterization of the underlying data structures.

\section{Theoretical Guarantees}
Under the framework of empirical risk minimization, the performance of the empirical risk minimizer $\big[\widehat R_c,\widehat R_t\big]$ can be evaluated by the excess risk.
In this section, we establish the  rate for the excess risk of the proposed TESR method.

Theoretical results are influenced by several factors, including the properties of $R_c,R_t$, the neural network classes $\mF_{R_c}$, $\mF_{R_t}$, the sample sizes of sources $n_s$ and  target datasets $n_0$. We introduce some mild theoretical conditions on these factors, which are commonly considered in the deep learning literature. For further explanation and technical details regarding the function class and neural network classes, please refer to Section A of the Supplementary Materials.

We first have assumptions on the data distributions,
\begin{assumption}
\label{assumption_finitexy}
    Let $\mu_X$ be the  probability measure of the covariates. The $\mathrm{supp}(\mu_{X})$ is contained in a compact set, say $[-B,B]^d$ with a finite $B$ and denote its density function as $f_{X}(x)$. $Y$ is bounded almost surely, say $\|Y\|\leq C$ a.s..
\end{assumption}

Then we introduce the conditions on the representation functions.
Recall that  we are interested in
$R_c = \big[R_{c,1},\dots,R_{c,r_c}\big]$,
$R_t = \big[R_{t,{1}},\dots,R_{t,{r_t}}\big]$,
and
$R_0 = \big[R_{0,{1}},\dots,R_{t,{r_0}}\big]$.
In this paper, we assume that these functions  belong to $B^{\bbeta}_{p,q,\betatilde}(\Omega)$, the  anisotropic Besov (a-Besov) function class
where
$\bbeta = (\beta_1,\dots,\beta_d)^\top \in \bR_{+}^d$
is the non-negative smoothness indices on the direction of  $d$ coordinates,
 $\betatilde$ is the average smoothness,
$
 \betatilde = (\sum_{k=1}^d 1/\beta_k)^{-1}
$
,
$\Omega$ is the domain and $p,q$ are norm indices
\citep{suzuki2021deep}.
The coordinate mixed smoothness  provides insights for overcoming the curse of dimensionality   \citep{suzuki2021deep}.
For simplicity, in the following analysis, we omit the $\bbeta$ and denote the  a-Besov space as $B_{p,q,{\betatilde}}$ as the $\betatilde$ directly impacts the final convergence analysis for the deep neural network estimator. We refer to the Section B.1 for more details of the a-Besov class.

We make the following assumptions on the structure of the functions.
\begin{assumption}[Smoothness of representation functions]
\label{assumption_repre_aBesov}
All  elements of representation functions
$R_c: \bR^d \to \bR^{r_c}$,
$R_t: \bR^d \to \bR^{r_t}$,
$R_0: \bR^d \to \bR^{r_0}$,
belong to  a-Besov  spaces.
Specifically,
\begin{itemize}
\setlength{\itemsep}{0pt}
\setlength{\parsep}{0pt}
\setlength{\parskip}{0pt}
\item[(i)]
Write $R_c = [R_{c,1},R_{c,2},\dots,R_{c,r_c}]$.
Assume that $R_{c,k} \in B_{p,q,\betatilde_{c,k}}$
where $\betatilde_{c,k}$ is the  smoothness index, $k=1,\dots,r_c.$
Denote
$\betatilde_{c} = \min\{\betatilde_{c,1},\dots,\betatilde_{c,r_c}\}.$

\item[(ii)]
Write   $R_t = [R_{t,1},R_{t,2},\dots,R_{t,r_t}]$.
Assume that $R_{t,k} \in B_{p,q,\betatilde_{t,k}}$
where $\betatilde_{t,k}$ is the smoothness index, $k=1,\dots,r_c.$
Denote $\betatilde_{t} = \min\{\betatilde_{t,1},\dots,\betatilde_{t,r_t}\}.$

\item[(iii)]
Write $R_0 = [R_{0,1},R_{0,2},\dots,R_{0,r_0}]$.
Assume $R_{0,k} \in B_{p,q,{\betatilde_{0,k}}},$ $k=1,\dots,r_c.$
Denote $\betatilde_{0} = \min\{\betatilde_{0,1},\dots,\betatilde_{0,r_0}\}.$
\end{itemize}
\end{assumption}
Assumption \ref{assumption_repre_aBesov} requires that all these representations belong to  the a-Besov classes with different smoothness parameters.
The a-Besov function space  includes
many popular classes, such as H\"older class and Besov class. Thus, the theoretical studies in our work are general and applicable to various problems.
As we mentioned before, the function $R_t$ usually has a much  simpler structure than  $R_0$, which means  $\betatilde_t> \betatilde_0$.

Since the a-Besov function classes are considered, \citet{suzuki2021deep} proves that the deep neural network from the ReLU neural network function class $\mF$ can effectively learn functions from the a-Besov class where $\mF$ is the  function class of the feed-forward neural network  with the Rectified Linear Unit (ReLU) activation function    \citep{Schmidhuber_2015}.
Let $\mF\equiv \mF(\btheta,\mathcal{H}, \mathcal{W}, \mS)$ be the set of such ReLU neural networks
$R: \bR^d \rightarrow \bR^r$ with weights $\btheta$,  depth   $\mH$,   width $\mW$ and  size $\mS$.
Here the depth $\mH$ refers to the number of hidden layers. A $(\mH+1)$-vector $(w_0, w_1, \ldots, w_{\mH})$ denotes the width of each layer. The width $\mW=\max\{w_1, \ldots, w_{\mH}\}$ is the maximum width of the hidden layers. The size $\mS=\sum_{i=0}^{\mH}[w_i\times w_{i+1}]$ is the total number of parameters in the network. In this paper, we denote $\mF_{R_c}$, $\mF_{R_t}$
as the network classes for learning $R_c$ and  $R_t$, respectively. Then we give the specification as follows.

\begin{assumption}[Neural network classes for learning $R_c$ and $R_t$]
\label{assumption_NNClass}
 Recall that $N$ denotes total sample size on the source datasets.
The neural network classes
$\mF_{R_c}$, $\mF_{R_t}$  for  learning  $R_c$ and $R_t$ is defined as follows,
\begin{itemize}
\setlength{\itemsep}{0pt}
\setlength{\parsep}{0pt}
\setlength{\parskip}{0pt}
    \item[(i)]
    We denote  $\mF_{R_c}$ as the deep neural network class  for learning the $R_c$.
    We set the depth
    $\mH_{R_c} = \mO\big(  \log(d)\log(N)   \big)$,
    width
    $\mW_{R_c} = \mO\big(r_c d N^{1/(2\betatilde_c + 1)}\big)$,
    model size  $\mS_{R_c} = \mO\big(r_cd^2\\N^{1/(2\betatilde_c + 1)} \log(N)\log(d)  \big)$.

    \item[(ii)]
    We denote  $\mF_{R_t}$ as the deep neural network class   for learning the  $R_t$.
    We set the depth
    $\mH_{R_t} = \mO\big(  \log(d)\log(n_0)  \big)$,
    width
    $\mW_{R_t} = \mO\big(r_t d n_0^{1/(2\betatilde_t + 1)}\big)$,
    model size  $\mS_{R_t} = \mO\big(r_td^2\\n_0^{1/(2\betatilde_t + 1)} \log(n_0)\log(d)  \big)$.
\end{itemize}
\end{assumption}
We note that the possible network structures are not unique, and these network structures are not designed to be optimal due to the possible heterogeneity in smoothness among the components of the representation functions.

In the following Lemma \ref{lemma_source_excessrisk}, we give the convergence result for $\widehat R_c$ with the source datasets.
\begin{lemma}[Convergence  result of learning representation on sources]
\label{lemma_source_excessrisk}
Denote $R^*_c$ as a solution of ( \ref{Sec3_obj1}). Set the tuning parameter $\lambda_E,\lambda_Z = \mO(1)$, with  Assumption \ref{assumption_finitexy}-
    \ref{assumption_NNClass},
    we have the excess risk bound for the $\widehat R_c$,
\begin{align*}
    \begin{aligned}
	\mL_{S}(\widehat R_c)-\mL_{S}(R^*_c)
= \widetilde \mO({r_c}^{1/2}N^{-\frac{\betatilde_c}{2\betatilde_c + 1}}).
\end{aligned}
\end{align*}
\end{lemma}
\begin{remark}
The a-Besov space includes the H\"{o}lder class as a special case   \citep{suzuki2021deep}. If the components of $R_c$ are from the H\"{o}lder class with smoothness index $\beta$, then we have $\betatilde_{c,k} = \beta/d$ for $k=1,\dots,r_c$.
Under this condition, the convergence rate derived in Lemma \ref{lemma_source_excessrisk} is $\widetilde \mO({r_c}^{1/2}N^{-\frac{\beta}{2\beta + d}})$.
The convergence rate deteriorates with increasing dimension
d, which is caused by the curse of dimensionality.
\end{remark}

Finally,  we provide the bound for the excess risk of  $[\widehat R_c,\widehat R_t]$ on the target domain $\mD_0$.
\begin{theorem}[Convergence result of TESR on the target domain]
\label{theorem_TransRep}
    Denote $[R^*_c,R^*_t]$ as a solution of expression \eqref{Sec3_obj2}.
    Set the tuning parameter $\lambda_E,\lambda_Z = \mO(1)$, with  Assumption \ref{assumption_finitexy}-\ref{assumption_NNClass} and the conditions in Lemma \ref{lemma_source_excessrisk},
    we have the excess risk bound for the $[\widehat R_c,\widehat R_t]$ on the $\mD_0$,
\begin{align*}
    \begin{aligned}
\mL_{T}([\widehat R_c,\widehat R_t])-\mL_{T}([ R^*_c, R^*_t])
= \widetilde\mO\big({r_t}^{1/2}n_0^{\frac{-\betatilde_t}{2\betatilde_t + 1}}\big) +
\widetilde \mO(r_c^{1/4}N^{-\frac{\betatilde_c/2}{2\betatilde_c + 1}}).
\end{aligned}
\end{align*}
\end{theorem}

Under the same conditions, it can be shown that estimating $R_0$ using only
the target dataset $\mD_0$  leads to the excess risk with order
$\widetilde\mO\big({r_0}^{1/2}n_0^{\frac{-\betatilde_{0}}{2\betatilde_{0} + 1}}\big),$
which can be derived using a similar approach as outlined in Lemma \ref{lemma_source_excessrisk}.
We compare the excess risk of the TESR and the excess risk of the estimated representation
that only uses the target dataset. The ratio of these two excess risks is
\begin{align}
\label{result_excess_trans_supervise}
    \frac{\text{excess risk of TESR}}{\text{excess risk of only using target data}}
    = \widetilde\mO\big(\frac{r_t^{1/2}}{r_0^{1/2}} n_0^{ \frac{\betatilde_0 - \betatilde_t}{(2\betatilde_0 + 1)(2\betatilde_t + 1)}}\big)
    + \widetilde\mO\big(
    \frac{r_c^{1/4}}{r_0^{1/2}}
    N^{\frac{-\betatilde_c/2}{2\betatilde_c + 1}}
    n_0^{\frac{\betatilde_0}{2\betatilde_0 + 1}}   \big).
\end{align}
The second term in expression \eqref{result_excess_trans_supervise} is $o(1)$ with $N \gg n_0$
in the context of transfer learning. So we can focus on the the first term.

The proposed transfer learning method outperforms 
the method without using the source datasets if
$r_t^{1/2}r_0^{-1/2} n_0^{ \frac{\betatilde_0 - \betatilde_t}{(2\betatilde_0 + 1)(2\betatilde_t + 1)}} = o(1)$.
This holds if $\betatilde_t>\betatilde_0.$ This is true if the additional component $R_t$
is smoother than $R_0$, the sufficient representation on the target domain.
We also note that the rate of the first term in \eqref{result_excess_trans_supervise} can be improved if
$r_t/r_0 = o(1)$, which
signifies that the majority of the useful information is already included in $R_c.$
Therefore, the intrinsic dimension of $R_t$ is smaller than that of $R_0$.
Existing transfer learning methods mainly focus on reducing complexity, often defined in terms of smoothness or sparsity   \citep{cai2022transfer,tian2022transfer}.
However, these approaches often overlook the role of knowledge volume, which is typically represented by model size and the number of latent representations.

\section{Simulation Studies}
In this section, we evaluate the performance of TESR. As TESR focuses on constructing data representations rather than making direct predictions, we assess its effectiveness using the results from a predictive model that takes inputs $[\widehat R_c,\widehat R_t]$ generated by TESR. We compare its performance against the following existing methods.

\begin{itemize}
    \item  Deep Neural Network-based Classifier (DNN)   \citep{Schmidhuber_2015}: the classical end-to-end deep neural network based classifier with logistic loss, relying  only on the target dataset.

    \item  Deep Dimension Reduction (DDR)   \citep{huang2020deep}: a supervised learning representation method that learns representations based on deep neural networks and distance covariance, relying solely on the target dataset. Similar to other methods, its performance is assessed using numerical results from a predictive model that takes the learned representation as input.

    \item TransIRM: a \textbf{Trans}fer learning method  based on \textbf{I}nvariant \textbf{R}isk \textbf{M}inimization  \citep{arjovsky2019invariant}. Specifically, we learn the invariant representation from the source datasets using the invariant regularization proposed by  \citet{arjovsky2019invariant}. The invariant representation is fixed and taken as input for an  deep neural network classifier which are trained using the logistic loss on the target dataset.
        A key difference between TESR and the TransIRM is that TESR takes into account the specific information in the target domains while the TransIRM does not.

\end{itemize}

All the experiments in this section were replicated 100 times. These methods mentioned above were implemented using the same network architecture and model size to ensure fair comparisons. The dimension of the representation estimator $\widehat R_c,\widehat R_T$ are set to 32 in all the simulations. Implementation details of these methods are given in Section C of the Supplementary Materials.

\subsection{Example 1: Models with various $(n_s, n_0, d)$}
We evaluate the performance of TESR and compare it with existing methods by considering various combinations of $(n_s, n_0, d)$.
\begin{itemize}
\item Target domain: the target domain model is a binary classification model with $Y_0 \in \{0, 1\}$ and $X_0 \in \mathbb{R}^d.$ The model is
\[
P(Y_0=1|X_0=x) =\frac{\exp [g(x) ]}{1+ \exp[g(x)]},
\]
where
$g(x) = g_0(x) - \mathbb{E}(g_0(X_0))$ with $g_0(x) = 2 f_1(x_1) + f_2(x_2,x_3) + f_3(x_3,x_4) + f_4(x_4,x_5).$ Here the component functions are
$f_1(u)=(u - 0.9)^2$,
$f_2(u,v)=-uv(u-0.5)^2$,
$f_3(u,v)=\sin(-0.2\pi uv) + 1$,
$f_4(u,v)=u(|v|+1)^2$,
$f_5(u)=\sin(0.5\pi u)+1$,
$f_6(u)=2\sin(\pi u)/(2-\sin(\pi u))$.

\item Source domains: source domain models are regression models with $Y_s \in \mathbb{R}$ and $X_s \in \mathbb{R}^d,
s = 1, \ldots, 4.$
The regression functions in the source domains are given below,
\begin{itemize}
\setlength{\itemsep}{0pt}
\setlength{\parsep}{0pt}
\setlength{\parskip}{0pt}
\item $\mD_1:
y = 3 f_1(x_1) + f_2(x_2,x_3)  + f_3(x_3,x_4) + f_5(x_6) + \epsilon_1;$
\item $\mD_2:
y = 3 f_1(x_1) + f_2(x_2,x_3)  + f_3(x_3,x_4) + 2f_5(x_6) + \epsilon_2;$
\item $\mD_3:
y = 2 f_1(x_1) + 1.5f_2(x_2,x_3)  + f_3(x_3,x_4) + f_6(x_7) + \epsilon_3;$
\item $\mD_4:
y = 2 f_1(x_1) + 1.5f_2(x_2,x_3)  + f_3(x_3,x_4) + 2f_6(x_7) + \epsilon_4.$
\end{itemize}
Across all the five domains, the covariates $x \sim N(\boldsymbol{0},\Sigma)$ where $\Sigma_{i,j} = 0.2^{|i-j|}$ for $i,j=1,\dots,d$.
For $s=1,2,3,4$, $\epsilon_s$  are independently drawn from $N(0,0.5^2)$. In Example 1,  $f_1, f_2, f_3$ are the shared components  but their coefficients vary across the five domains.
We note that  $x_5$ is the unique information for the target domain $\mD_0$ and $x_6,x_7$ is only active in the source domains  $\mD_S, s=1,2,3,4$.

\end{itemize}

\begin{figure}[H]
\centering
\includegraphics[width=0.9\textwidth]{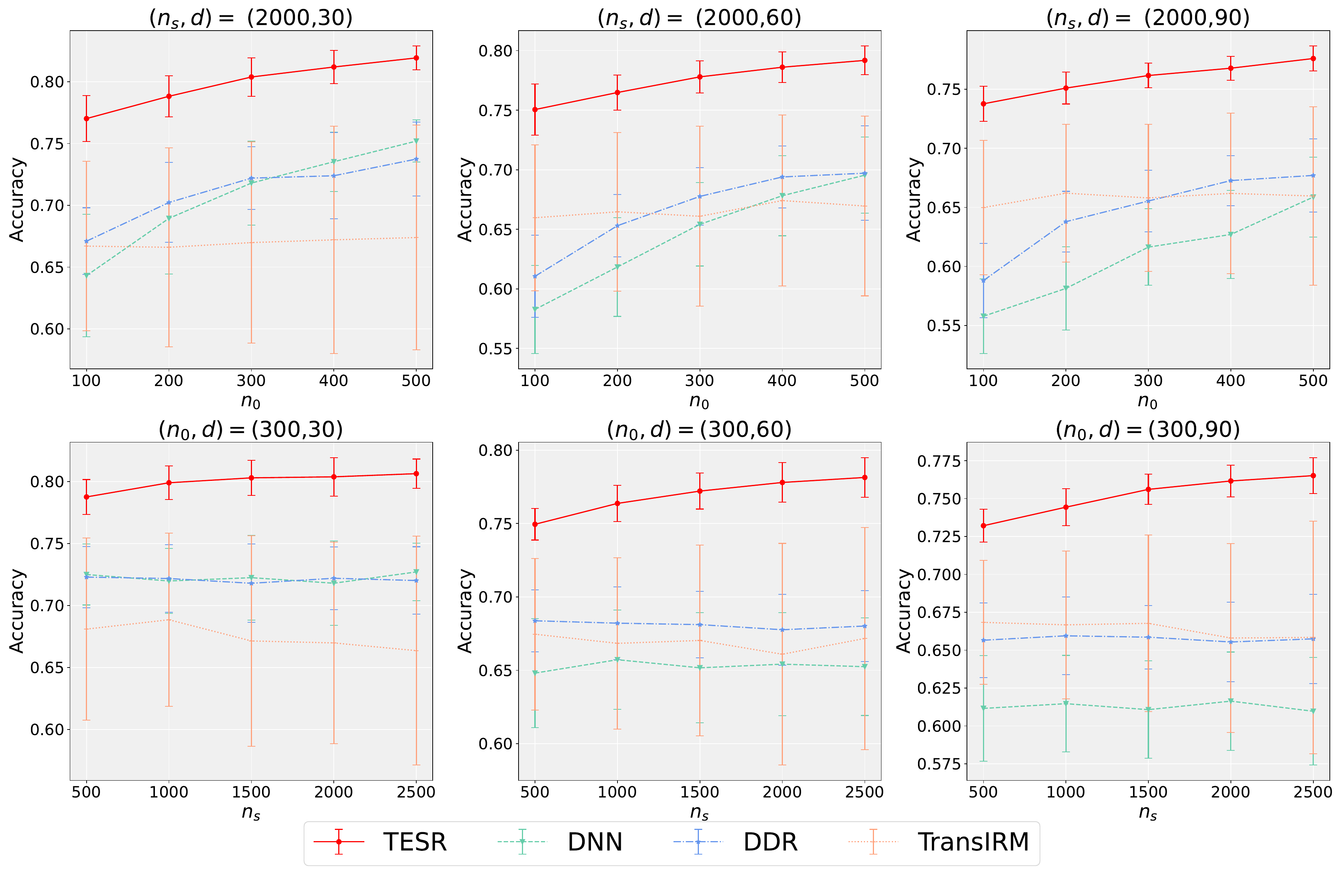}
\caption{Classification accuracy and its standard deviation (represented by the width of the error bar) for the four methods, TESR, DNN, DDR, and TransIRM, are evaluated with various values of
$(n_s,n_0,d)$ over 100 replications in Example 1.
}
\label{Fig_Simu_1}
\end{figure}

Figure \ref{Fig_Simu_1} shows the classification accuracy and its standard deviation across 100 replications. Notably, TESR consistently outperforms all its competitors. As the source sample size
increases, TESR's performance improves significantly, while TransIRM's performance fluctuates. Furthermore, by effectively capturing specific information from the target, TESR becomes significantly more efficient than TransIRM as the target sample size
increases. TransIRM surpasses DDR and DNN when the dimensionality
$d$ is large and the target data size $n_0$ is relatively small. However, it remains significantly less effective compared to TESR.

To further demonstrate the efficiency of TESR, we examine its performance across different  designs of dimensions for the representation functions $[\widehat R_c,\widehat R_t]$  in Table \ref{Table_simu_dimension_change}. TESR maintains stable performance with varying $r^*$ where $r^*=\mathop{dim}(\widehat R_c) = \mathop{dim}(\widehat R_t)$. TESR consistently outperforms all competitors with higher accuracy and smaller standard deviation.

\begin{table}[!htbp]
\centering
\renewcommand{\tabcolsep}{0.6pc}\renewcommand{\arraystretch}{0.5}
    \centering
    \caption{The  classification accurary  and its standard deviation with $(n_s,n_0,d)=(2000,300,60)$ and different  $r^*$, dimensions of the representation estimator modules in Example 1.}
\begin{tabular}{lllll}
\toprule
    $r^*$ & TESR  & DNN & DDR & TransIRM   \\
\midrule
8 & 0.782 (0.015)& 0.661 (0.036) & 0.676 (0.021) & 0.602(0.092)   \\
16 & 0.783 (0.015) & 0.664 (0.032) & 0.673 (0.023) & 0.639(0.086)   \\
32 & 0.784 (0.014)  & 0.652 (0.036)& 0.677 (0.019)  &  0.653(0.075)  \\
64 & 0.783 (0.013) & 0.657 (0.033)& 0.673 (0.027)  & 0.652(0.072)  \\
 \bottomrule
    \end{tabular}
    \label{Table_simu_dimension_change}
\end{table}

Additionally, we also consider a case where the sources and target are all regression problems in  Example S.1 of  the Supplementary Materials.

\subsection{Example 2: Models with two target tasks}

In this section, we examine the performance of the proposed method on two independent target tasks, showing that its effectiveness is independent of regression function similarity.
Specifically, we generate  4 sources and 2 target datasets.

\begin{itemize}
\item Target domains: there are two target domains, both have a binary response variable.
The first target domain model is,
$$
P(Y_{01}=1|X_{01}=x) =\frac{\exp[g_1(x) ]}{1+ \exp[g_1(x)]},
$$
where
$g_1(x) = g_{01}(x) - \mathbb{E}(g_{01}(X_{01})),$ with
$g_{01}(x) = 2 f_1(x_1) + f_2(x_2,x_3) + f_3(x_3,x_4) + f_4(x_4,x_5).$
The second target domain model is
is generated from
$$
P(Y_{02}=1|X_{02}=x) =\frac{\exp[g_2(x) ]}{1+ \exp[g_2(x)]},
$$
where
$g_2(x) =g_{02}(x) - \mathbb{E}(g_{02}(X_{02}))$ with $g_{02}(x) =  2 f_1(x_8) + f_2(x_9,x_{10}) + f_3(x_{10},x_{11}) + f_4(x_{11},x_{12})$.
The two target tasks are independent and depend on  entirely different covariates.
\item Source domains: regression models with $Y_s \in \mathbb{R}$ and $X_s \in \mathbb{R}^d,
s = 1, \ldots, 4.$ The regression functions are given below,
\begin{itemize}
\setlength{\itemsep}{0pt}
\setlength{\parsep}{0pt}
\setlength{\parskip}{0pt}
\item $\mD_1:
y = f_1(x_1) + 2f_2(x_9, x_{10}) + 2f_3(x_{10},x_{11}) +  f_5(x_{6}) + \epsilon_1;$
\item $\mD_2:
y = f_1(x_1) + 2f_2(x_9,x_{10}) + 2f_3(x_{10},x_{11}) +  2f_5(x_{6}) + \epsilon_2;$
\item $\mD_3:
y = f_1(x_8) + 2f_2(x_2,x_{3}) + 2f_3(x_{3},x_{4}) +  2f_5(x_{13}) + \epsilon_3;$
\item $\mD_4:
y = f_1(x_8) + 2f_2(x_2,x_{3}) + 2f_3(x_{3},x_{4}) +  2f_5(x_{13}) + \epsilon_4,$
\end{itemize}
where covariates and the error terms are generated from standard Gaussian distribution and
the component functions the same as that in Example 1.
\end{itemize}

Figure \ref{Fig_Simu_2} reports the classification accuracy on tasks $T_1$, $T_2$ of the 4 methods.
TESR achieves the best accuracy on both target tasks. TESR learns the sufficient and invariant representation $R_c$ from the sources, which effectively captures generalizable features despite the high heterogeneity among the sources. However, in contrast to Example 1,
TransIRM performs consistently worse than the other methods in Example 2. One possible explanation is that  the invariant loss proposed by \citet{arjovsky2019invariant} is unable to capture robust information from these highly heterogeneous sources effectively.
In conclusion,  TESR, which incorporates representations with theoretical guarantees, can draw information from the heterogeneous sources and enhance the performance on different down-streaming tasks.

\begin{figure}[H]
\centering
\includegraphics[width=0.85\textwidth]{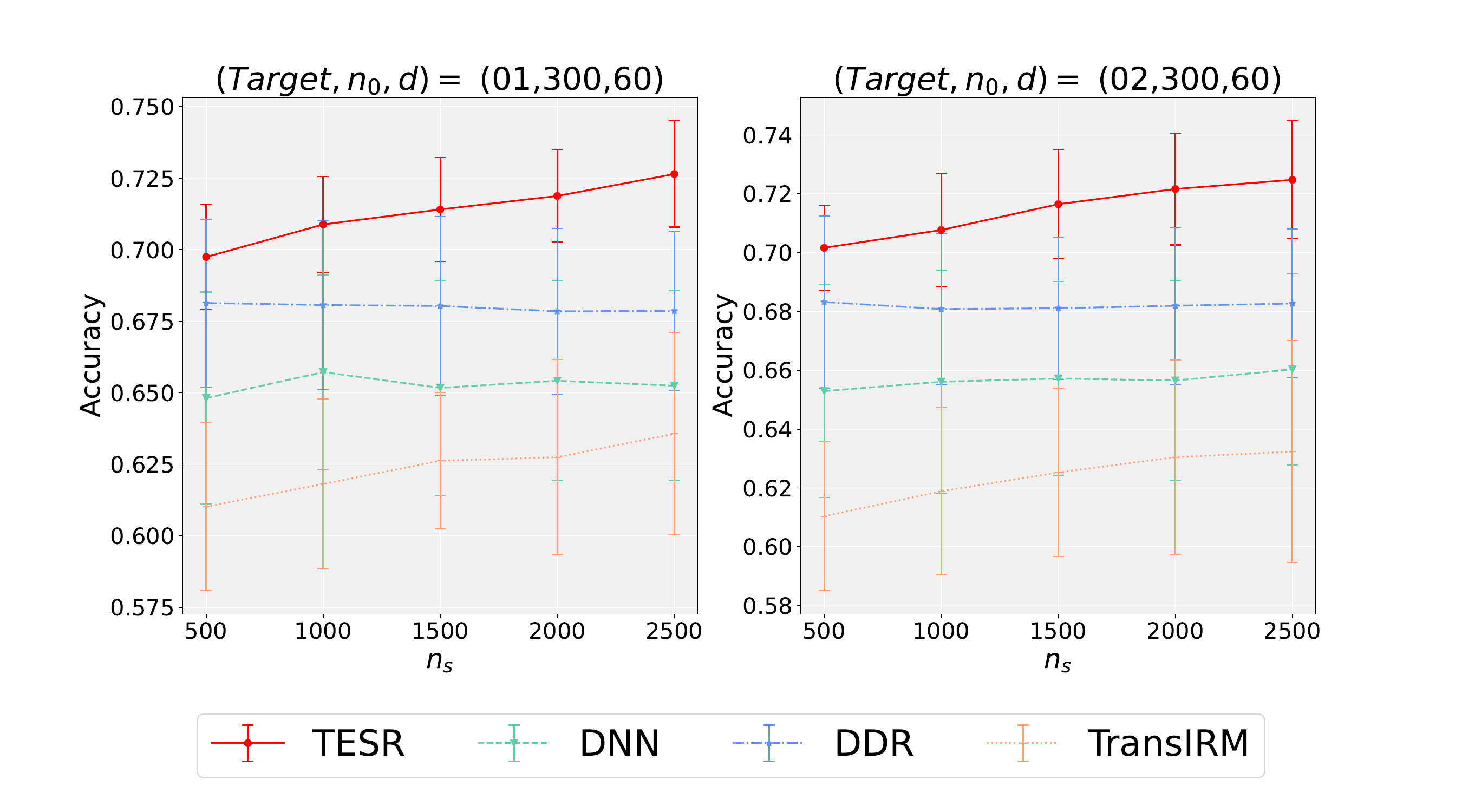}
\caption{Classification accuracy and its standard deviation (width of error bar) on the two independent target tasks in Example 2.}\label{Fig_Simu_2}
\end{figure}

The proposed method can capture the heterogeneous information from sources and enhance the performance of independent target tasks, unlike the conventional transfer learning methods which depend on high degree of similarity on regression function.
Different downstream tasks benefit from the specific components of
the multi-dimensional sufficient representation.
Furthermore, the highly structured features preserved in these representation functions provide valuable insights into the success of pre-trained models across diverse downstream tasks.

\subsection{Example 3: Models with heterogeneity between source and target domains}
 In Example 3, we examine how the performance of TESR is affected by increasing heterogeneity between the source domains and the target domain.

\begin{itemize}
\item Target domain: the target domain model is a binary classification model, which takes the form:
\[
P(Y_0=1|X_0=x) =\frac{\exp [g(x) ]}{1+ \exp[g(x)]},
\]
where
$g(x) = g_0(x) - \mathbb{E}(g_0(X_0))$ with $g_0(x)=2 f_1(x_1) + f_2(x_2,x_3) + f_3(x_3,x_4) + f_4(x_4,x_5).$

\item Source domains: source domains models are regression models:
\begin{align*}
\mD_s:\  y = 2\gamma_{s,1} f_1(x_1) + \gamma_{s,2}\Big[ f_2(x_2,x_3) + f_3(x_3,x_4)\Big] + 2f_5(x_6) + \epsilon_s, \ s = 1,\ldots, 8,
\end{align*}
where $\bx$ is drawn from a standard Gaussian distribution, and the component functions
$f_1, \ldots, f_5$ and error terms are the same as those in Example 1.
The coefficients $(\gamma_{s,1},\gamma_{s,2})$  control the differences
between the functions in $\mD_s$ and $\mD_0.$ When $\gamma_{s,1}=\gamma_{s,2}=1,$
the shared components of target model and source model functions are identical. Any deviation from these values introduces
heterogeneity between target and source models.
\end{itemize}

We consider two types of departures  with coefficients $(\gamma_{s,1},\gamma_{s,2})$.
\begin{itemize}
\setlength{\itemsep}{0pt}
\setlength{\parsep}{0pt}
\setlength{\parskip}{0pt}
    \item For the source index $s=1,\dots, 6$,
    \begin{itemize}
    \setlength{\itemsep}{0pt}
\setlength{\parsep}{0pt}
\setlength{\parskip}{0pt}
        \item Type I (\textbf{$\bf{L_1}$ distance}):
    $\Big(\gamma_{s,1},\gamma_{s,2}\Big) = \Big(1+0.5s,1+0.5s\Big)$;
\item Type II (\textbf{cosine distance}):
$\Big(\gamma_{s,1},\gamma_{s,2}\Big) = \Big(\cos(s\pi/3) - \sin(s\pi/3), \cos(s\pi/3) + \sin(s\pi/3)\Big)$;
    \end{itemize}
\item For the source index $s=7,8$, in both Type I and II: $\Big(\gamma_{s,1},\gamma_{s,2}\Big)= (0,0).$
\end{itemize}
Under the Type I departure, the $L_1$ distance among the regression functions from $\mD_s$ and $\mD_0$ increases with $s$ \citep{cai2022transfer}. Meanwhile, under the Type II departure, the cosine distance  changes with $s$ \citep{gu2022robust}. Details of the  $L_1$  and  cosine distances between  the regression functions  from $\mD_s$ and $\mD_0$ are reported in Section C of Supplementary Materials.
In short, the departure on regression functions of $\mD_s$ from $\mD_0$ increases with  $s$.
Notably, the datasets $\mD_7$ and  $\mD_8$
are entirely redundant and offer no benefit for the target.

\begin{figure}[H]
\centering
\includegraphics[width=0.85\textwidth]{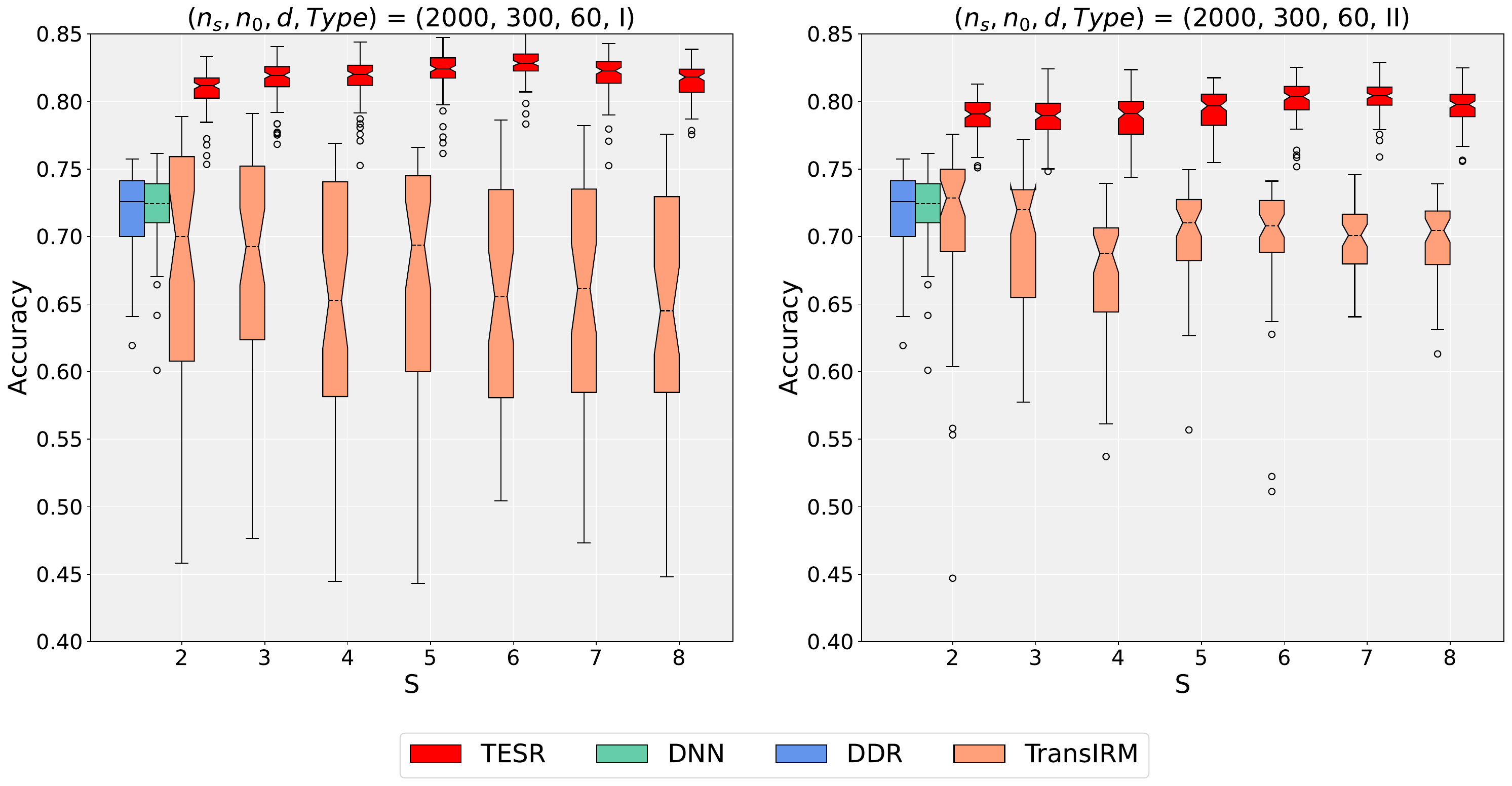}
\caption{The box plot illustrates classification accuracy over 100 replications on the target dataset  in Example 3, under $L_1$  (left panel) and cosine distance departure (right panel) conditions, where
source datasets are sequentially added into the modeling process. TESR (notched box with solid line) outperforms TransIRM (notched box with solid line), DDR (rectangular box with solid line), and DNN (rectangular box with dashed line). DDR and DNN are only presented once, as they do not utilize the source information.
}\label{Fig_Simu_3}
\end{figure}

In Example 3, we sequentially incorporate source datasets into the model and evaluate the performance of four methods. Figure \ref{Fig_Simu_3} displays the classification accuracy over 100 replications. TESR consistently outperforms the other methods in both scenarios, with its performance improving as additional sources ($s=2,\ldots, 6$) are included. However, as more sources are added, the heterogeneity among them increases, making it more challenging for TransIRM to capture invariant features. Despite this, TESR maintains satisfactory performance and proves particularly advantageous for analyzing large-scale datasets, where eliminating redundant sources using complex source selection algorithms can be computationally intensive \citep{cai2022transfer,li2022transfer}.

We take a closer look at the Type I departure case.
For $s=1, \ldots, 6$,  the coefficients $(\gamma_{s,1},\gamma_{s,2})$ of the components $f_1(x_1)$ and $\Big( f_2(x_2,x_3) + f_3(x_3,x_4)\Big)$ increase with $s$.
As $s$ increases, the $L_1$ distance between the regression functions from $\mD_s$ and $\mD_0$  increases,  leading to larger discrepancy  between  their regression functions \citep{li2022transfer,tian2022transfer}.
Conventional transfer learning methods regard these sources with large distance from the target task $\mD_0$ as redundant and introduce sophisticated source selection algorithms to remove them.
However, large coefficients $(\gamma_{s,1},\gamma_{s,2})$ lead to a higher signal-to-noise ratio on sources which is expected to facilitate the learning of representations.
As shown in the left panel of Figure \ref{Fig_Simu_3},  the classification accuracy of TESR improves as additional sources are included.
A similar conclusion can also be drawn from the cases under Type II departures as shown in the right panel of Figure \ref{Fig_Simu_3}, although $(\gamma_{s,1},\gamma_{s,2})$
may vary between positive and negative under Type II  departures   \citep{gu2022robust}.

In conclusion, even though the source domain regression functions with diverse coefficients
$(\gamma_{s,1},\gamma_{s,2})$ differ significantly from the target domain function for classification, they still offer valuable knowledge for the target task. TESR, by achieving knowledge transfer through representation functions, is capable of capturing valuable information from these heterogeneous sources.

\section{Real data examples}
In this section, we evaluate TESR in terms of prediction and classification performance on two datasets, comparing it with existing methods: DDN, DDR, and TransIRM. The datasets used for this evaluation include a gene expression dataset and an image dataset.

\subsection{Prediction of JAM2 gene expression}
Primary familial brain calcification (PFBC) is an infrequent, autosomal dominant neurological disorder, distinguished by the presence of bilateral calcifications within the basal ganglia and additional cerebral areas. Recent studies have identified the gene JAM2 as a novel causative gene of autosomal recessive PFBC \citep{cen2020biallelic,schottlaender2020bi}.
JAM2 encodes junctional adhesion molecule 2, which is highly expressed in neurovascular unit-related cell types, such as endothelial cells and astrocytes, and is predominantly localized on the plasma membrane. \citet{schottlaender2020bi} have illustrated that specific genetic variants result in diminished levels of JAM2 mRNA expression and an absence of the JAM2 protein in fibroblasts derived from patients, aligning with a loss-of-function mechanism.
Consequently, predicting the expression level of JAM2 in target brain tissues is of significant interest.

We utilize TESR to construct models for each tissue to predict the expression level of JAM2, using data from the Genotype-Tissue Expression (GTEx) project \citep{gtex2015}.
In the GTEx project, genes related to the central nervous system are organized into MODULE 137, which includes 13 tissues and a total of 545 genes, along with an additional 1,632 genes that are significantly enriched in the same experiments as the module's genes.
Table \ref{table:samplesizeJAM2} provides the sample sizes of the datasets across the various tissues.
These datasets are available at  \url{https://gtexportal.org/home/}.

\begin{table}[!htbp]
\renewcommand{\tabcolsep}{0.8pc}\renewcommand{\arraystretch}{0.8}
\centering
\caption{The sample size of the datasets across the tissues}
\label{table:samplesizeJAM2}
\begin{tabular}{l|l|l}
\toprule
Target task&Full name & Sample size \\
\midrule
Amygdala&Brain\_Amygdala & 152 \\
Anterior&Brain\_Anterior\_cingulate\_cortex\_BA24 & 176 \\
Caudate&Brain\_Caudate\_basal\_ganglia & 246 \\
Cerebellar&Brain\_Cerebellar\_Hemisphere & 215 \\
Cerebellum&Brain\_Cerebellum & 241 \\
Cortex &Brain\_Cortex & 255 \\
Frontal&Brain\_Frontal\_Cortex\_BA9 & 209 \\
Hippocampus&Brain\_Hippocampus & 197 \\
Hypothalamus&Brain\_Hypothalamus & 202 \\
Nucleus&Brain\_Nucleus\_accumbens\_basal\_ganglia & 246 \\
Putamen&Brain\_Putamen\_basal\_ganglia & 205 \\
Spinal&Brain\_Spinal\_cord\_cervical\_c-1 & 159 \\
Substantia&Brain\_Substantia\_nigra & 139 \\
\midrule
Average& & 203\\
\bottomrule
\end{tabular}
\end{table}

To predict the expression level of JAM2 in a specific tissue, we utilize data from other tissues as source datasets. For each target tissue dataset, we employ 5-fold cross-validation, allocating 60\% of the data for training, 20\% for evaluation, and 20\% for testing. The source datasets are divided into training and evaluation sets in a 4:1 ratio. We compare the results with existing methods through cross-validation.

\begin{figure}[H]
\centering
\includegraphics[width=0.85\textwidth,trim = 20 80 50 20,clip]{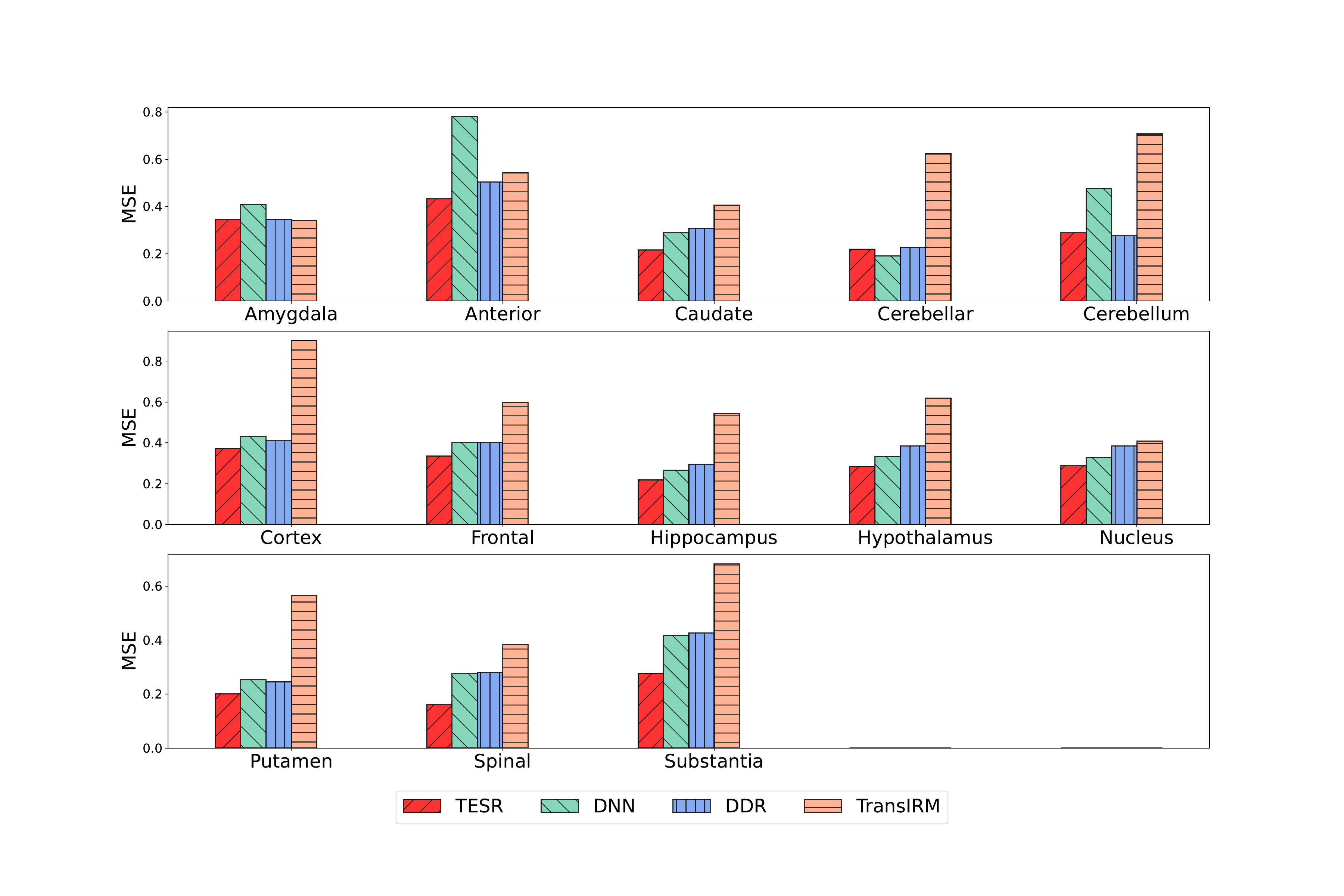}
\caption{The Mean Square Error (MSE)  averaged from 5-fold cross-validation on 13 different tissue groups. Each groups are denoted with their abbreviations.}\label{case:JAM1}
\end{figure}

The average MSE from 5-fold cross-validation is presented in Figure \ref{case:JAM1}. TESR achieves the best performance in 10 out of the 13 tissues when compared to all other methods. In the \textit{Amygdala},
\textit{Cerebellar}
and
\textit{Cerebellum} tissues, where TESR does not outperform the others, its performance remains competitive. The variations in performance across tissues may be attributed to the heterogeneity of information in the distributions of different tissues. Based on this experiment, we conclude that TESR effectively learns knowledge from relevant tasks, demonstrating its potential for broad applicability. The superiority of TESR over TransIRM also highlights the importance of adapting specific information from the target domain in transfer learning.

\subsection{Image classification using PACS dataset}
We use TESR to construct image classification models for each style within the PACS dataset,
which includes images in four distinct styles: Photo (P), Art painting (A), Cartoon (C), and Sketch (S)
 \citep{li2017deeper}.
Each style contains images across seven categories, as shown in the right panel of Figure \ref{real_data}. For this purpose, we utilize the CIFAR-10 dataset \citep{krizhevsky2009learning} as the source and each of the four styles of the PACS dataset  as the target task. The source dataset comprises 10 classes, while the target datasets consist of 7 classes. As a result, conventional transfer learning methods that require the same data structure for source and target datasets are not well-suited to this scenario.

\bigskip
\begin{figure}[!htbp]
\centering
\includegraphics[width=2.5 in, height=2.0 in]{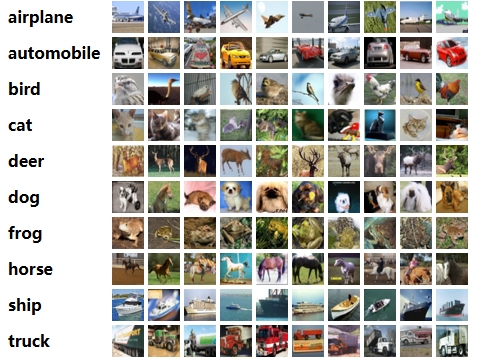}\quad
\includegraphics[width=2.5 in, height=2.0 in]{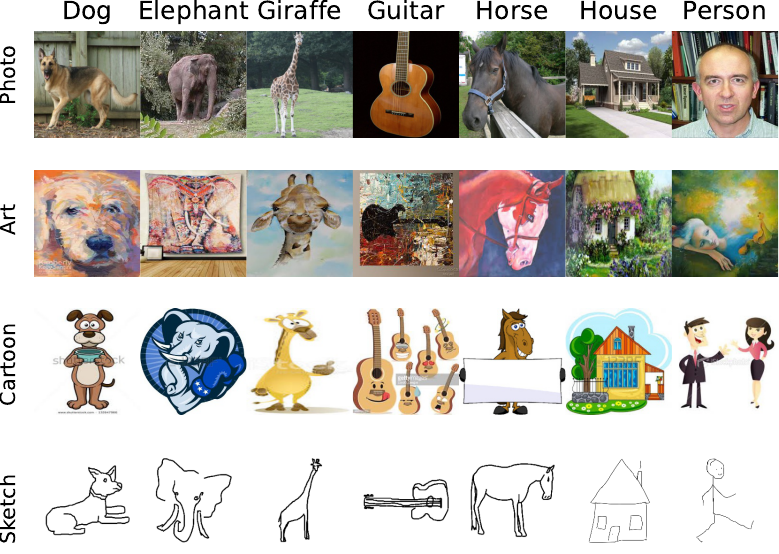}
\caption{Examples from the CIFAR-10 (Left) and PACS (Right) datasets.
We consider transferring the knowledge in CIFAR-10 to 4 domains of the PACS dataset. }\label{real_data}
\end{figure}

\begin{table}[H]
    \centering
    \caption{The accuracy of classification and its standard deviation (in brackets) on different targets.}
\begin{tabular}{lllll}
    \toprule
     & TESR & TransIRM & DDR & DNN  \\
    \midrule
    Photo & \textbf{0.700 (0.038)} & 0.533 (0.062) & 0.634 (0.035) & 0.617 (0.030)  \\
    Art & \textbf{0.488 (0.063)} & 0.401 (0.043) & 0.336 (0.024) & 0.411 (0.031) \\
    Cartoon & \textbf{0.720 (0.012)} & 0.430 (0.089) & 0.672 (0.040) & 0.669 (0.029)\\
    Sketch & 0.620 (0.047) & 0.414 (0.101) & 0.594 (0.037) & \textbf{0.742 (0.024)}\\
    \bottomrule
    \end{tabular}
    \label{Table:pacs}
\end{table}

We use 5-fold cross-validation, allocating 60\% of the target datasets for training, 20\% for evaluation, and 20\% for testing. For the source datasets, we use 80\% for training and 20\% for evaluation. The results are presented in Table \ref{Table:pacs}. TESR outperforms TransIRM and DDR across the four different target domains and achieves superior performance compared to DNN, except for the Sketch data. One possible explanation is that the Sketch data contains fewer details and has a simpler structure, which can be effectively captured by simpler models, rendering the texture information from other domains less useful.

\section{Conclusion and Discussion}
In this paper, we introduce TESR, a transfer learning framework that leverages sufficient and invariant data representations to enhance learning across different domains.  The core of our approach involves learning data representations from source tasks and enhancing them with an additional component designed to extract specific information for the target task. Unlike traditional transfer learning methods that rely heavily on critical source selection algorithms to ensure close alignment between source and target datasets, TESR allows the sufficient and invariant  representations from the sources to be only partially relevant to, or even independent of, the target task.  The augmenting component is capable of extracting relevant knowledge for the target domain. This flexibility enhances the generality and robustness of transfer learning, making it particularly effective for complex datasets where validating task similarities is infeasible or costly.

Several questions warrant further investigation. First, developing a data-driven approach based on cross-validation for selecting the dimensions of Sufficient and Invariant Representations (SIRep) for the source domains, as well as the augmenting representations from the target domain, would be beneficial. Second, exploring alternative measures for conditional independence, such as mutual information, could provide deeper insights and enhance the framework's effectiveness. Third, it would be valuable to develop alternative methods for adapting the SIRep from the source domains to the target domain. For instance, while the proposed method involves enhancing SIRep with an independent component estimated from the target data, exploring other strategies could be fruitful. We  leave these questions for future work.

\bibliographystyle{chicago}
\bibliography{TESRarxiv250210}

\begin{thebibliography}{}

\bibitem[\protect\citeauthoryear{Anthony and Bartlett}{Anthony and
  Bartlett}{2009}]{anthony2009neural}
Anthony, M. and P.~L. Bartlett (2009).
\newblock {\em Neural Network Learning: Theoretical Foundations}.
\newblock cambridge University Press.

\bibitem[\protect\citeauthoryear{Arjovsky, Bottou, Gulrajani, and
  Lopez-Paz}{Arjovsky et~al.}{2019}]{arjovsky2019invariant}
Arjovsky, M., L.~Bottou, I.~Gulrajani, and D.~Lopez-Paz (2019).
\newblock Invariant risk minimization.
\newblock {\em arXiv preprint arXiv:1907.02893\/}.

\bibitem[\protect\citeauthoryear{Bartlett, Harvey, Liaw, and
  Mehrabian}{Bartlett et~al.}{2019}]{bartlett2019}
Bartlett, P.~L., N.~Harvey, C.~Liaw, and A.~Mehrabian (2019).
\newblock Nearly-tight vc-dimension and pseudodimension bounds for piecewise
  linear neural networks.
\newblock {\em Journal of Machine Learning Research\/}~{\em 20}, 1--17.

\bibitem[\protect\citeauthoryear{Bastani}{Bastani}{2021}]{bastani2021predicting}
Bastani, H. (2021).
\newblock Predicting with proxies: Transfer learning in high dimension.
\newblock {\em Management Science\/}~{\em 67\/}(5), 2964--2984.

\bibitem[\protect\citeauthoryear{Bengio, Courville, and Vincent}{Bengio
  et~al.}{2013}]{bengio2013representation}
Bengio, Y., A.~Courville, and P.~Vincent (2013).
\newblock Representation learning: A review and new perspectives.
\newblock {\em IEEE Transactions on Pattern Analysis and Machine
  Intelligence\/}~{\em 35\/}(8), 1798--1828.

\bibitem[\protect\citeauthoryear{Cai and Pu}{Cai and
  Pu}{2022}]{cai2022transfer}
Cai, T.~T. and H.~Pu (2022).
\newblock Transfer learning for nonparametric regression: Non-asymptotic
  minimax analysis and adaptive procedure.
\newblock {\em arXiv preprint arXiv:2401.12272\/}.

\bibitem[\protect\citeauthoryear{Cai and Wei}{Cai and
  Wei}{2021}]{tony2021transClass}
Cai, T.~T. and H.~Wei (2021).
\newblock Transfer learning for nonparametric classification: minimax rate and
  adaptive classifier.
\newblock {\em Annals of Statistics\/}~{\em 49\/}(1), 100--128.

\bibitem[\protect\citeauthoryear{Cen, Chen, Chen, Wang, Yang, Zhang, Wu, Wang,
  Tang, Ye, et~al.}{Cen et~al.}{2020}]{cen2020biallelic}
Cen, Z., Y.~Chen, S.~Chen, H.~Wang, D.~Yang, H.~Zhang, H.~Wu, L.~Wang, S.~Tang,
  J.~Ye, et~al. (2020).
\newblock Biallelic loss-of-function mutations in jam2 cause primary familial
  brain calcification.
\newblock {\em Brain\/}~{\em 143\/}(2), 491--502.

\bibitem[\protect\citeauthoryear{Chen, Jiao, Qiu, and Yu}{Chen
  et~al.}{2024}]{YuGMDD2024}
Chen, Y., Y.~Jiao, R.~Qiu, and Z.~Yu (2024).
\newblock {Deep nonlinear sufficient dimension reduction}.
\newblock {\em Annals of Statistics\/}~{\em 52\/}(3), 1201 -- 1226.

\bibitem[\protect\citeauthoryear{Consortium, Ardlie, Deluca, Segr{\`e},
  Sullivan, Young, Gelfand, Trowbridge, Maller, Tukiainen, et~al.}{Consortium
  et~al.}{2015}]{gtex2015}
Consortium, G., K.~G. Ardlie, D.~S. Deluca, A.~V. Segr{\`e}, T.~J. Sullivan,
  T.~R. Young, E.~T. Gelfand, C.~A. Trowbridge, J.~B. Maller, T.~Tukiainen,
  et~al. (2015).
\newblock The genotype-tissue expression (gtex) pilot analysis: multitissue
  gene regulation in humans.
\newblock {\em Science\/}~{\em 348\/}(6235), 648--660.

\bibitem[\protect\citeauthoryear{Cook and Ni}{Cook and
  Ni}{2005}]{cook2005sufficient}
Cook, R.~D. and L.~Ni (2005).
\newblock Sufficient dimension reduction via inverse regression: A minimum
  discrepancy approach.
\newblock {\em Journal of the American Statistical Association\/}~{\em
  100\/}(470), 410--428.

\bibitem[\protect\citeauthoryear{De~la Pena and Gin{\'e}}{De~la Pena and
  Gin{\'e}}{2012}]{de2012decoupling}
De~la Pena, V. and E.~Gin{\'e} (2012).
\newblock {\em Decoupling: from dependence to independence}.
\newblock Springer Science \& Business Media.

\bibitem[\protect\citeauthoryear{Gretton, Borgwardt, Rasch, Sch{\"o}lkopf, and
  Smola}{Gretton et~al.}{2012}]{gretton2012kernel}
Gretton, A., K.~M. Borgwardt, M.~J. Rasch, B.~Sch{\"o}lkopf, and A.~Smola
  (2012).
\newblock A kernel two-sample test.
\newblock {\em Journal of Machine Learning Research\/}~{\em 13\/}(1), 723--773.

\bibitem[\protect\citeauthoryear{Gu, Han, and Duan}{Gu
  et~al.}{2024}]{gu2022robust}
Gu, T., Y.~Han, and R.~Duan (2024).
\newblock Robust angle-based transfer learning in high dimensions.
\newblock {\em Journal of the Royal Statistical Society Series B: Statistical
  Methodology\/}, qkae111.

\bibitem[\protect\citeauthoryear{Hu, Zhao, Yi, Yao, Hong, Sun, and Chi}{Hu
  et~al.}{2022}]{hu2022improving}
Hu, Z., Z.~Zhao, X.~Yi, T.~Yao, L.~Hong, Y.~Sun, and E.~Chi (2022).
\newblock Improving multi-task generalization via regularizing spurious
  correlation.
\newblock {\em Advances in Neural Information Processing Systems\/}~{\em 35},
  11450--11466.

\bibitem[\protect\citeauthoryear{Huang, Jiao, Liao, Liu, and Yu}{Huang
  et~al.}{2024}]{huang2020deep}
Huang, J., Y.~Jiao, X.~Liao, J.~Liu, and Z.~Yu (2024).
\newblock Deep dimension reduction for supervised representation learning.
\newblock {\em IEEE Transactions on Information Theory\/}~{\em 70\/}(5),
  3583--3598.

\bibitem[\protect\citeauthoryear{Huo and Sz{\'e}kely}{Huo and
  Sz{\'e}kely}{2016}]{huo2016fast}
Huo, X. and G.~J. Sz{\'e}kely (2016).
\newblock Fast computing for distance covariance.
\newblock {\em Technometrics\/}~{\em 58\/}(4), 435--447.

\bibitem[\protect\citeauthoryear{Jiao, Lin, Luo, and Yang}{Jiao
  et~al.}{2024}]{jiao2024deepTrans}
Jiao, Y., H.~Lin, Y.~Luo, and J.~Z. Yang (2024).
\newblock Deep transfer learning: Model framework and error analysis.
\newblock {\em arXiv preprint arXiv:2410.09383\/}.

\bibitem[\protect\citeauthoryear{Krizhevsky}{Krizhevsky}{2009}]{krizhevsky2009learning}
Krizhevsky, A. (2009).
\newblock Learning multiple layers of features from tiny images.
\newblock {\em Master's thesis, University of Toronto\/}.

\bibitem[\protect\citeauthoryear{Lee, Li, and Chiaromonte}{Lee
  et~al.}{2013}]{lee2013general}
Lee, K.-Y., B.~Li, and F.~Chiaromonte (2013).
\newblock {A general theory for nonlinear sufficient dimension reduction:
  Formulation and estimation}.
\newblock {\em Annals of Statistics\/}~{\em 41\/}(1), 221 -- 249.

\bibitem[\protect\citeauthoryear{Li, Yang, Song, and Hospedales}{Li
  et~al.}{2017}]{li2017deeper}
Li, D., Y.~Yang, Y.-Z. Song, and T.~M. Hospedales (2017).
\newblock Deeper, broader and artier domain generalization.
\newblock In {\em Proceedings of the IEEE international conference on computer
  vision}, pp.\  5542--5550.

\bibitem[\protect\citeauthoryear{Li}{Li}{1991}]{li1991sliced}
Li, K.-C. (1991).
\newblock Sliced inverse regression for dimension reduction.
\newblock {\em Journal of the American Statistical Association\/}~{\em
  86\/}(414), 316--327.

\bibitem[\protect\citeauthoryear{Li, Cai, and Li}{Li
  et~al.}{2022}]{li2022transfer}
Li, S., T.~T. Cai, and H.~Li (2022).
\newblock Transfer learning for high-dimensional linear regression: Prediction,
  estimation and minimax optimality.
\newblock {\em Journal of the Royal Statistical Society Series B: Statistical
  Methodology\/}~{\em 84\/}(1), 149--173.

\bibitem[\protect\citeauthoryear{Ma and Zhu}{Ma and
  Zhu}{2013}]{ma2013efficient}
Ma, Y. and L.~Zhu (2013).
\newblock Efficient estimation in sufficient dimension reduction.
\newblock {\em Annals of Statistics\/}~{\em 41\/}(1), 250.

\bibitem[\protect\citeauthoryear{Maurer, Pontil, and Romera-Paredes}{Maurer
  et~al.}{2016}]{maurer2016benefit}
Maurer, A., M.~Pontil, and B.~Romera-Paredes (2016).
\newblock The benefit of multitask representation learning.
\newblock {\em Journal of Machine Learning Research\/}~{\em 17\/}(81), 1--32.

\bibitem[\protect\citeauthoryear{Neyshabur, Sedghi, and Zhang}{Neyshabur
  et~al.}{2020}]{neyshabur2020being}
Neyshabur, B., H.~Sedghi, and C.~Zhang (2020).
\newblock What is being transferred in transfer learning?
\newblock {\em Advances in Neural Information Processing Systems\/}~{\em 33},
  512--523.

\bibitem[\protect\citeauthoryear{Pan and Yang}{Pan and
  Yang}{2009}]{pan2009survey}
Pan, S.~J. and Q.~Yang (2009).
\newblock A survey on transfer learning.
\newblock {\em IEEE Transactions on Knowledge and Data Engineering\/}~{\em
  22\/}(10), 1345--1359.

\bibitem[\protect\citeauthoryear{Rizzo and Székely}{Rizzo and
  Székely}{2016}]{rs2016}
Rizzo, M.~L. and G.~J. Székely (2016).
\newblock Energy distance.
\newblock {\em Wiley Interdisciplinary Reviews: Computational
  Statistics\/}~{\em 8}, 27--38.

\bibitem[\protect\citeauthoryear{Schmidhuber}{Schmidhuber}{2015}]{Schmidhuber_2015}
Schmidhuber, J. (2015).
\newblock Deep learning in neural networks: An overview.
\newblock {\em Neural Networks\/}~{\em 61}, 85–117.

\bibitem[\protect\citeauthoryear{Schmidt-hieber}{Schmidt-hieber}{2020}]{schmidt2020nonparametric}
Schmidt-hieber, J. (2020).
\newblock Nonparametric regression using deep neural networks with relu
  activation function.
\newblock {\em Annals of Statistics\/}~{\em 48\/}(4), 1875--1897.

\bibitem[\protect\citeauthoryear{Schottlaender, Abeti, Jaunmuktane, Macmillan,
  Chelban, O’callaghan, McKinley, Maroofian, Efthymiou, Athanasiou-Fragkouli,
  et~al.}{Schottlaender et~al.}{2020}]{schottlaender2020bi}
Schottlaender, L.~V., R.~Abeti, Z.~Jaunmuktane, C.~Macmillan, V.~Chelban,
  B.~O’callaghan, J.~McKinley, R.~Maroofian, S.~Efthymiou,
  A.~Athanasiou-Fragkouli, et~al. (2020).
\newblock Bi-allelic jam2 variants lead to early-onset recessive primary
  familial brain calcification.
\newblock {\em The American Journal of Human Genetics\/}~{\em 106\/}(3),
  412--421.

\bibitem[\protect\citeauthoryear{Sheng and Yin}{Sheng and
  Yin}{2013}]{sheng2013direction}
Sheng, W. and X.~Yin (2013).
\newblock Direction estimation in single-index models via distance covariance.
\newblock {\em Journal of Multivariate Analysis\/}~{\em 122}, 148--161.

\bibitem[\protect\citeauthoryear{Sheng and Yin}{Sheng and
  Yin}{2016}]{sheng2016sufficient}
Sheng, W. and X.~Yin (2016).
\newblock Sufficient dimension reduction via distance covariance.
\newblock {\em Journal of Computational and Graphical Statistics\/}~{\em
  25\/}(1), 91--104.

\bibitem[\protect\citeauthoryear{Suzuki}{Suzuki}{2019}]{suzuki2018adaptivity}
Suzuki, T. (2019).
\newblock Adaptivity of deep {ReLU} network for learning in {Besov} and mixed
  smooth {Besov} spaces: optimal rate and curse of dimensionality.
\newblock In {\em International Conference on Learning Representations}.

\bibitem[\protect\citeauthoryear{Suzuki and Nitanda}{Suzuki and
  Nitanda}{2021}]{suzuki2021deep}
Suzuki, T. and A.~Nitanda (2021).
\newblock Deep learning is adaptive to intrinsic dimensionality of model
  smoothness in anisotropic besov space.
\newblock {\em Advances in Neural Information Processing Systems\/}~{\em 34},
  3609--3621.

\bibitem[\protect\citeauthoryear{Sz{\'e}kely and Rizzo}{Sz{\'e}kely and
  Rizzo}{2009}]{szekely2009dCov}
Sz{\'e}kely, G.~J. and M.~L. Rizzo (2009).
\newblock Brownian distance covariance.
\newblock {\em Annals of Applied Statistics\/}~{\em 3\/}(4), 1236--1265.

\bibitem[\protect\citeauthoryear{Sz{\'e}kely, Rizzo, and Bakirov}{Sz{\'e}kely
  et~al.}{2007}]{szekely2007measuring}
Sz{\'e}kely, G.~J., M.~L. Rizzo, and N.~K. Bakirov (2007).
\newblock {Measuring and testing dependence by correlation of distances}.
\newblock {\em Annals of Statistics\/}~{\em 35\/}(6), 2769 -- 2794.

\bibitem[\protect\citeauthoryear{Tian and Feng}{Tian and
  Feng}{2023}]{tian2022transfer}
Tian, Y. and Y.~Feng (2023).
\newblock Transfer learning under high-dimensional generalized linear models.
\newblock {\em Journal of the American Statistical Association\/}~{\em
  118\/}(544), 2684--2697.

\bibitem[\protect\citeauthoryear{Torrey and Shavlik}{Torrey and
  Shavlik}{2010}]{torrey2010transfer}
Torrey, L. and J.~Shavlik (2010).
\newblock Transfer learning.
\newblock In {\em Handbook of research on machine learning applications and
  trends: algorithms, methods, and techniques}, pp.\  242--264. IGI global.

\bibitem[\protect\citeauthoryear{Weiss, Khoshgoftaar, and Wang}{Weiss
  et~al.}{2016}]{weiss2016survey}
Weiss, K., T.~M. Khoshgoftaar, and D.~Wang (2016).
\newblock A survey of transfer learning.
\newblock {\em Journal of Big data\/}~{\em 3}, 1--40.

\bibitem[\protect\citeauthoryear{Xu, Zhu, and Fan}{Xu
  et~al.}{2022}]{xu2022distributed}
Xu, K., L.~Zhu, and J.~Fan (2022).
\newblock Distributed sufficient dimension reduction for heterogeneous massive
  data.
\newblock {\em Statistica Sinica\/}~{\em 32}, 2455--2476.

\bibitem[\protect\citeauthoryear{Yang and Lin}{Yang and
  Lin}{2018}]{yang2018rsg}
Yang, T. and Q.~Lin (2018).
\newblock {RSG}: Beating subgradient method without smoothness and strong
  convexity.
\newblock {\em Journal of Machine Learning Research\/}~{\em 19\/}(6), 1--33.

\bibitem[\protect\citeauthoryear{Zhu, Zhu, and Feng}{Zhu
  et~al.}{2010}]{zhu2010dimension}
Zhu, L.-P., L.-X. Zhu, and Z.-H. Feng (2010).
\newblock Dimension reduction in regressions through cumulative slicing
  estimation.
\newblock {\em Journal of the American Statistical Association\/}~{\em
  105\/}(492), 1455--1466.

\bibitem[\protect\citeauthoryear{Zhuang, Qi, Duan, Xi, Zhu, Zhu, Xiong, and
  He}{Zhuang et~al.}{2020}]{zhuang2020comprehensive}
Zhuang, F., Z.~Qi, K.~Duan, D.~Xi, Y.~Zhu, H.~Zhu, H.~Xiong, and Q.~He (2020).
\newblock A comprehensive survey on transfer learning.
\newblock {\em Proceedings of the IEEE\/}~{\em 109\/}(1), 43--76.

\end{thebibliography}

\newpage
\appendix
\setcounter{section}{0}
\renewcommand{\thesection}{\Alph{section}}
\setcounter{equation}{0}
\renewcommand{\theequation}{A.\arabic{equation}}
\setcounter{page}{1}
\renewcommand{\thepage}{A.\arabic{page}}

\def\thetable{S\arabic{table}}
\def\thefigure{S\arabic{figure}}

\renewcommand{\thetable}{S\arabic{table}}
\renewcommand{\thefigure}{S\arabic{figure}}
\renewcommand{\theequation}{\thesection.\arabic{equation}}
\renewcommand{\thetheorem}{\thesection.\arabic{theorem}}
\renewcommand{\theassumption}{\thesection.\arabic{assumption}}

\begin{center}
\textbf{\LARGE Appendix}
\end{center}

The appendix contains additional
technical details, further numerical results, and implementation specifications.

\section{Additional technical details}

\subsection{Anisotropic Besov Space Function Class}
We now present the definition of the anisotropic Besov space  \citep[a-Besov]{suzuki2021deep}.
The a-Besov space is denoted  as
$B^{\bbeta}_{p,q,\betatilde}(\Omega)$ for $\bbeta = (\beta_1,\dots,\beta_d)^\top \in \bR_{++}^d$ and $\betatilde = (\sum_{k=1}^d 1/\beta_k)^{-1}$.
In this main text, we also use the abbreviation $B_{p,q,{\betatilde}}$  as the $\betatilde$ directly impact the final convergence analysis for the deep neural network estimator.

For function $f: [0,1]^d \to \bR$, we define the $r$th difference of $f$ in the direction $h\in \bR^d$ as
$$\Delta_h^r(f)(x):=\Delta_h^{r-1}(f)(x+h)-\Delta_h^{r-1}(f)(x), \Delta_h^0(f)(x):=f(x),$$
for $x,x+rh \in [0,1]^d$, otherwise we set $\Delta_h^r(f)(x) = 0$.

\begin{definition}
\label{def:Besov}
For a function $f \in L^p(\Omega)$ where $p \in (0,\infty]$,
the $r$-th modulus of smoothness of $f$ is defined by
$
w_{r,p}(f,t) = \sup_{h \in \bR^d: |h_i| \leq t_i} \|\Delta_h^r(f)\|_{p},
$
for $t = (t_1,\dots,t_d),~t_i >0$.
\end{definition}

\begin{definition}[Anisotropic Besov space $B^{\bbeta}_{p,q,\betatilde}(\Omega)$ \citep{suzuki2021deep}]
For $0 < p,q \leq \infty$, $\bbeta=(\beta_1,\dots,\beta_d)^\top \in \bR_{++}^d$,
$r:=  \max_i \lfloor \beta_i \rfloor + 1$,
let the semi-norm $|\cdot|_{B^{\bbeta}_{p,q,\betatilde}(\Omega)}$ be
\begin{align*}
\textstyle
& |f|_{B^{\bbeta}_{p,q,\betatilde}(\Omega)} :=
\begin{cases}
\left( \sum\limits_{k=0}^\infty [2^k w_{r,p}(f, (2^{-k/\beta_1}, \dots, 2^{-k/\beta_d})) ]^q \right)^{1/q} & (q < \infty), \\
\sup_{k \geq 0} 2^k w_{r,p}(f, (2^{-k/\beta_1}, \dots, 2^{-k/\beta_d}))  & (q = \infty).
\end{cases}
\end{align*}
The norm of the anisotropic Besov space $B^{\bbeta}_{p,q,\betatilde}(\Omega)$ is defined by
$\|f\|_{B^{\bbeta}_{p,q,\betatilde}(\Omega)} := \|f\|_{p} + |f|_{B^{\bbeta}_{p,q,\betatilde}(\Omega)}$
and 
the a-Besov space is defined as
$B^{\bbeta}_{p,q,\betatilde}(\Omega) = \{f \in L^p(\Omega) : \|f\|_{B^{\bbeta}_{p,q,\betatilde}(\Omega)} < \infty\}$
.
\end{definition}

The a-Besov space can be used to avoid the curse of dimensionality and encompasses analyses of the H\"{o}lder space \citep{schmidt2020nonparametric} and Besov space \citep{suzuki2018adaptivity}, as well as the low-dimensional manifold setting as special cases.
We recall that  the smoothness of the H\"{o}lder class or  Sobolev class is defined by  a scalar $\beta$, where the smoothness property is uniform across different directions.
In contrast,  $\bbeta\in \bR^d$, the vector of smooth parameter for a-Besov class, represents the heterogeneous smoothness in each direction.
An important quantity in the convergence analysis for a-Besov space is the average smoothness,
\begin{align}
\label{def:smooth_aBesov}
     \betatilde = (\sum_{k=1}^d 1/\beta_k)^{-1}.
\end{align}
We require that $\betatilde$ is large enough, i.e., $\betatilde  > 1/p$, then functions  are continuous.
If $\beta_i$ is large, then a function in $B^{\bbeta}_{p,q,\betatilde}(\Omega)$ is smooth to the $i$th coordinate direction,
otherwise, it is non-smooth to that direction.

Again,  we denote $B_{p,q,{\betatilde}}$  as an abbreviation of the a-Besov space in the main text as the $\betatilde$ directly impact the final convergence analysis for the deep neural network estimators.

\subsection{Details of Linear Case}
In the linear case presented in Section 3.4,
 the solutions $B_c^*$ and $B_t^*$ are typically non-unique. However, all solutions lead to the same central subspace \citep{ma2013efficient,xu2022distributed}.

 Here we study the projection matrix of $B_c$ and $B_t$.
 For $B_c$,
            \begin{align*}
              \boldsymbol{P}_{B_c}  =     \bX B_c \big(B_c^\top \bX^\top  \bX B_c \big)^{-1}B_c^\top \bX^\top  =
                \bX B_c B_c^\top \bX^\top,
            \end{align*}
            where $ \big(B_c^\top \bX^\top  \bX B_c \big)^{-1} = I_{r_c}$
            as $\bE(\bX)=\boldsymbol{0}$, $\operatorname{Cov}(\bX)=\boldsymbol{I}$   and $B_c^\top B_c = I_{r_c}$.
            Similarly, we have
            $
              \boldsymbol{P}_{B_t}  =   \bX B_t \big(B_t^\top \bX^\top  \bX B_t \big)^{-1}B_t^\top \bX^\top  =
                \bX B_t B_t^\top \bX^\top.
            $
            Then we obtain the projection matrix for $[B_c,B_t]$,
            {
            \small
            \begin{align}
            \begin{aligned}
            \boldsymbol{P}_{[B_c,B_t]}
            = &
                \bX
            \left[
             B_c,
             B_t
            \right ]
            \Big(
            \left[
            \begin{array}{c}
             B_c^\top \\
             B_t^\top
            \end{array}
            \right ]
            \bX^\top
                \bX
            \left[
             B_c,
             B_t
            \right ]
            \Big)^{-1}
            \left[
            \begin{array}{c}
             B_c^\top \\
             B_t^\top
            \end{array}
            \right ]
            \bX^\top
            \\
            =&
                \bX
            \left[
             B_c,
             B_t
            \right ]
            \left[
            \begin{array}{c}
             B_c^\top \\
             B_t^\top
            \end{array}
            \right ]
            \bX^\top
            = \bX
             B_cB_c^\top
            \bX^\top
            +
            \bX
             B_tB_t^\top
            \bX^\top\\
            =& \boldsymbol{P}_{B_c} + \boldsymbol{P}_{B_t},
            \end{aligned}
            \end{align}
            }
            where
            $B_c^\top B_t = 0$ and
             $\Big(
            \left[
            \begin{array}{c}
             B_c^\top \\
             B_t^\top
            \end{array}
            \right ]
            \bX^\top
                \bX
            \left[
             B_c,
             B_t
            \right ]
            \Big)^{-1} = \boldsymbol{I}$ .
            It is clear that the space spanned by $[B_c,B_t]$ combines the information in both $B_c$ and $B_t$.

            We consider a simple example as follows,
            $$\mD_0: Y \indep X\mid [ X_1, X_2, X_3, X_4];
            \quad
            \mD_1: Y \indep X\mid [ X_1, X_2];
            \quad
            \mD_2: Y \indep X\mid [ X_1, X_3],$$
            where
            $\bE(X)= \boldsymbol{0}$.
            $Cov(X)= \boldsymbol{I}$.
            Clearly, we have
            $\boldsymbol{P}_{B_c} = \boldsymbol{P}_{[X_1,X_2,X_3]}$,
            $\boldsymbol{P}_{B_t} = \boldsymbol{P}_{[X_4]}$
            and $\boldsymbol{P}_{[B_c,B_t]} = \boldsymbol{P}_{B_c} + \boldsymbol{P}_{B_t}$.
            Then
            we have
            $[B_c,B_t]$ is a sufficient representation for the task on $\mD_0$ where
            $B_c$ and $B_t$ draw the information of $(X_1,X_2,X_3)$ and $X_4$, respectively.

\section{Additional theoretical results and proofs}
Below, $C$  represents generic constants that may vary from line to line. The notation
$\mO$ indicates the stochastic equivalent, while $\widetilde\mO$
 denotes the stochastic equivalent up to some logarithmic factors.

\subsection{Theoretical results for  learning  $R_c$}
Recall
$
 R^*_c = \argmin \ \mL_{S}(R)
$
where
    \begin{align*}
       \mL_{S}(R) = \sum\limits_{s=1}^S
     \Big\{-\mV(R(X_s),Y_s)
    + \lambda_E \bD(R(X_s)\Vert \gamma_{r_c})
    \Big\}+ \lambda_Z \mV(R(X_{pool}),Z).
    \end{align*}
The empirical estimator $\widehat R_c$ is given as
\begin{align*}
\widehat R_c = \underset{R\in\mF_{R_c}}{\argmin} \ \mL_{S,n}(R),
\end{align*}
where
$\mF_{R_c}$ is the ReLU network class for learning $R_c$ and
\begin{align*}
\small
\mL_{S,n}(R) &= -\sum\limits_{s=1}^S
     \Big\{\mV_n(R(X_s),Y_s)
    + \lambda_E \bD_n(R(X_s)||\gamma_r)
    \Big\}+ \lambda_Z \mV_n(R(X_{pool}),Z_{pool}).
\end{align*}

Under  Assumption 1-3,
set the tuning parameter $\lambda_E,\lambda_Z = \mO(1)$,
we have the excess risk bound for the $\widehat R_c$,
\begin{align*}
    \begin{aligned}
	\mL_{S}(\widehat R_c)-\mL_{S}(R^*_c)
= \widetilde \mO(\sqrt{r_c}N^{-\frac{\betatilde_c}{2\betatilde_c + 1}}).
\end{aligned}
\end{align*}

To obtain the result, we decompose the excess risk $\mL_{S}(\widehat R_c)-\mL_{S}(R^*_c)$ as follows,
\begin{align*}
\mL_{S}(\widehat R_c)-\mL_{S}(R^*_c)
=&
\mL_{S}(\widehat R_c)  - \mL_{S,n}(\widehat R_c) \nonumber\\
 &+ \mL_{S,n}(\widehat R_c)  -   \mL_{S,n}(\widetilde R_c) \nonumber \\
 &+ \mL_{S,n}(\widetilde  R_c)  -   \mL_{S}(\widetilde  R_c) \nonumber \\
&+  \mL_{S}(\widetilde  R_c)  - \mL_{S}(R^*_c),
\end{align*}
 where
  $\tilde R_c\in \mF_{R_c}$ is any candidates from the  neural network function class $\mF_{R_c}$,
 $R^*_c$ is the optimal sufficient and invariant representation.

The terms
$\mL_{S}(\widehat R_c)  - \mL_{S,n}(\widehat R_c)$ and
$\mL_{S,n}(\widetilde  R_c)  -   \mL_{S}(\widetilde  R_c)$ can be bounded by  the stochastic error
$\underset{R\in \mF_{R_c}}{\sup}\big|\mL_{S}(R)  - \mL_{S,n}(R) \big|$.
The  term $\mL_{S}(\widetilde  R_c)  - \mL_{S}(R^*_c)$ depends on the approximation power of the neural network class $\mF_{R_c}$.
It can be controlled with $\underset{\widetilde R\in\mF_{R_c}}{\inf}\big|\mL_{S}(\widetilde R_c)  - \mL_{S}(R^*_c) \big|$ following the definition of $\widetilde R_c$.
The term $\mL_{S,n}(\widehat R_c)  -   \mL_{S,n}(\widetilde R_c)\leq 0$ can be bounded by 0 in the inequality  following the definition of $\widehat R_c$.

Thus  we have
\begin{align}
	\mL_{S}(\widehat R_c)-\mL_{S}(R^*_c)
<
2\underset{R\in \mF_{R_c}}{\sup}\big|\mL_{S}(R)  - \mL_{S,n}(R) \big|
+
\underset{\widetilde R\in \mF_{R_c}}{\inf}\big|\mL_{S}(\widetilde  R)  - \mL_{S}(R^*_c) \big|.
\end{align}

\subsubsection{Analysis for the  approximation and the truncation error}
In this section, we give the  bound of
$$\underset{\widetilde  R\in \mF_{R_c}}{\inf}\big|\mL_{S}(\widetilde R)  - \mL_{S}(R^*_c) \big|.$$

As the $R^*_c$ follows the Gaussian distribution and is unbounded, we introduce
its truncated version
\begin{align}
    T_B R(X) =
    \begin{cases}
    R(x), &\text{if } |R(x)|\leq B,\\
    B ,  &\text{if } R(x)> B,\\
    -B ,  &\text{if } R(x)< -B,\\
    \end{cases}
\end{align}
where $B = \mO(\log(N)^{1/2})$.

By the definition of $\mL_{S}$,
\begin{align}
\begin{aligned}
   \big|\mL_{S}(\widetilde R)  - \mL_{S}(R^*_c) \big|
    < &
    \sum\limits_{s=1}^S
    \big| \mV(\widetilde R(X_s),Y_s) - \mV(R^*_c(X_s),Y_s)\big| \\
    +&
    \lambda_E\sum\limits_{s=1}^S
    \big| \bD(\widetilde R(X_s)||\gamma_{r_c}) - \bD(R^*_c(X_s)||\gamma_{r_c})\big| \\
    +& \lambda_Z \big|\mV(\widetilde R(X_{pool}),Z_{pool}) - \mV(R^*_c(X_{pool}),Z_{pool})\big|.
\end{aligned}
\end{align}

We analyze these terms one by one.
For the first term
$
    \mV(\widetilde R(X_s),Y_s) - \mV(R^*_c(X_s),Y_s),
$
with  $s=1,\dots,S$.
We have
\begin{align}
\begin{aligned}
    |\mV(\widetilde R(X_s)&,Y_s) - \mV(R^*_c(X_s),Y_s)|
 \\=&
   |
   \mV(\widetilde R(X_s),Y_s)
   -\mV(T_BR^*_c(X_s),Y_s)
   +\mV(T_BR^*_c(X_s),Y_s)
   -\mV(R^*_c(X_s),Y_s)
   |
   \\<&
   |
   \mV(\widetilde R(X_s),Y_s)
   -\mV(T_BR^*_c(X_s),Y_s)|
   +
   |
   \mV(T_BR^*_c(X_s),Y_s)
   -\mV(R^*_c(X_s),Y_s)
   |.
\end{aligned}
\end{align}
Recall that
$
\mV[z, y]= \bE\Big[\|z_{1}-z_{2}\||y_{1}-y_{2}|\Big]-2 \bE\Big[\|z_{1}-z_{2}\||y_{1}-y_{3}|\Big]
+\bE\Big[\|z_{1}-z_{2}\|\Big] \bE\Big[|y_{1}-y_{2}|\Big],
$
where $(z_i,y_i), i = 1,2,3$ are i.i.d. copies of $(z,y)$ \citep{szekely2007measuring}.
We have
\begin{align}
\label{result_approx_Vxy}
\begin{aligned}
&\Big| \mV[T_BR_c^*(x), y]- \mV[\widetilde R(x), y]\Big|\\
&\leq  \Big|\bE\Big[(\|T_BR_c^*(x_1)-T_BR_c^*(x_2)\|-\|\widetilde R(x_1)-\widetilde R(x_2)\|)|y_{1}-y_{2}|\Big]\Big|\\
&+2 \Big|\bE\Big[(\|T_BR_c^*(x_1)-T_BR_c^*(x_2)\|-\|\widetilde R(x_1)-\widetilde R(x_2)\|)|y_{1}-y_{3}|\Big]\Big|\\
&+\Big|\bE\Big[\|T_BR_c^*(x_1)-T_BR_c^*(x_2)\|-\|\widetilde R(x_1)-\widetilde R(x_2)\Big] \bE\Big[\Big\|y_{1}-y_{2}\Big\|\Big]\Big|\\
& \leq 8C\bE\Big|\|T_BR_c^*(x_1)-T_BR_c^*(x_2)\|-\|\widetilde R(x_1)-\widetilde R(x_2)\|\Big|\\
&\leq  16C\bE\|T_BR_c^*(x)-\widetilde R(x)\|.
\end{aligned}
\end{align}
The first and third inequalities follow the triangle inequality.
The second inequality holds due to  the boundedness of $y$.

Similarly,
\begin{align*}
    \begin{aligned}
         |
   \mV(T_BR^*_c(X_s),Y_s)
   -\mV(R^*_c(X_s),Y_s)
   |
   < 16C \bE\|T_BR_c^*(x)-R_c^*(x)\|.
    \end{aligned}
\end{align*}
For  $ 1\leq i\leq d$, denote $R^*_{c,k}$ as the $k\text{-}th $ element of $R^*_{c}$ and $B = \mO(\log(N)^{1/2})$,
\begin{align}
\label{truncate_R_0205}
    \bE\vert R^*_{c,k} - T_{B} R^*_{c,k}\vert
    &=
    2\int_{B}^\infty \frac{x}{\sqrt{2\pi}} e^{-x^2/2}dx=\frac{-2e^{-x^2/2}}{\sqrt{2\pi}}\vert_{B}^\infty=\frac{2e^{-B^2/2}}{\sqrt{2\pi}}
    = \mO(N^{-1}).
    \end{align}

In conclusion,
\begin{align}
\begin{aligned}
    |\mV(\widetilde R(X_s),Y_s) -\mV(R^*_c(X_s),Y_s)|
= \mO\Big(\bE\|T_BR_c^*(x)-\widetilde R(x)\|\Big)
    + \mO(N^{-1}),
\end{aligned}
\end{align}
with $B = \mO(\log(N)^{1/2})$.

For the second term, it is similar that we have
\begin{align}
\label{result_approx_Vxz}
\begin{aligned}
\Big|
\mV[\widetilde R(x), Z_{pool}]
-
 \mV[R_c^*(x), Z_{pool}]
\Big|\leq  16C\bE\|\widetilde R(x)-R_c^*(x)\| + \mO(N^{-1}).
\end{aligned}
\end{align}

For the third term
$
    \big| \bD(\widetilde R(X_s)||\gamma_{r_c}) - \bD(R_c^*(X_s)||\gamma_{r_c})\big|,
$
for $s=1,\dots,S$.
Obviously we have $ \bD(R_c^*(X_s)||\gamma_{r_c}) = 0$ given $R_c^*$ follows the standard Gaussian distribution.
Recall that
$\bD(X||\gamma_{r_c}) = \bD(X,U)$ where $U\sim N(0,I_{r_c})$ is the independently generated random samples.
\begin{align}
\label{appro_D0205}
\begin{aligned}
\big|   \bD&(\widetilde R(X_s),U) - \bD(R^*(X_s),U)  \big|
 \\=&
\big|
\bD(\widetilde R(X_s),U)
- \bD(T_BR^*(X_s),U)
+ \bD(T_BR^*(X_s),U)
- \bD(R^*(X_s),U)
\big|
\\\leq &
\big|
\bD(\widetilde R(X_s),U)
- \bD(T_BR^*(X_s),U)
\big|
+
\big|\bD(T_BR^*(X_s),U)
- \bD(R^*(X_s),U)
\big|
\end{aligned}
\end{align}

For the first term in \eqref{appro_D0205},
\begin{align}
\label{result_approx_Drx1}
\begin{aligned}
    &\big|   \bD(\widetilde R(X_s),U) - \bD(T_BR^*(X_s),U)  \big| \\
    &\leq
            2\big|
             E\|\widetilde R(X) - U\| - E\| T_BR^*(X) - U\|
            \big| \\
            &+
            \big|
            E\|\widetilde R(X_1) - \widetilde R(X_1)\| - E\| T_BR^*(X_1) - T_BR^*(X_2)\|
            \big|\\
            &+
            \big|
            E\|U_1 - U_2\| - E\| U_1 - U_2\|
            \big| \\
    &\leq
    4
          E\|\widetilde R(X) -  T_BR^*(X) \|  .
\end{aligned}
\end{align}
The last inequality holds following the triangular inequality.
Similarly,
for the second term in \eqref{appro_D0205},
\begin{align}
\label{result_approx_Drx2}
\begin{aligned}
    &\big|   \bD(T_B R^*(X_s),U) - \bD(R^*(X_s),U)  \big| \leq
    4
          E\|T_B R^*(X) -  R^*(X) \|
    =\mO(N^{-1}).
\end{aligned}
\end{align}
The last equality follows result \eqref{truncate_R_0205} and  $B=\mO(\log(N)^{1/2})$.

Combining the results \eqref{result_approx_Vxy},\eqref{result_approx_Vxz},
\eqref{result_approx_Drx1},
\eqref{result_approx_Drx2},
we have that
\begin{align}
    \begin{aligned}
        \underset{\widetilde R\in \mF}{\inf}\big|\mL_{S}(\widetilde R_c)  - \mL_{S}(R^*_c) \big|
        =  \mO\Big( E\|\widetilde R(X) -  T_BR_c^*(X) \| \Big)
         + \mO(N^{-1}).
    \end{aligned}
\end{align}
Actually the first term is the approximation error for neural network class in learning functions and the second term is the truncated error.

Then we derive the bound for the term
$$ \bE\|\widetilde R(X) -  T_BR_c^*(X) \|.$$
It suffices to control the $\|\widetilde R(X) -  T_BR^*(X) \|_{L^\infty}$.
For the $i$-th element  of the $\widetilde R$ and $T_BR_c^*$, we have
 \begin{align*}
 \|T_BR^*_{c,i}-\widetilde R_{i}\|_{L^{2}(\mu_{x})}
 &= [\int (T_BR^*_{c,i}(x)-\widetilde R_{i}(x))^2 f_X(x)\mathrm{d}x]^{1/2}\\
 &\leq  \|T_B{R}^*_{c,i}-\widetilde R_{i}\|_{L^{\infty}}\int f_X(x)\mathrm{d}x\\
 &\leq C \|T_B{R}^*_{c,i}-\widetilde R_{i}\|_{L^{\infty}}.
 \end{align*}

\begin{lemma}\label{lemma_approx_aBesov}
Recall that  $R^*_c = (R^*_{c,1},\dots,R^*_{c,r_c})$ where    $R^*_{c,i}\in B_{p,q,\betatilde_{c,i}}(\Omega)$ for $i=1,\dots,r_c$.
Here exists  $\widetilde R_{c,i} \in \mF_{R_c}$ where the $\mF_{R_c}$ satisfying the Assumption 2 in the main text, we have that
\begin{align}
    \|\widetilde R_{c,i} - T_BR^*_{c,i}\|_{L^\infty}
    \leq \widetilde\mO(N^{\betatilde_c/(2\betatilde_c + 1)}),
\end{align}
where $\betatilde_{c} = \min\{\betatilde_{c,1},\dots,\betatilde_{c,r_c}\}$,
$N=\sum_{s=1}^S n_s$ and $N =\mO(n_s)$ as $S$ is finite.
\begin{proof}
This Lemma fellows directly   Proposition 2 of  \cite{suzuki2021deep}.
\end{proof}
\end{lemma}

Then we have
\begin{align}
\label{Result_source_approx}
    \|\widetilde R_{c} - R_{c}\|_{L^\infty}  & \leq \widetilde\mO(\sqrt{r_c}N^{\betatilde_c/(2\betatilde_c + 1)}),
    \nonumber\\
    \underset{R\in \mF_{R_c}}{\inf}\big|\mL_{S}(\widetilde R_c)  - \mL_{S}(R^*_c) \big|
    &\leq \widetilde\mO(\sqrt{r_c}N^{\betatilde_c/(2\betatilde_c + 1)}).
\end{align}

\subsubsection{Analysis for the  stochastic error}
In this section, we  bound the stochastic error
$\underset{R\in \mF_{R_c}}{\sup}\big|\mL_{S,n}(R)  - \mL_{S}(R) \big|.$

We have
\begin{align}
\label{source_3errors_stochastic}
    \begin{aligned}
        \underset{R\in \mF_{R_c}}{\sup}\Big|\mL_{S,n}(R)  - \mL_{S}(R) \Big|
        \leq
        &
        \underset{R\in \mF_{R_c}}{\sup}\sum_{s=1}^S\Big|\mV_n(R(X_s),Y_s)  - \mV(R(X_s),Y_s) \Big| \\
        +& \underset{R\in \mF_{R_c}}{\sup}\sum_{s=1}^S\Big|\bD_n(R(X_s)||\gamma_{r_c})  - \bD(R(X_s)||\gamma_{r_c}) \Big| \\
        +&\underset{R\in \mF_{R_c}}{\sup}\Big|\mV_n(R(X_{pool}),Z_{pool})  - \mV(R(X_{pool}),Z_{pool}) \Big|.
    \end{aligned}
\end{align}

In the present problem,  the objective function
is the combination of loss functions with $U$-process type indexed by a class of neural networks.

First we consider the first term of \eqref{source_3errors_stochastic},
\begin{align*}
\begin{aligned}
      \underset{R\in  \mF_{R_c}}{\sup}\sum_{s=1}^S\big|\mV_n(R(X_s),Y_s) & - \mV(R(X_s),Y_s) \big| \\
    \leq  &
    \sum_{s=1}^S     \underset{R\in  \mF_{R_c}}{\sup}\big|\mV_n(R(X_s),Y_s)  - \mV(R(X_s),Y_s) \big|\\
    \le &
     CS     \underset{R\in  \mF_{R_c}}{\sup}\big|\mV_n(R(X_1),Y_1)  - \mV(R(X_1),Y_1) \big|.
\end{aligned}
\end{align*}
where $S$ is the number of sources and $C$ is some constant.
It suffice for us to bound the stochastic error
$$\underset{R\in  \mF_{R_c}}{\sup}\big|\mV_n(R(X_1),Y_1)  - \mV(R(X_1),Y_1) \big|.$$

Denote
$\forall R \in \mF_{R_c}$ where $\mF_{R_c}$ is the class of neural networks.
Let  $\tilde{O} = (R(x),y)$
denotes the random variables from the domain $s=1$ and
we omit the source subscript here for simplicity.

We denote $\tilde{O}_i = (R(x_i),y_i), i=1,...n_1$ are i.i.d copy of $\tilde{O}$.
We define centered kernel as
\begin{equation}\label{newkernel}
\begin{array}{l}
\bar{h}_{R}(\tilde{O}_1,\tilde{O}_2, \tilde{O}_3,\tilde{O}_4) =
\frac{1}{4} \sum\limits_{1 \leq i, j \leq 4, \atop i \neq j}\|R(x_{i})-R(x_{j})\| |y_{i}-y_{j}|\\
-\frac{1}{4} \sum\limits_{i=1}^{4}\left(\sum\limits_{1 \leq j \leq 4, \atop j \neq i}\left\|R(x_{i})-R(x_{j})\right\| \sum_{1 \leq j \leq 4,\atop i \neq j}|y_{i}-y_{j}|\right) \\
\quad+\frac{1}{24} \sum\limits_{1 \leq i, j \leq 4, \atop i \neq j}\left\|R(x_{i})-R(x_{j})\right\| \sum\limits_{1 \leq i, j \leq 4, \atop i \neq j}|y_{i}-y_{j}| -\mV[R(x), y].
\end{array}
\end{equation}

Then,  the  centered $U$-statistics
$\mV_n(R(X),Y)  - \mV(R(X),Y)$
can be expressed as
$$
\mathbb{U}_{n_1}(\bar{h}_R) = \frac{1}{\tbinom{n_1}{4}} \sum_{i_{1}<i_{2}<i_{3}<i_{4}} \bar{h}_{R}(\tilde{O}_{i_1},\tilde{O}_{i_2}, \tilde{O}_{i_3},\tilde{O}_{i_4}).$$

By the symmetrization randomization
(Theorem 3.5.3 in  \cite{de2012decoupling}),
we have
\begin{equation}\label{rpm}
\bE\Big[\sup_{R\in  \mF_{R_c}}|\mathbb{U}_{n_1}(\bar{h}_{R})|\Big]
\leq C \bE\Big[\sup_{R\in  \mF_{R_c}}|  \frac{1}{\tbinom{n_1}{4}} \sum_{1 \leq i_{1}<i_{2}<i_{3}<i_{4} \leq n} \epsilon_{i_1}\bar{h}_{R}(\tilde{O}_{i_1},\tilde{O}_{i_2}, \tilde{O}_{i_3},\tilde{O}_{i_4})|\Big],
\end{equation}
where, $\epsilon_{i_1}, i_1 = 1,...n_1$ are i.i.d  Rademacher variables that are also independent with $\tilde{O}_i, i = 1,\ldots, n_1.$

We then  bound the above Rademacher process with the metric entropy of  the neural network class $\mF_{R_c}$.
Recall that under  Assumption 1,
the kernel $\bar h_R$ is also bounded for $R \in \mF_{R_c}$.
$\forall {R \in \mF_{R_c}},$
we define a random empirical measure for the pair $(R,R_\delta)$,
$$
e_{n_1,1}(R, R_\delta)= \bE_{\epsilon_{i_1}, i_1 = 1,...,n_1}\Big|\frac{1}{\tbinom{n_1}{4}} \sum_{i_1<  i_2<i_3<i_4} \epsilon_{i_{1}}(\bar{h}_{R}-\bar{h}_{R_\delta})(\tilde{O}_{i_{1}}, \ldots, \tilde{O}_{i_{4}})\Big|.$$

Condition on $\{X_i,Y_i\}_{i=1,\ldots, n_1}$, let $\mN(\mF_{R_c}, e_{n_1,1}, \delta)$ be  the covering number of the neural network class $\mF_{R_c}$ with respect  to the  empirical  distance  $e_{n_1,1}$ at scale of $\delta>0$. Denote $\mF_{\delta}$ as  the  covering set of  $\mF_{R_c}$ with cardinality of $\mN(\mF_{R_c}, e_{n_1,1}, \delta))$.
Then,
\begin{align*}
&\bE_{\epsilon_{i_1}}\Big[\sup_{R\in  \mF_{R_c}}\Big|  \frac{1}{\tbinom{n_1}{4}} \sum_{i_{1}<i_{2}<i_{3}<i_{4}} \epsilon_{i_1}\bar{h}_{R}(\tilde{O}_{i_1},\tilde{O}_{i_2}, \tilde{O}_{i_3},\tilde{O}_{i_4})\Big|\Big]\\
&\leq \delta + \bE_{\epsilon_{i_1}}\Big[\sup_{R \in \mF_\delta}\Big|  \frac{1}{\tbinom{n_1}{4}} \sum_{i_{1}<i_{2}<i_{3}<i_{4} } \epsilon_{i_1}\bar{h}_{R}(\tilde{O}_{i_1},\tilde{O}_{i_2}, \tilde{O}_{i_3},\tilde{O}_{i_4})\Big|\Big]\\
&\leq \delta + C \frac{1}{\tbinom{n_1}{4}} (\log \mN(\mF_{R_c}, e_{n_1,1}, \delta))^{1/2} \max_{R \in \mF_\delta} [\sum_{i_1 = 1}^{n_1} (\sum_{i_2<i_3<i_4} \bar{h}_{R}(\tilde{O}_{i_1},\tilde{O}_{i_2}, \tilde{O}_{i_3},\tilde{O}_{i_4}))^2]^{1/2}\\
&\leq \delta + C \mB (\log \mN(\mF_{R_c}, e_{n_1,1}, \delta))^{1/2} \frac{1}{\tbinom{n_1}{4}}[\frac{n_1(n_1 !)^{2}}{((n_1-3) !)^{2}}]^{1 / 2}\\
&\leq  \delta + 2C \mB(\log \mN(\mF_{R_c}, e_{n_1,1}, \delta))^{1/2}/\sqrt{n_1}.
\end{align*}
Where $\mB$ is the upper bound of the weights in the neural network class.
The third inequality follows the Lemma 8.8 of  \citet{huang2020deep}.
Thus the stochastic error can be bounded with the sample size $n_1$, the pre-given covering scale $\delta$ and the corresponding covering number $\mN(\mF_{R_c}, e_{n_1,1}, \delta)$.

We mention that here we learn $R_c$ with dimension $r_c$.
In the Assumption 3, we consider that
the width of the network $\mF_{r_c}:\bR^d\to  \bR^{r_c}$ is linear with $r_c$, ensuring the expression power for  learning  the $R$ of dimension $r_c$.
We have  the result that
$
\mN(\mF_{R_c},e_{n,1},\delta)) <  \mN^{r_c}(\mF_1,e_{n,1},\delta)),$
where $\mF_1$ is the neural network class with output of dimension 1,
 the depth
    $\mH_{\mF_1} =  \mO\big(  \log(d)\log(N)   \big)$,
    width
    $\mW_{\mF_1} =\mO\big(d N^{1/(2\betatilde_c + 1)}\big)$ and
    model size  $\mS_{\mF_1}  = \mO\big(d^2N^{1/(2\betatilde_c + 1)} \log(N)\log(d)  \big)$.
 The size of the covering set of $\mF_{R_c}$ can be bounded as the product of the size of the  covering set of $\mF_1$.

With the result $\mN(\mF_1, e_{n_1,1}, \delta)< \mN(\mF_1, e_{n_1,\infty}, \delta)$
and
 and the relationship between the metric entropy and the VC-dimension of $\mF$,
we have \citep{anthony2009neural},
$$\log \mN(\mF_1,e_{n_1,\infty},\delta)) \leq \mathrm{VC}_{\mF_1} \log \frac{2en_1\mB}{\delta\mathrm{VC}_{\mF_1}}.$$

Following the result of  \citet{bartlett2019}, the $\mathrm{VC}_{\mF_1}$ can be bounded by the depth,width and the number of parameters of the ReLU network,
\begin{align}
    c \mH_{\mF_1} \mS_{\mF_1} \log \mS_{\mF_1} \leq \mathrm{VC}_{\mF_1} \leq C \mH_{\mF_1} \mS_{\mF_1} \log \mS_{\mF_1} ,
\end{align}
where $c,C$ are two different constant.


Thus we have
\begin{align}
\label{result_source_VC}
\begin{aligned}
&\bE_{\epsilon_{i_1}}\Big[\sup_{R\in  \mF_{R_c}}|  \frac{1}{\tbinom{n_1}{4}} \sum_{ i_{1}<i_{2}<i_{3}<i_{4} } \epsilon_{i_1}\bar{h}_{R}(\tilde{O}_{i_1},\tilde{O}_{i_2}, \tilde{O}_{i_3},\tilde{O}_{i_4})|\Big]\\
&\leq  \delta + 2C \mB(r_c\log \mN(\mF_1, e_{n_1,1}, \delta))^{1/2}/\sqrt{n_1}\\
&\leq  \delta + C\mB ( r_c\mH_{\mF_1} \mS_{\mF_1} \log \mS_{\mF_1}  \log \frac{\mB n_1}{\delta \mH_{\mF_1} \mS_{\mF_1} \log \mS_{\mF_1} })^{1/2}/\sqrt{n_1}\\
&\leq
\widetilde\mO \big(\sqrt{r_c}n_1^{\frac{-\betatilde}{2\betatilde + 1}}\big) + \frac{1}{n_1}.
\end{aligned}
\end{align}
The last line holds with setting $\delta = 1/n_1$ and  {Assumption 3} for the DNN structure.
In conclusion,
\begin{align}
\label{Result_source_stochastic_1}
    \begin{aligned}
 &   \underset{R\in  \mF_{R_c}}{\sup}\sum_{s=1}^S\big|\mV(R(X_s),Y_s)  - \mV_n(R(X_s),Y_s) \big|  \\
    &\lesssim
        CS\underset{R\in  \mF_{R_c}}{\sup}\big|\mV(R(X_1),Y_1)  - \mV_n(R(X_1),Y_1) \big| \\
     &  \leq \widetilde\mO\big(\sqrt{r_c}n_1^{\frac{-\betatilde}{2\betatilde + 1}}\big)
     = \widetilde\mO\big(\sqrt{r_c}N^{\frac{-\betatilde}{2\betatilde + 1}}\big).
    \end{aligned}
\end{align}
The last equality holds as $S$, the number of sources,  is finite and $N = \mO(n_s)$ for $s=1,\dots,S$.

For the second term of \eqref{source_3errors_stochastic},
we have
\begin{align}
\label{Result_source_stochastic_2}
    \begin{aligned}
        \underset{R\in  \mF_{R_c}}{\sup}\big|\mV(R(X_{pool}),Z_{pool})  - \mV_n(R(X_{pool}),Z_{pool}) \big|
        \leq \widetilde\mO\big(\sqrt{r_c}N^{\frac{-\betatilde}{2\betatilde + 1}}\big).
    \end{aligned}
\end{align}
The result can be obtained similarly as that for the first term of \eqref{source_3errors_stochastic}.

For the third term of \eqref{source_3errors_stochastic}, we have
\begin{align}
\begin{aligned}
      \underset{R\in  \mF_{R_c}}{\sup}\sum_{s=1}^S\Big|\bD(R(X_s)||\gamma_{r_c})  - \bD_n(R(X_s)||\gamma_{r_c}) \Big|
    \leq  &
    \sum_{s=1}^S     \underset{R\in  \mF_{R_c}}{\sup}\Big|\bD(R(X_s)||\gamma_{r_c})  - \bD_n(R(X_s)||\gamma_{r_c}) \Big|\\
    \le &
     C S   \underset{R\in  \mF_{R_c}}{\sup} \Big|\bD(R(X_s)||\gamma_{r_c})  - \bD_n(R(X_s)||\gamma_{r_c}) \Big|.
\end{aligned}
\end{align}
It suffices to bound the term
$\underset{R\in \mF}{\sup} \big|\bD(R(X_s),U)  - \bD_n(R(X_s),U) \big|$ where the $U \sim N(0,I_{r_c})$ is drawn from the standard Gaussian distribution
and  \begin{align*}
\widetilde\bD_n(R(X),Y)  = \frac{1}{\tbinom{n}{2}} \sum_{1\leq i,j\leq n}h_e(\widetilde O_i;\widetilde O_j),
\end{align*}
where
$h_e(\widetilde O_i;\widetilde O_j) = h_e(u_1,u_2;v_1,v_2)=   \|u_1- v_2\| + \|u_2- v_1\|
-
\|u_1- u_2\| - \|v_1- v_2\|
$
\citep{gretton2012kernel}.
The $\tilde{O}_i = (R(x_i),y_i), i=1,...n_1$ is defined before the expression \eqref{newkernel}.
Thus we can also construct the following centered kernel
\begin{align*}
    \begin{aligned}
        \bar h_e(\widetilde O_i,\widetilde O_j):=&\|R(x_i)-U_i\|+\|R(x_j)-U_j\| - \|R(x_i)-R(x_j)\|-\|U_i-U_j\|\\ &- \bD(R(x),U).
    \end{aligned}
\end{align*}

Thus the centered U-statistics $\bD_n - \bD$ can be expressed as
\begin{align*}
    \begin{aligned}
        \mathbb{U}_{e,n} = \frac{1}{\tbinom{n}{2}} \sum_{1\leq i_1< i_2\leq n}\bar h_e(\widetilde O_{i_1},\widetilde O_{i_2}).
    \end{aligned}
\end{align*}
By the symmetrization randomization, we have
\begin{align*}
    \bE\Big[ \underset{R\in  \mF_{R_c}}{\sup}  \mathbb{U}_{e,n}(\bar h_e) \Big]
    \leq
    C\bE\big[ \underset{R\in  \mF_{R_c}}{\sup}  \frac{1}{\tbinom{n_1}{2}}  \sum_{i_1< i_2}
 \epsilon_{i_1}  \bar h_e(\widetilde O_{i_1},\widetilde O_{i_2}) \big],
\end{align*}
where $\epsilon_{i_1}$ is the Rademacher variables corresponding  for the decoupling.

Define
$$
\tilde e_{n,1}(R, R_\delta)= \bE_{\epsilon_{i}, i = 1,...,n}|\frac{1}{\tbinom{n}{2}} \sum_{1\leq {i_1}<  {i_2}\leq n} \epsilon_{i}(\bar{h}_{R}-\bar{h}_{R_\delta})(\tilde{O}_{i_1}, \tilde{O}_{i_2})|.$$
Condition on $\{X_i,Y_i\}_{i=1,\ldots, n_1}$, let $\mN(\mF_{R_c}, \tilde e_{n_1,1}, \delta)$ be  the covering number of the neural network class $\mF_{R_c}$ with respect  to the  empirical  distance  $\tilde  e_{n_1,1}$ at scale of $\delta>0$. Denote $\mF_{\delta,\tilde e}$ as  the  covering set of  $\mF$ with cardinality of $\mN(\mF_{R_c}, \tilde e_{n_1,1}, \delta))$.

Then,
\begin{align*}
&\bE_{\epsilon_{i_1}}\Big[\sup_{R\in  \mF_{R_c}}|  \frac{1}{\tbinom{n_1}{2}} \sum_{ i_{1}<i_{2}} \epsilon_{i_1}\bar{h}_{e}(\tilde{O}_{i_1},\tilde{O}_{i_2})|\Big]\\
&\leq \delta + \bE_{\epsilon_{i_1}}\Big[\sup_{R \in \mF_\delta}|  \frac{1}{\tbinom{n_1}{2}} \sum_{i_{1}<i_{2}} \epsilon_{i_1}\bar{h}_{R}(\tilde{O}_{i_1},\tilde{O}_{i_2})|\Big]\\
&\leq \delta + C  (\log \mN(\mF_{R_c},\tilde e_{n_1,1}, \delta))^{1/2} \frac{1}{\tbinom{n_1}{2}}\max_{R \in \mF_\delta} [\sum_{i_1 = 1}^{n_1} (\sum_{i_2} \bar{h}_{R}(\tilde{O}_{i_1},\tilde{O}_{i_2}))^2]^{1/2}\\
&\leq  \delta + 2C \mB(\log \mN(\mF_{R_c},  \tilde e_{n_1,1}, \delta))^{1/2} \frac{2}{n_1(n_1-1)} [\sum_{i=1}^{n_1} Cn_1^2]^{1/2}\\
&\leq  \delta + 2C \mB(\log \mN(\mF_{R_c}, \tilde e_{n_1,1}, \delta))^{1/2}/\sqrt{n_1}\\
&\leq  \widetilde\mO \big(\sqrt{r_c}n_1^{\frac{-\betatilde}{2\betatilde + 1}}\big) + \frac{1}{n_1}.
\end{align*}
The last inequality follows the result of \eqref{result_source_VC}.

Thus we have
\begin{align}
\label{Result_source_stochastic_3}
    \begin{aligned}
 &   \underset{R\in  \mF_{R_c}}{\sup}\sum_{s=1}^S\big|\bD(R(X_s)\|\gamma_{r})  - \bD_n(R(X_s)\|\gamma_{r_c}) \big|
    \le
      \widetilde\mO\big(\sqrt{r_c}n_1^{\frac{-\betatilde}{2\betatilde + 1}}\big)
      =
      \widetilde\mO\big(\sqrt{r_c}N^{\frac{-\betatilde}{2\betatilde + 1}}\big).
    \end{aligned}
\end{align}

Combining
\eqref{Result_source_approx},
\eqref{Result_source_stochastic_1},
\eqref{Result_source_stochastic_2},
\eqref{Result_source_stochastic_3},
we have the excess risk bound for the $\widehat R_c$ and the $R^*_c$,
\begin{align}
\label{Result_source_excessrisk}
    \begin{aligned}
	\mL_{S}(\widehat R_c)-\mL_{S}(R_c^*)
<&2\underset{R\in  \mF_{R_c}}{\sup}\big|\mL_{S}(R)  - \mL_{S,n}(R) \big| + \underset{R\in  \mF_{R_c}}{\inf}\big|\mL_{S}( R)  - \mL_{S}(R^*) \big| \\
\leq &\widetilde \mO(\sqrt{r_c}N^{-\frac{\betatilde}{2\betatilde + 1}}).
    \end{aligned}
\end{align}

We mention that the optimal solution for SIRep is non-unique but the existence of such $R_c \in \mM_{R_c}$ is ensured
where
$
\mM_{R_c}=\{
R_c:\bR^d\to\bR^{r_c},
Y_s\indep X_s \big| R(X_s), \text{ for } s=1,\dots,S,
R_c(X)\sim N(\mathbf{0},\mathbf{I}_{r_c}) \text{ and } R_c(X) \indep Z \}
$ \citep{lee2013general,YuGMDD2024}.
Following the idea of  \citet{zhu2010dimension,sheng2016sufficient},
we study  the   the distance between the $\widehat R$ and $R_c^*$ where $R_c^*$ is   one of the global optimal points such that $\widehat R_c$ converges to $R_c^*$.

Following   \citet{yang2018rsg},
given $R_c^*$ and
   the local set  $\mM_\rho(R_c^*)$ where
    $\mM_\rho(R_c^*) = \{R_c^\prime: \|R_c^*,R_c^\prime\|\leq \rho \text{,and }  R^\prime \sim N(0,1)\}$,
we have for $R_c^\prime \in \mM_\rho(R_c^*)$,
    \begin{align}
        \|R_c^\prime - R_c^*\|_2^2 \leq C | \mV(R_c^*,y) - \mV(R_c^\prime,y)|.
    \end{align}
Without loss of generality,
we prove the result considering the
case for $R: \bR^d \to \bR^r$ with $r=1$.
As $R_c^\prime \sim N(0,1)$ and
$R_c^\prime \in \mM_\rho$,
we have the following expression,
$R_c^\prime =  (1 - \rho)  R_c^* +\epsilon$
where
$\rho$ is the correlation coefficient
and
$\epsilon \sim N(0,2\rho-\rho^2)$.
First we have
\begin{align*}
\begin{aligned}
         ||R_c^\prime - R_c^*||^2_{L_2(P(X))} = (\rho R_c^* + \epsilon)^2
     &= \rho^2 \mathop{Var}(R_c^*) + \mathop{Var}(\epsilon)\\
     &= \rho^2 + 2\rho-\rho^2
     = 2\rho.
\end{aligned}
\end{align*}
Asymptotically, we have $\rho \rightarrow 0$ and thus $\|R_c^\prime - R_c^*\| = \mO(\rho^{1/2}) \rightarrow 0$.
Then,
by Theorem 3 of  \citet{szekely2009dCov},
\begin{align*}
\begin{aligned}
        \mV(R_c^\prime ,y) = \mV[(1-\rho) R_c^* + \epsilon, y]
    \leq  (1-\rho)\mV[ R_c^*, y] + \mV[\epsilon, 0]
     \leq  (1-\rho) \mV[R_c^*, y].
\end{aligned}
\end{align*}
Thus we have
$
        \mV(R_c^*,y) - \mV(R_c^\prime,y) > \rho \mV(R_c^*,y)> \rho C
$.
$\mV(R_c^*,y)>C$  is natural as the sufficient and invariance representation $R_c^*$ has prediction power for $y$ thus the distance covariance is bounded away from 0.
Thus
$
        \|R_c^\prime - R_c^*\|_2^2 \leq 2C | \mV(R_c^*,y) - \mV(R_c^\prime,y)|
$ holds.
Then we have
\begin{align}
\label{result_boundR_c}
\|\widehat R_c -  R^*_c\|
  \leq \widetilde \mO({r_c}^{1/4}n_s^{-\frac{\betatilde_c/2}{2\betatilde_c + 1}}).
\end{align}

\subsection{Theoretical results for learning  $R_t$.}
Recall that
$$
 \mL_{T}(\big[R_t,R_c\big]) =
     -\mV(\big[R_t(X_0), R_c(X_0)\big],Y_0)
    + \lambda_{E,0} \bD(R_t(X_0)||\gamma_{r_t})
    + \lambda_C \mV(R_t(X_0),R_c(X_0)),
$$
and its empirical counterpart is
$$
\mL_{T,n}(\big[ R_t, R_c]) = -\mV_n(\big[ R_t(X_0), R_c(X_0)\big],Y_0)
    + \lambda_{E,0} \bD_n( R_t(X_0)||\gamma_{r_t})
    + \lambda_C \mV_n( R_t(X_0), R_c(X_0))
,
$$
where the tuning parameters $\lambda_{E,0}, \lambda_C = \mO(1)$, $\mV_n$, $\bD_n$ are the empirical version of $\mV,\bD$, respectively.

The excess risk of $[\widehat R_t,\widehat R_c]$ can be decomposed as follows,
\begin{align}
    \begin{aligned}
    \mL_{T}([\widehat R_t,\widehat R_c]) &- \mL_{T}([R^*_t, R^*_c]) \\
    = &
    -\mV(\big[\widehat R_t(X_0),\widehat  R_c(X_0)\big],Y_0) + \mV(\big[ R^*_t(X_0),  R^*_c(X_0)\big],Y_0) \\
    &+ \lambda_C \mV(\widehat R_t(X_0),\widehat R_c(X_0)) -
     \lambda_C \mV( R^*_t(X_0), R^*_c(X_0))\\
    &+ \lambda_{E,0} \bD(\widehat  R_t(X_0)||\gamma_{r_t}) - \lambda_{E,0} \bD( R^*_t(X_0)||\gamma_{r_t})\\
    =& I
     + II
     + III.
    \end{aligned}
\end{align}

For the term $I$,
\begin{align}
\label{Problem_target_I}
    \begin{aligned}
       I =&   \mV(\big[ R^*_t(X_0),  R^*_c(X_0)\big],Y_0)- \mV(\big[\widehat R_t(X_0),\widehat  R_c(X_0)\big],Y_0)\\
         = &
           \mV(\big[ R^*_t(X_0),  R^*_c(X_0)\big],Y_0)
         -\mV(\big[\widehat R_t(X_0), R^*_c(X_0)\big],Y_0)\\
&+
                  \mV(\big[\widehat R_t(X_0),  R^*_c(X_0)\big],Y_0)
         -\mV(\big[\widehat R_t(X_0),\widehat  R_c(X_0)\big],Y_0)
    \end{aligned}
\end{align}

For the second part of \eqref{Problem_target_I},
we have,
\begin{align}
\label{result_I_term1}
\begin{aligned}
       \mV(\big[\widehat R_t(X_0),  &R^*_c(X_0)\big],Y_0)
     -\mV(\big[\widehat R_t(X_0),\widehat  R_c(X_0)\big],Y_0) \\
     = &\mO(\|[\widehat R_t,\widehat R_c] - [\widehat R_t, R^*_c]\|)
 =\mO(\|\widehat R_c -  R^*_c\|)
  =\widetilde \mO(r_c^{1/4}n_s^{-\frac{\betatilde_c/2}{2\betatilde_c + 1}}).
\end{aligned}
\end{align}
The first equality follows the result \eqref{result_approx_Vxy}.
The second equality follows the definition of the $L_2$ distance and the last equality follows the result \eqref{result_boundR_c}.

Then we consider
the first part of \eqref{Problem_target_I},
\begin{align*}
    \begin{aligned}
        \mV(\big[ R^*_t(X_0),&  R^*_c(X_0)\big],Y_0)
         -\mV(\big[\widehat R_t(X_0), R^*_c(X_0)\big],Y_0) \\
         \leq&
         \underset{\widetilde R_t\in \mF_{R_t}}{\inf}
         |\mV(\big[ R^*_t(X_0),  R^*_c(X_0)\big],Y_0)
         -\mV(\big[\widetilde R_t(X_0), R^*_c(X_0)\big],Y_0)|\\
         &+
         2\underset{ R_t\in \mF_{R_t}}{\sup}
         |\mV_n(\big[ R_t(X_0), R^*_c(X_0)\big],Y_0)
         -\mV(\big[ R_t(X_0), R^*_c(X_0)\big],Y_0)|.
    \end{aligned}
\end{align*}
We then bound the approximation error and the stochastic error, respectively.
For the approximation error, under Assumption 1-3, we have
\begin{align}
\label{result_target_ry_approx}
    \begin{aligned}
        \mV(\big[ R^*_t(X_0),&  R^*_c(X_0)\big],Y_0)
         -\mV(\big[\widetilde R_t(X_0), R^*_c(X_0)\big],Y_0) \\
         &\lesssim \|\big[ R^*_t(X_0),  R^*_c(X_0)\big]  - \big[\widetilde R_t(X_0), R^*_c(X_0)\big]\| \\
         &= \|  \widetilde R_t(X_0)  -  R^*_t(X_0)\|
         = \widetilde\mO(r_t^{1/2}n_0^{-\frac{\betatilde_t}{2\betatilde_t + 1}}).
    \end{aligned}
\end{align}
The result can be derived similar as result \eqref{Result_source_approx}.
Then for the stochastic error, we have
\begin{align}
\label{result_target_ry_stochastic}
    \begin{aligned}
\underset{ R_t\in \mF_{R_t}}{\sup}
         |\mV_n(\big[ R_t(X_0), R^*_c(X_0)\big],Y_0)
         -\mV(\big[ R_t(X_0), R^*_c(X_0)\big],Y_0)|.
    \end{aligned}
\end{align}
Similar to \eqref{Result_source_stochastic_1}, we have
$$
\underset{R_t\in \mF_{R_t}}{\sup}
         |\mV_n(\big[ R_t(X_0), R^*_c(X_0)\big],Y_0)
         -\mV(\big[ R_t(X_0), R^*_c(X_0)\big],Y_0)|
         \leq \widetilde\mO\big(r_t^{1/2}n_0^{\frac{-\betatilde_t}{2\betatilde_t + 1}}\big).$$
Combining these results
\eqref{result_I_term1},
\eqref{result_target_ry_approx} and \eqref{result_target_ry_stochastic}, we have
\begin{align}
\label{Result_target_I}
    \begin{aligned}
      I = &  \mV(\big[ R^*_t(X_0),  R^*_c(X_0)\big],Y_0)
         -\mV(\big[\widehat R_t(X_0), R^*_c(X_0)\big],Y_0) \\
=& \widetilde\mO\big(r_t^{1/2}n_0^{\frac{-\betatilde_t}{2\betatilde_t + 1}}\big)
+
\widetilde \mO(r_c^{1/4}n_s^{-\frac{\betatilde_c/2}{2\betatilde_c + 1}}).
    \end{aligned}
\end{align}

Then we consider the term
\begin{align*}
\begin{aligned}
   II =  \lambda_C \mV(\widehat R_t(X_0),\widehat R_c(X_0)) -
     \lambda_C \mV( R^*_t(X_0), R^*_c(X_0)).
\end{aligned}
\end{align*}
We emphasize that $\widehat R_c(\cdot)$ is estimated with the source datasets and thus can be regarded as pre-given regarding the target dataset.
Then we consider the following decomposition,
\begin{align}
\label{target_problem_rtrc}
\begin{aligned}
     \mV(\widehat R_t(X_0),&\widehat R_c(X_0)) -
      \mV( R^*_t(X_0), R^*_c(X_0))
      \\=&\mV(\widehat R_t(X_0),\widehat R_c(X_0)) -
      \mV(\widehat R_t(X_0), R^*_c(X_0))\\
      &+
      \mV(\widehat R_t(X_0), R^*_c(X_0)) -
      \mV( R^*_t(X_0), R^*_c(X_0)) .
\end{aligned}
\end{align}
The first term of \eqref{target_problem_rtrc} can be bounded as
\begin{align*}
    \begin{aligned}
        \mV(\widehat R_t(X_0),\widehat R_c(X_0)) -
      \mV(\widehat R_t(X_0), R^*_c(X_0)) = \widetilde\mO(\|\widehat R_C(X_0) - R^*_c(X_0)\|)  <
      \widetilde \mO(r_c^{1/4}n_s^{-\frac{\betatilde_c}{2\betatilde_c + 1}}).
    \end{aligned}
\end{align*}
Then we consider the second term of the \eqref{target_problem_rtrc},
\begin{align}
    \begin{aligned}
      \mV(\widehat R_t(X_0), R^*_c(X_0)) -  &
      \mV( R^*_t(X_0), R^*_c(X_0)) \\
      \lesssim&
         \underset{\widetilde R_t\in \mF_{R_t}}{\inf}
         \Big|\mV(\widehat R_t(X_0), R^*_c(X_0))
         -\mV( R^*_t(X_0), R^*_c(X_0))\Big|\\
         &+
         2\underset{\widetilde R_t\in \mF_{R_t}}{\sup}
         \Big|\mV_n(\widehat R_t(X_0), R^*_c(X_0))
         -\mV(\widehat R_t(X_0), R^*_c(X_0))\Big|.
    \end{aligned}
\end{align}
For the approximation error part, we have
\begin{align}
\label{res_tar_0205_1}
    \underset{\widetilde R_t\in \mF_{R_t}}{\inf}
         \Big|\mV(\widehat R_t(X_0), R^*_c(X_0))
         -\mV( R^*_t(X_0), R^*_c(X_0))\Big| = \widetilde\mO\big(r_t^{1/2}n_0^{\frac{-\betatilde_t}{2\betatilde_t + 1}}\big).
\end{align}
For the stochastic error part, we have
\begin{align}
\label{res_tar_0205_2}
    \underset{\widetilde R_t\in \mF_{R_t}}{\sup}
         \Big|\mV_n(\widehat R_t(X_0), R^*_c(X_0))
         -\mV(\widehat R_t(X_0), R^*_c(X_0))\Big| = \widetilde\mO\big(r_t^{1/2}n_0^{\frac{-\betatilde_t}{2\betatilde_t + 1}}\big).
\end{align}
Combining result \eqref{res_tar_0205_1}, \eqref{res_tar_0205_2}, we have
\begin{align}
\label{Result_target_II}
    II = \widetilde\mO\big(r_t^{1/2}n^{\frac{-\betatilde_t}{2\betatilde_t + 1}}\big) +
    \widetilde \mO(r_c^{1/4}n_s^{-\frac{\betatilde_c/2}{2\betatilde_c + 1}}).
\end{align}

Then
\begin{align}
\label{Result_target_III}
III = \lambda_{E,0}
\Big[
\mD(\widehat  R_t(X_0)||\gamma_{r_t}) -  \mD( R^*_t(X_0)||\gamma_{r_t})
\Big]= \widetilde \mO(r_t^{1/2}n_0^{-\frac{\betatilde_t}{2\betatilde_t + 1}}).
\end{align}
is obtained similarly as the result
\eqref{result_approx_Drx1},\eqref{result_approx_Drx2}
and \eqref{Result_source_stochastic_3}.

Combining
\eqref{Result_target_I},
\eqref{Result_target_II},
\eqref{Result_target_III},
we have the excess risk bound
\begin{align*}
    \mL_{T}([\widehat R_t,\widehat R_c]) &- \mL_{T}([R^*_t, R^*_c]) \nonumber\\
    =&
    I + II + III
    = \widetilde\mO\big(r_t^{1/2}n_0^{\frac{-\betatilde_t}{2\betatilde_t + 1}}\big) +
    \widetilde \mO(r_c^{1/4}n_s^{-\frac{\betatilde_c/2}{2\betatilde_c + 1}}).
\end{align*}
The proof is completed.

\section{Additional simulation results}

\subsection{Details for TransIRM}

TransIRM is based on the invariant risk minimization (IRM)  \citep{arjovsky2019invariant} which aims to learn the predictive representation $R_{so}(X)$ from source domains.

In this paper, we use IRMv1 \citep{arjovsky2019invariant}.  The objective function in source domains is
\begin{align*}
    \mL^{IRM}_{S}(w,g_c,R_{so}) = \sum_{s=1}^S \big[\mL_{pred}(w^Tg_c\circ R_{so}(X_s),Y_s)+\lambda\|\nabla_{w|w=1} \mL_{pred}(w^Tg_c\circ R_{so}(X_s),Y_s)\|_2^2\big],
\end{align*}
where $\mL_{pred}$ is the mean square loss for regression and cross-entropy loss for classification, $g_c$ is the regression function with the representation as input.
After training, we only transfer the representation $R_{so}$ to the target domain. The objective function of TransIRM in the target domain is defined as
\begin{align*}
    \mL_{pred}(g_0\circ R_{so}(X_0),Y_0),
\end{align*}
where
$R_{so}$ is frozen and $g_0$ is the prediction function taking $R_{so}$ as input.

The $R_{so}(X)$ learned with TransIRM contains useful information for predicting on the source domain but lacks sufficiency guarantees. We consider TransIRM as a competitor of the proposed method to illustrate the importance of adapting to the specific characteristics of the target domain.

\subsection{Details of implementation}

In this part we give the details of the network structure in all the numerical experiments.
We give the network structure of the 4 methods in the simulation in the Figure
\ref{fig:netStructure}-
\ref{fig:netStructureDNN}.
\begin{figure}[!htbp]
    \centering
    \includegraphics[width=1\textwidth]{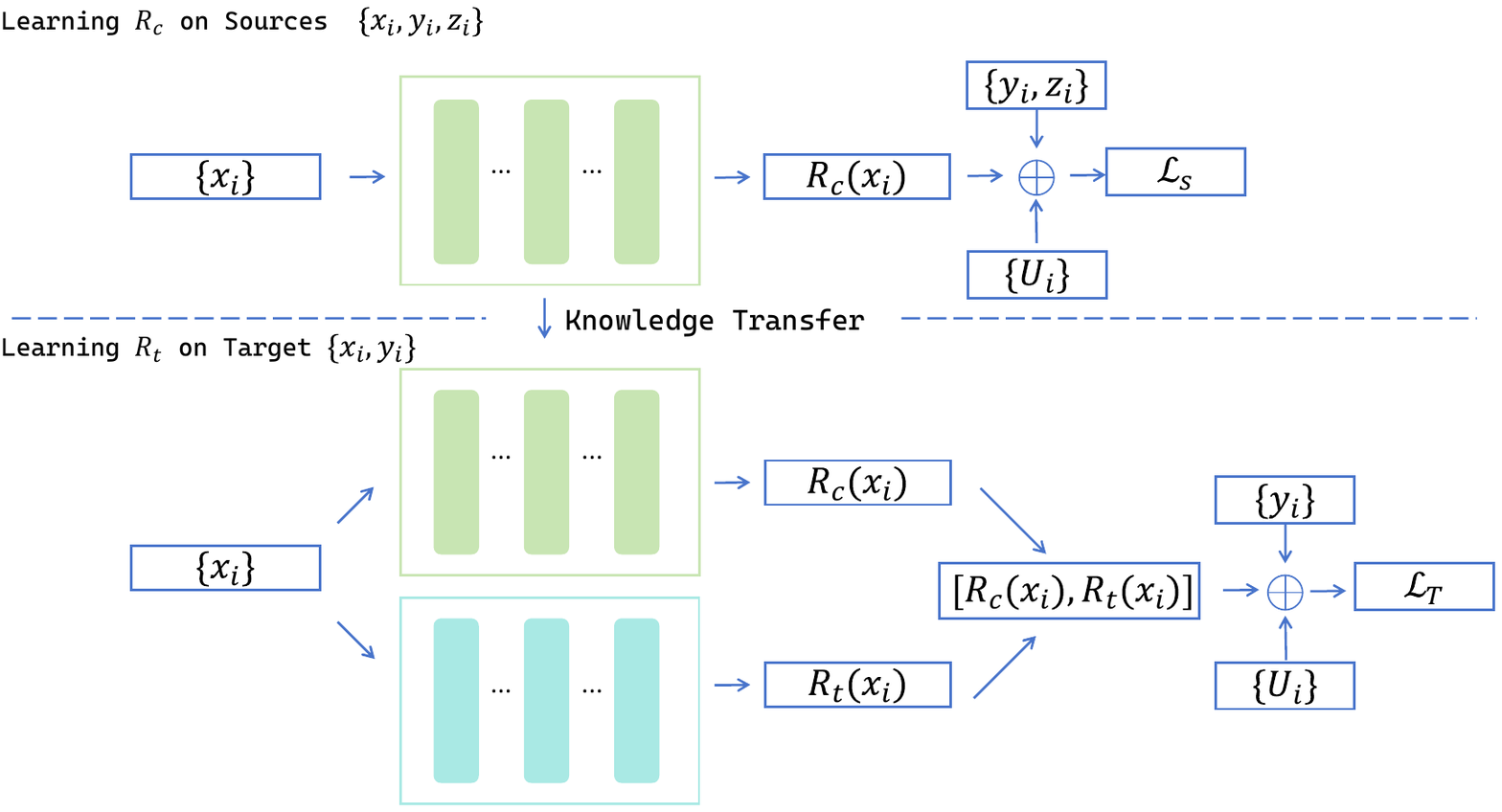}
    \caption{The Structure of the TESR framework}
    \label{fig:netStructure}
\end{figure}

\begin{figure}[!htbp]
    \centering
    \includegraphics[width=1\textwidth,trim = 0 70 10 0,clip]{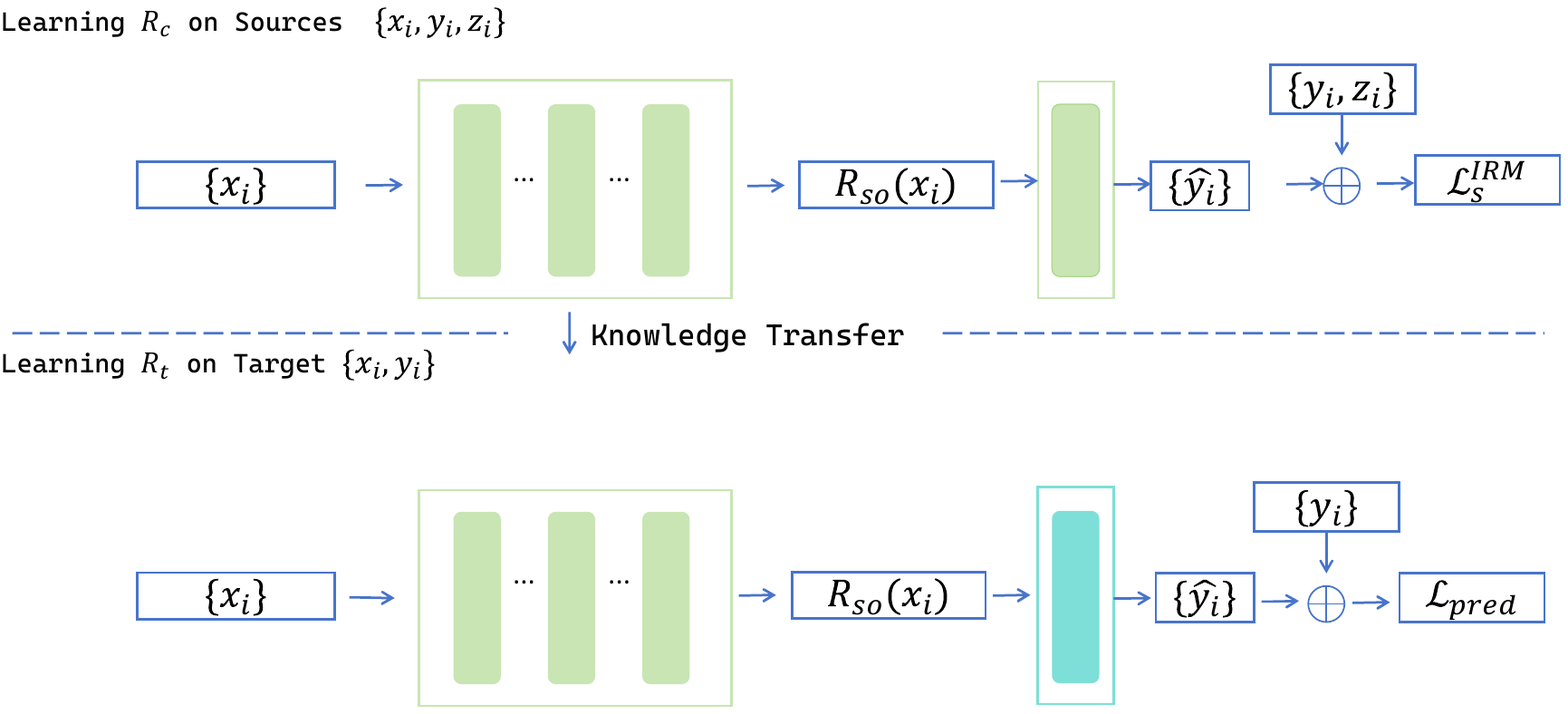}
    \caption{The Structure of the TransIRM framework. TransIRM is an end-to-end method and the prediction model taken the representation as inputs are shown in the figure.}
    \label{fig:netStructureIRM}
\end{figure}

\begin{figure}[!htbp]
    \centering
    \includegraphics[width=1\textwidth,trim = 0 150 10 0,clip]{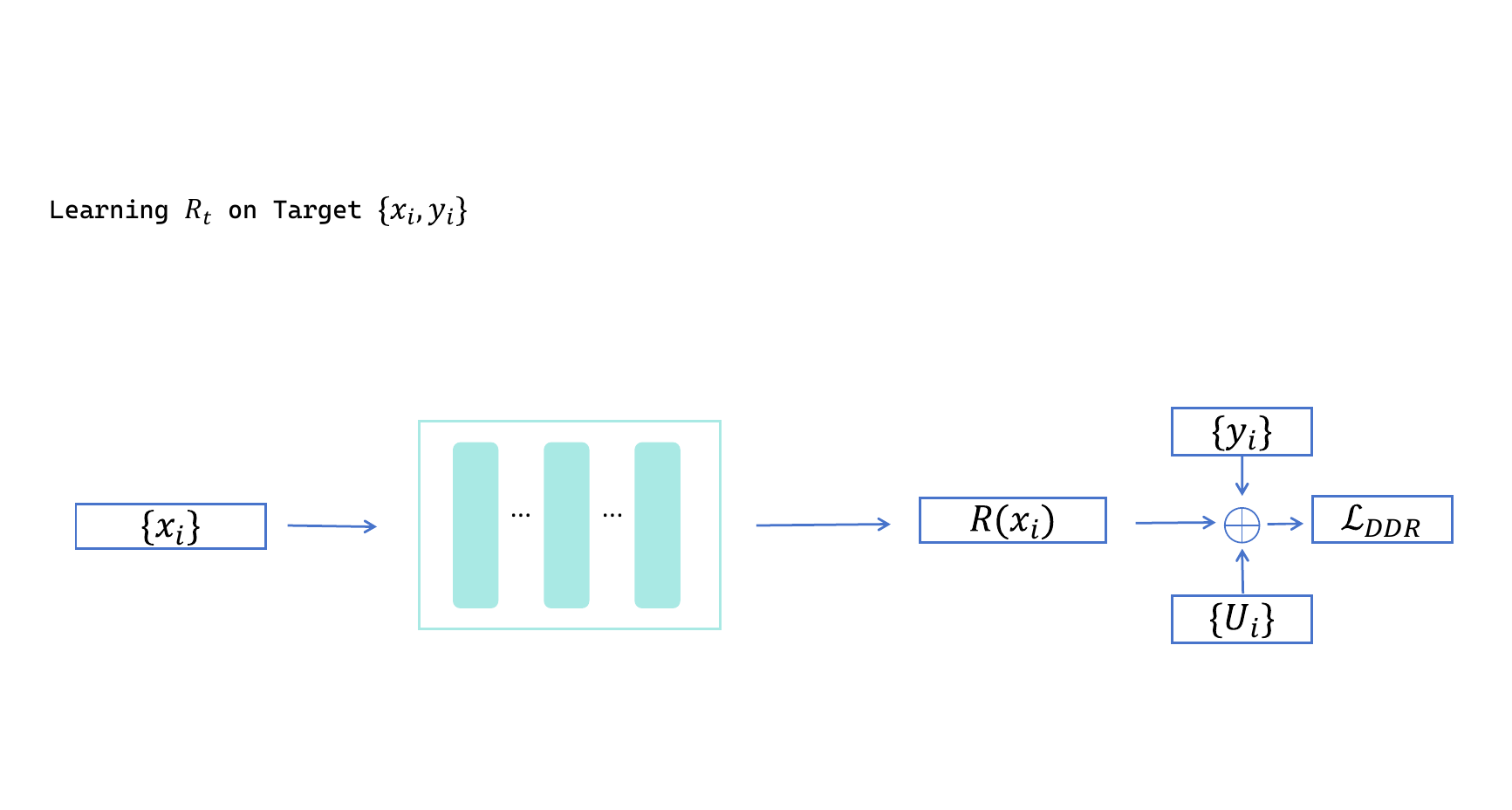}
    \caption{The Structure of the DDR framework}
    \label{fig:netStructureDDR}
\end{figure}

\begin{figure}[!htbp]
    \centering
    \includegraphics[width=\textwidth,trim = 0 150 10 0,clip]{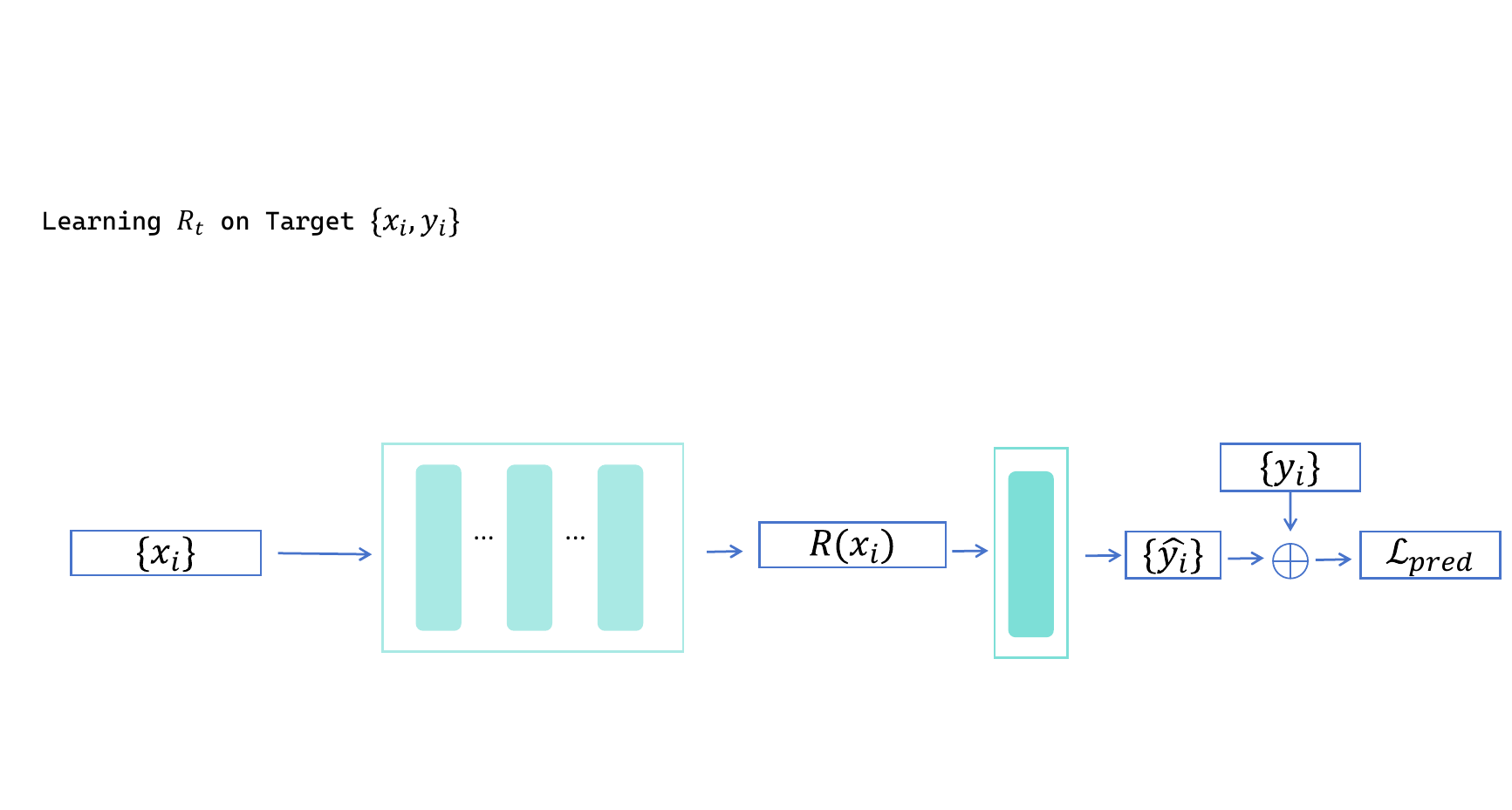}
    \caption{The Structure of the DNN framework. }
\label{fig:netStructureDNN}
\end{figure}

The  hyper-parameters for the simulated experiments are given in Table \ref{Table:parametersSimu}, where $\lambda_E$, $\lambda_Z$, $\lambda_C$, $\lambda_{E,0}$ are parameters in the objective functions, $bs$ is the mini-batch size, $ld$ is the dimension of representation and loop indicates the number of times to repeat the experiments.
In the simulation, we consider the RMSprop algorithm for the optimization for the neural network with the PyTorch package for implementation.
\begin{table}[!htbp]
\small
    \centering
    \caption{Hyper-parameters for TESR in simulated examples.}
\begin{tabular}{lccccccccc}
    \toprule
     & $\lambda_E$ & $\lambda_Z$ & $\lambda_C$ & $\lambda_{E,0}$ & $bs$ & $ld$ & Learning rate &weight decay & Epoch\\
    \midrule
   Example 1 & 0.1 & 0.1 & 0.1 & 0.1 & 64 & 32 &1e-3&1e-4& 300\\
   Example 2 & 0.1 & 0.1 & 0.1 & 0.1 & 64 & 32 &1e-3&1e-4& 300\\
   Example 3 & 0.1 & 0.1 & 0.1 & 0.1 & 64 & 32 &1e-3&1e-4& 300 \\
    \bottomrule
    \end{tabular}
    \label{Table:parametersSimu}
\end{table}

\begin{table}[!htbp]
\small
    \centering
    \caption{Hyper-parameters for TESR in Real data cases.}
\begin{tabular}{lccccccccc}
    \toprule
     & $\lambda_E$ & $\lambda_Z$ & $\lambda_C$ & $\lambda_{E,0}$ & $bs$ & $ld$ & Learning rate &weight decay & Epoch\\
    \midrule
      PACS & 0.01 & 0.01 & 1 & 0.01 & 128 & 64 &0.5$\times$1e-3&1e-4& 300\\
   JAM2 & 0.01 & 0.01 & 0.1 & 0.01 & 64 & 64 &0.5$\times$1e-3&1e-4& 200\\
    \bottomrule
    \end{tabular}
    \label{Table:parametersReal}
\end{table}

\begin{table}[!htbp]
\centering\footnotesize\renewcommand{\tabcolsep}{0.2pc}\renewcommand{\arraystretch}{0.8}
    \centering
    \caption{MLP architectures for $R_c$ and $R_t$ in simulation examples. The $d$ in the input size is the dimension of $X$.}
\begin{tabular}{llcc|lcc}
    \toprule
    &\multicolumn{3}{c}{$R_c$} &\multicolumn{3}{c}{$R_t$} \\
  \midrule
    Layers & Details & Input size & Output size & Details & Input size & Output size\\
    \midrule
Layer 1 & Linear & d & 64 & Linear & d& 64 \\
Activation &  LeakyReLU(0.2) & 64 & 64 &  LeakyReLU(0.2)  & 64 & 64 \\
Layer 2 & Linear & 64& 32 & Linear & 64  & 32 \\
Activation &  LeakyReLU(0.2) & 32 & 32 &  LeakyReLU(0.2) & 32 & 32 \\
Layer 3 & Linear &32&  32 & Linear & 32 &32 \\
\bottomrule
    \end{tabular}
    \label{Table:MLP }
\end{table}

\begin{table}[!htbp]
\centering\footnotesize\renewcommand{\tabcolsep}{0.2pc}\renewcommand{\arraystretch}{0.8}
    \centering
    \caption{CNN architectures for the analysis on dataset PACS. The BN denotes batch normalization layer.}
\begin{tabular}{llcc|lcc}
    \toprule
    &\multicolumn{3}{c}{$R_c$} &\multicolumn{3}{c}{$R_t$} \\
 \midrule
    Layers & Details & Input size & Output size & Details & Input size & Output size\\
    \midrule
Layer 1 & Convolution 3x3, BN & (3,32,32) & (48,32,32) & Convolution 3x3, BN & (3,32,32) & (48,32,32) \\
Activation &  LeakyReLU(0.2) & (48,32,32) & (48,32,32) &  LeakyReLU(0.2)  & (48,32,32) & (48,32,32) \\
Layer 2 & Convolution 3x3, BN & (48,32,32) & (96,32,32) & Convolution 3x3, BN & (48,32,32) & (96,32,32) \\
Activation &  LeakyReLU(0.2) & (96,32,32) & (96,32,32) &  LeakyReLU(0.2) & (96,32,32) & (96,32,32) \\
Layer 3 & MaxPooling 2x2  & (96,32,32) & (96,16,16) & MaxPooling 2x2  & (96,32,32) & (96,16,16) \\
Layer 4 & Convolution 3x3, BN & (96,16,16) & (192,16,16) & Convolution 3x3, BN & (96,16,16) & (192,16,16) \\
Activation &  LeakyReLU(0.2) & (192,16,16) & (192,16,16) &  LeakyReLU(0.2)  & (192,16,16) & (192,16,16) \\
Layer 5 & Convolution 3x3, BN & (192,16,16)6 & (256,16,16) & Convolution 3x3, BN & (192,16,16) & (256,16,16) \\
Activation &  LeakyReLU(0.2) & (256,16,16) & (256,16,16)  &  LeakyReLU(0.2) & (256,16,16) & (256,16,16) \\
Layer 6 & MaxPooling 2x2  &  (256,16,16) & (256,8,8)& MaxPooling 2x2  & (256,16,16) & (256,8,8) \\
Layer 7 & Linear & 16384 & 1024 & Linear & 16384 & 1024 \\
    \bottomrule
    \end{tabular}
    \label{Table:PACS }
\end{table}

\subsection{Additional simulation results}

\begin{table}[!htbp]
\centering\footnotesize\renewcommand{\tabcolsep}{0.2pc}\renewcommand{\arraystretch}{0.8}
    \centering
    \caption{MLP architectures for $R_c$ and $R_t$ in the  real data analysis for Gene JAM2.}
\begin{tabular}{llcc|lcc}
    \toprule
    &\multicolumn{3}{c}{$R_c$} &\multicolumn{3}{c}{$R_t$} \\
    \midrule
    Layers & Details & Input size & Output size & Details & Input size & Output size\\
    \midrule
Layer 1 & Linear & 1813 & 64 & Linear & 1813 & 64 \\
Activation &  LeakyReLU(0.2) & 64 & 64 &  LeakyReLU(0.2)  & 64 & 64 \\
Layer 2 & Linear & 64& 64 & Linear & 64  & 64 \\
Activation &  LeakyReLU(0.2) & 64 & 64 &  LeakyReLU(0.2) & 64 & 64 \\
Layer 3 & Linear &64&  64 & Linear & 64 &64 \\
\bottomrule
    \end{tabular}
    \label{Table:medical reg }
\end{table}

\subsubsection{Additional simulation: knowledge transfer from regression to regression tasks}
In this part, we give  additional simulation results where the sources task $\mD_s$,$s=1,2,3,4$ and the target task $\mD_0$ are all regression tasks.
In other word, the response in all the tasks are all continuous variable.
This is a commonly considered scenario in conventional transfer learning methods.

In this example, the response on the target domain $\mD_0$  is continuous and the $L_2$ loss is considered for the prediction modules of all four methods. This is different from the main text, where the logistic loss is applied to estimate the binary responses in the target domains.

{\noindent\bf Example S.1}
 We generate   4 sources and 1 target
 datasets with the following model:
\begin{itemize}
    \item $\mD_0:
y =3 f_1(x_1) + 1.5 f_2(x_2)  f_3(x_3) + f_6(x_6) + \epsilon_0;$
\end{itemize}
and the sources models:
\begin{itemize}
    \item $\mD_1:
y =2 f_1(x_1) + 1 f_2(x_2)  f_3(x_3) + f_4(x_4) + \epsilon_1;$
\item $\mD_2:
y =2 f_1(x_1) + 1 f_2(x_2)  f_3(x_3) + 2f_4(x_4) + \epsilon_2;$
\item $\mD_3:
y = 2 f_1(x_1) + 1.5 f_2(x_2)  f_3(x_3) + f_5(x_5) +  \epsilon_3;$
\item $\mD_4:
y = 2 f_1(x_1) + 1.5 f_2(x_2)  f_3(x_3) + 2f_5(x_5) + \epsilon_4;$
\end{itemize}
where
$f_1(u)=u$,
$f_2(u)=2u+1$,
$f_3(u)=2u-1$,
$f_4(u)=0.1\sin(\pi u)+0.2\cos(\pi u)$,
$f_5(u)=\sin(\pi u)/(2-\sin(\pi u))$,
$f_6(u)=u(|u|+1)^2$.
The $x \sim U(0,1)$ follows the uniform distribution
and $\epsilon_s$ for $s=0,1,2,3,4$ are independent random variables drawn from $N(0,0.5^2)$.

In Example S.1, we study the numerical performance with different $(n_s,n_0,p)$.
As shown in Figure \ref{Fig_Simu_S1},
the proposed TESR method outperforms all its competitor.
It also shows the superiority and applicability of the proposed method.

\begin{figure}[!htbp]
\centering
\includegraphics[width=\textwidth]{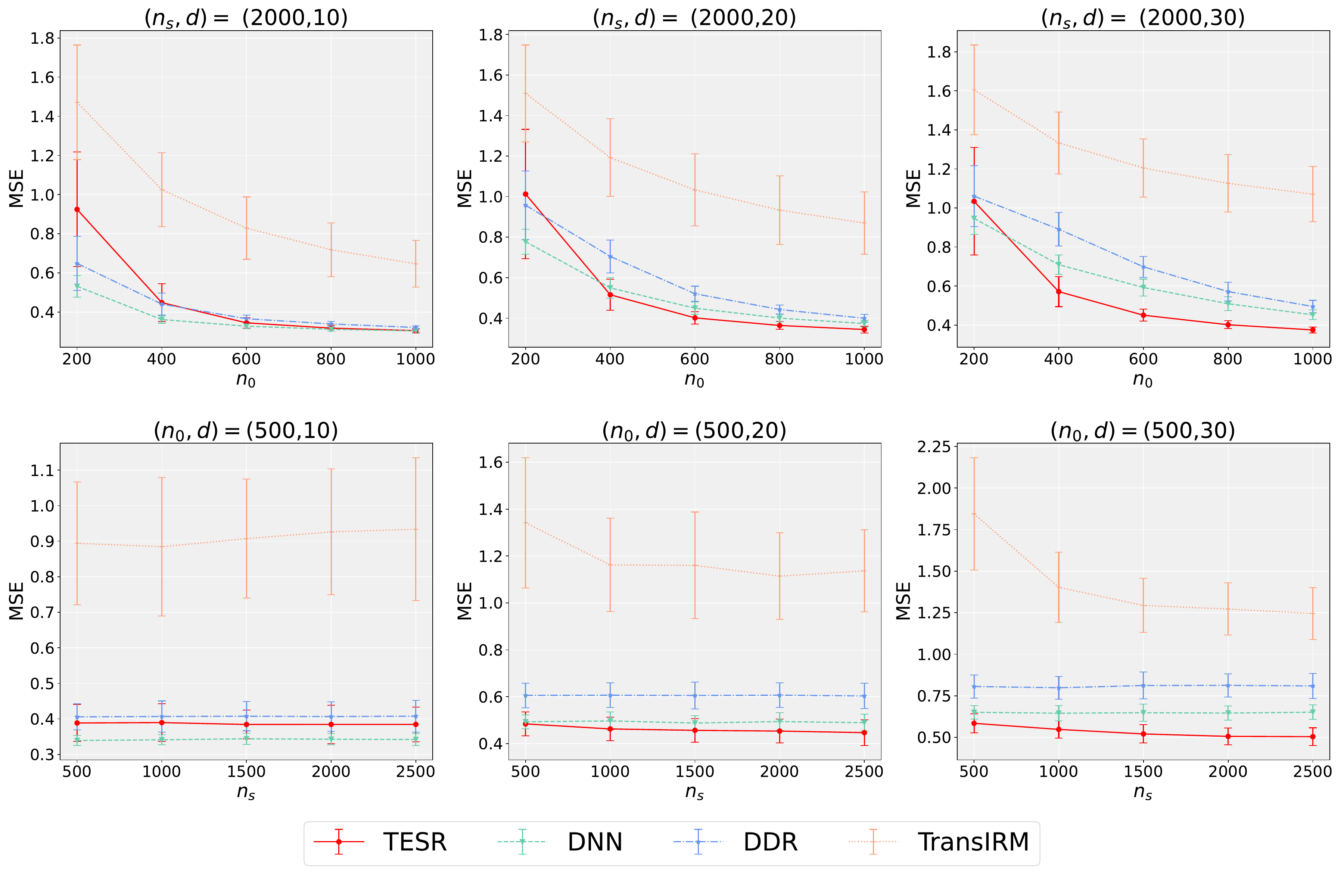}
\caption{the MSE on the $\mD_0$ in Example S.1 with different $(n_s,n_0,d)$.}\label{Fig_Simu_S1}
\end{figure}

\subsubsection{Distance between the sources and the target in Example 3}

In Example 3, we consider  two types of model departures.
In Figure \ref{Fig_Simu_3_distance}, we calculate the numerical $L_1$ distance and the cosine distance  under these departures.

Specifically, for two regression function $g_1(x)$, $g_2(x)$, the $L_1$ distance is defined as
$$
   \int|g_1(x) - g_2(x)|\mathrm{d}P(x),
$$
and the cosine distance is defined as
$$
1 -  \mathop{Cor} \Big[g_1(x),g_2(x)\Big].
$$
It is clear that cosine distance lies in $[0,1)$ implying the positive correlation and $(1,2]$ for negative correlation.

In the left panel of Figure \ref{Fig_Simu_3_distance}, we report the $L_1$ distance under the Type I departure.
It is clear that $L_1$ distance increases with $s=1,\dots,6$.
We note that the $\mD_1,\dots,\mD_6$ have larger $L_1$ distance with the target data, compared with the
$\mD_7$,
$\mD_8$. However,
$\mD_7$,
$\mD_8$
contain no useful information for the target task.
It implies the invalidity of $L_1$ distance as an task similarity measure.

In the right panel of Figure \ref{Fig_Simu_3_distance}, we report the cosine similarity
 distances under the Type II departure.
It is clear that correlation changes for $s=1,\dots,6$.
We note that the correlation  between the target task and $\mD_1,\dots,\mD_6$ can be positive and negative. It introduces more heterogeneity for the transfer learning problem.

The experiment results in Example 3 support our argument that TESR does not rely on the the premise of the regression function similarity between the source and the target data.

\begin{figure}[!htbp]
\centering
\includegraphics[width=0.9\textwidth]{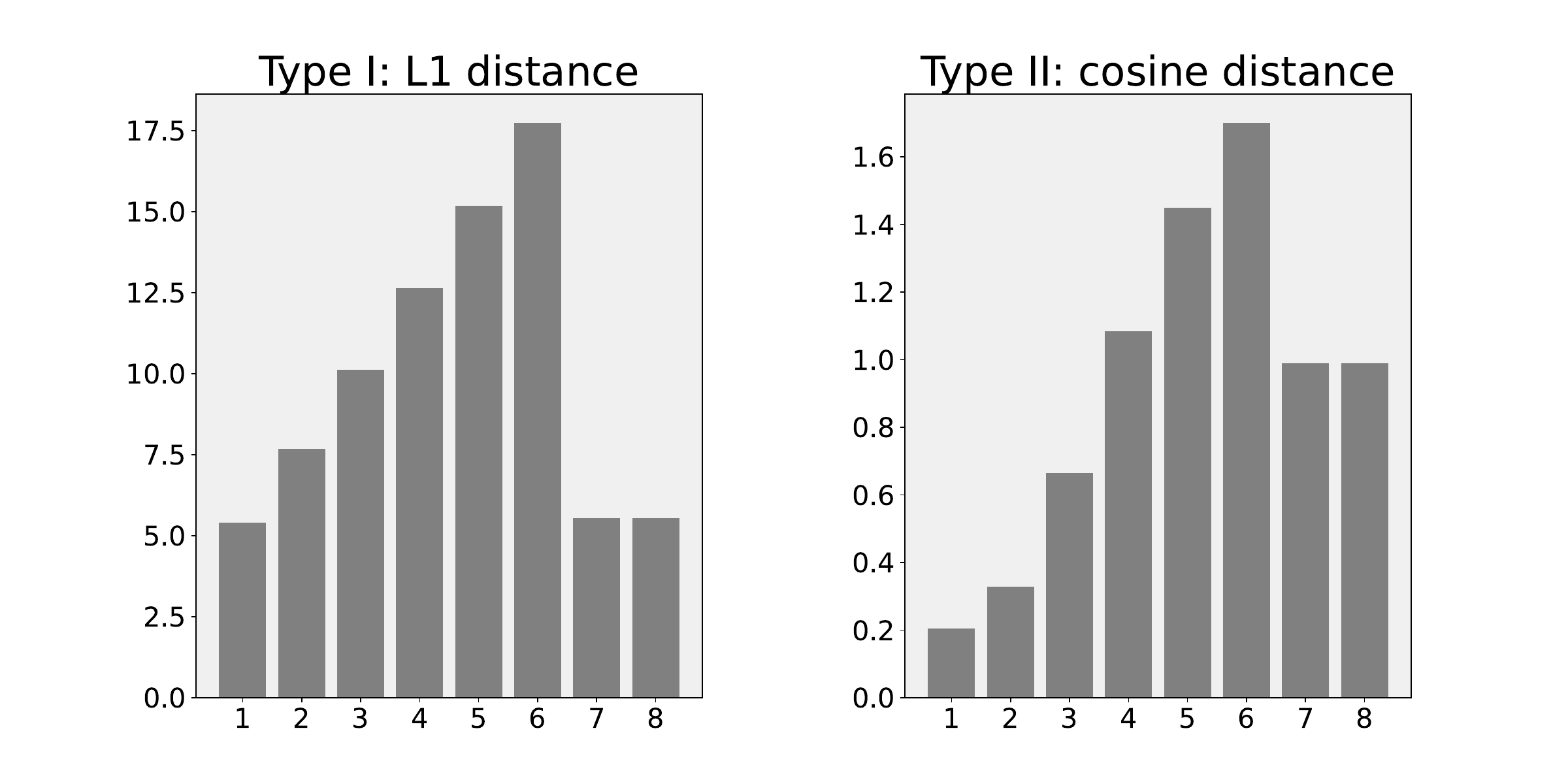}
\caption{
The empirical distances between the regression functions from the sources and the target in Example 3.
}\label{Fig_Simu_3_distance}
\end{figure}
\end{document}